\documentclass[11pt]{amsart}
\usepackage[margin=1in]{geometry}
\usepackage{graphicx}
\usepackage{natbib}
\usepackage{hyperref}
\usepackage{amssymb}
\usepackage{amsthm}
\usepackage{dsfont}
\usepackage{mathtools}
\usepackage{xfrac}
\usepackage{algorithm}
\usepackage{algpseudocode}  
\usepackage{lipsum}
\usepackage{thmtools}
\usepackage{xcolor}
\usepackage{soul}

\newcommand{\I}{\mathbb{I}}
\newcommand{\N}{\mathbb{N}}
\newcommand{\Z}{\mathbb{Z}}

\newcommand{\R}{\mathbb{R}}
\newcommand{\C}{\mathbb{C}}
\newcommand{\Ebb}{\mathbb{E}}
\newcommand{\p}[1]{\mathbb{P}(#1)}
\newcommand{\abs}[1]{\left|#1\right|}

\newcommand{\avg}[1]{\left<#1\right>}
\newcommand{\E}[2][]{\Ebb_{#1}[#2]}
\newcommand{\gauss}[3][]{\mathcal{N}(#1 #2, #3)}

\newcommand{\htheta}{\hat{\theta}}
\newcommand{\hTheta}{\hat{\Theta}}
\newcommand{\cK}{\mathcal{K}}
\newcommand{\cAK}{\mathcal{AK}}
\newcommand{\cHK}{\mathcal{HK}}

\newcommand{\oC}{\widetilde{C}}

\newcommand{\Var}{\operatorname{Var}}
\newcommand{\trans}{\operatorname{trans}}

\newcommand{\inter}{\operatorname{inter}}
\newcommand{\intra}{\operatorname{intra}}

\DeclareMathOperator*{\argmax}{arg\,max}
\DeclareMathOperator*{\argmin}{arg\,min}

\newtheorem{counter}{Counter}[section]
\newtheorem{lemma}[counter]{Lemma}

\newtheorem{proposition}[counter]{Proposition}
\newtheorem{theorem}[counter]{Theorem}
\newtheorem{corollary}[counter]{Corollary}
\newtheorem{definition}[counter]{Definition}
\newtheorem{hypothesis}[counter]{Hypothesis}
\newtheorem{remark}[counter]{Remark}

\newcommand{\iid}{\stackrel{iid}{\sim}}
\newcommand{\lln}{\stackrel{LLN}{\sim}}

\newcommand{\texteq}[1]{\stackrel{#1}{=}}

\begin{document}

\title[]{Communities in the Kuramoto Model: Dynamics and Detection via Path Signatures}
\author{Tâm J Nguyên$^{1,2}$, Darrick Lee$^3$ and Bernadette J Stolz$^{1,4,}$\textsuperscript{*}}

\address{$^1$ Laboratory for Topology and Neuroscience, School of Life Sciences, École Polytechnique Fédérale de Lausanne, 1015 Lausanne, Switzerland}
\address{$^2$ Laboratory of Computational Neuroscience, School of Computer and Communication Sciences and School of Life Sciences, École Polytechnique Fédérale de Lausanne, 1015 Lausanne, Switzerland}
\address{$^3$ School of Mathematics and Maxwell Institute for Mathematical Sciences, University of Edinburgh, Edinburgh EH9 3FD, Scotland}
\address{$^4$ Department of Machine Learning and Systems Biology, Max Planck Institute of Biochemistry, Am Klopferspitz 18, 82152 Martinsried, Germany}
\thanks{\textsuperscript{*}Corresponding author. Email: \texttt{stolz@biochem.mpg.de}}

\begin{abstract}

The behavior of multivariate dynamical processes is often governed by underlying structural connections that relate the components of the system. For example, brain activity, which is often measured via time series is determined by an underlying structural graph, where nodes represent neurons or brain regions and edges represent cortical connectivity. Existing methods for inferring structural connections from observed dynamics, such as correlation-based or spectral techniques, may fail to fully capture complex relationships in high-dimensional time series in an interpretable way. Here, we propose the use of path signatures—a mathematical framework that encodes geometric and temporal properties of continuous paths—to address this problem. 
Path signatures provide a reparametrization-invariant characterization of dynamical data and, in particular, can be used to compute the lead matrix, which reveals lead-lag phenomena. 
We showcase our approach on time series from coupled oscillators in the Kuramoto model defined on a stochastic block model graph, termed the \emph{Kuramoto Stochastic Block Model} (KSBM). Using mean-field theory and Gaussian approximations, we analytically derive reduced models of KSBM dynamics in different temporal regimes and theoretically characterize the lead matrix in these settings. Leveraging these insights, we propose a novel signature-based community detection algorithm, achieving exact recovery of structural communities from observed time series in multiple KSBM instances. {We also explored the performance of our community detection on a stochastic variant of the KSBM as well as on real neuropixels of cortical recordings to demonstrate applicability on real-world data}. Our results demonstrate that path signatures provide a novel perspective on analyzing complex neural data and other high-dimensional systems, explicitly exploiting temporal functional relationships to infer underlying structure.
\end{abstract}

\maketitle

\section{Introduction}

\begin{figure}[ht!]
  \centering
  \includegraphics[width=1\linewidth]{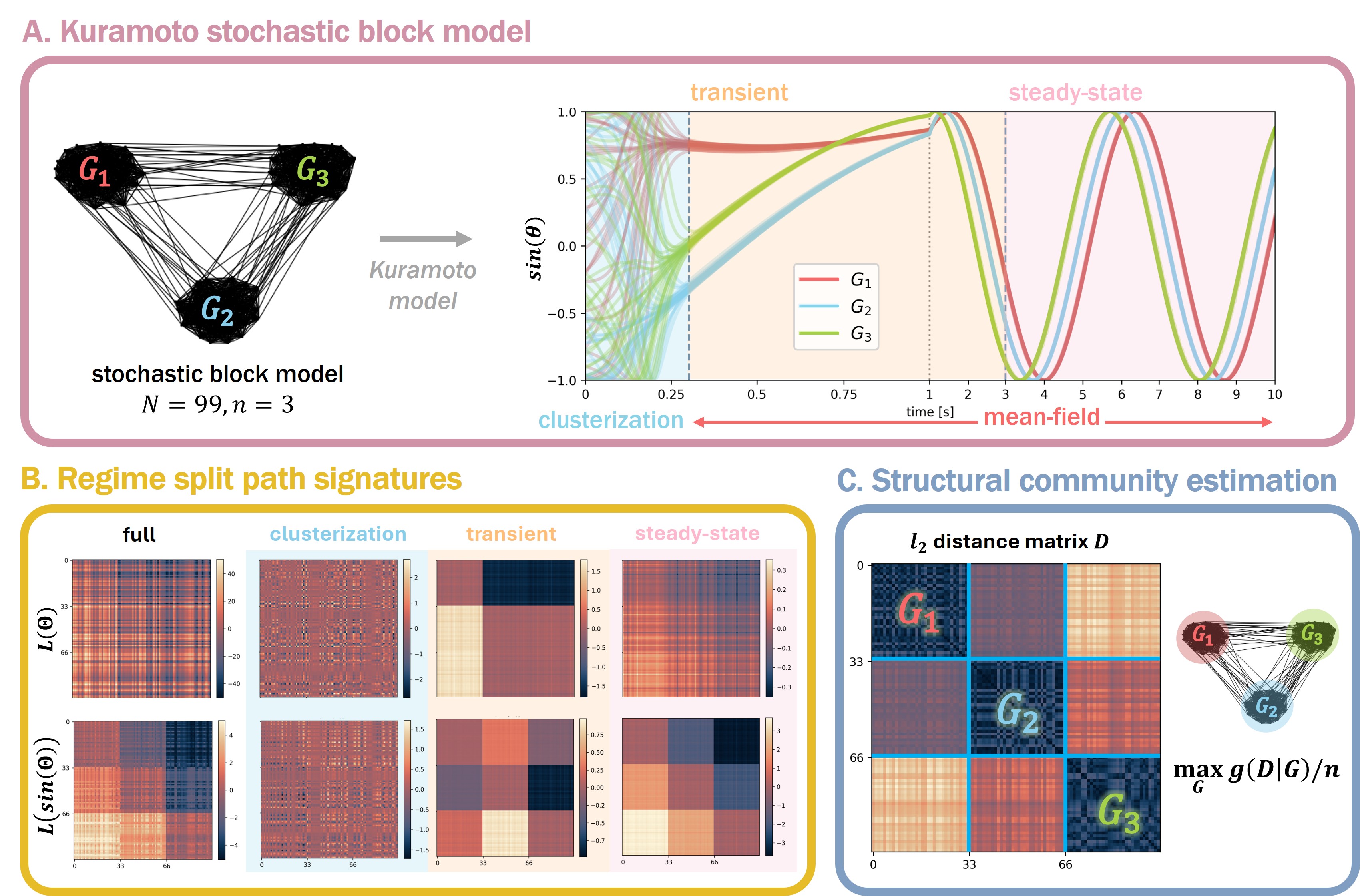}
  \caption{\textbf{Community estimation in the Kuramoto Stochastic Block Model}\\
   (A) The Kuramoto Stochastic Block Model (KSBM) is a version of the Kuramoto model with underlying coupling between oscillators defined by a stochastic block model graph. Time series $\Theta$ arising from the KSBM can be split into different temporal regimes with different macro-scale properties. (B) We investigate high-dimensional time series using path signatures. From the path signatures, we derive lead matrices $L(\Theta)$ and $L(\sin(\Theta))$ which capture lead and lag behavior. When computed over the (full) time series produced by our KSBM, i.e. ignoring the different time regimes shown in (A), {high variation of lead values within community blocks} prevents identification of communities in $L(\Theta)$. When splitting the time series into different time regimes, we find clearer block patterns in $L(\Theta)$ for the transient regime and in $L(\sin(\Theta))$ for the transient and steady state regimes. (C) We propose a novel clustering algorithm based on path signatures and lead matrices that estimates communities in the KSBM's underlying stochastic block model graph from time series split into different time regimes. Our algorithm is based on maximizing a block-clustering metric $g(D|G)/n$ of the distance matrix $D$ over community assignment $G$ with $n$ communities.}
  \label{fig:overview}
\end{figure}

Many dynamical processes in our physical world are rooted in structured interactions between entities of interest. This underlying structure is often not accessible, and we are often limited to observing resulting dynamical processes through time series from experiments.  
In many research fields such as neuroscience, estimating structural interactions which give rise to functional activity of a system is of great interest \cite{tirabassi2015, rubido2014, shandilya2011, timme2007, timme2014, levnajic2011}. 

\subsection{Motivation: Neuroscience}

Cognitive function is determined by anatomical connections between neurons or brain regions. These underlying connections can be interpreted as structural networks in which the neurons or brain regions are nodes and their connections represent edges.
Identifying the underlying structural networks that give rise to specific cognitive functions is one of the key questions in experimental neuroscience~\cite{Purves2001}. However, due to experimental and ethical constraints, direct observation of these structural connections is often impossible. Instead, we typically study neural activity that reflects the dynamics arising in connected neurons or brain regions. Linking this observed neural activity to the underlying structural network is a complex challenge. Brain activity arises from the simultaneous occurrence of multiple neural processes involving many neurons; the resulting data is thus often high-dimensional and noisy. Biologically, brain activity is driven by overlapping circuits of connected neurons, which fire in response to inputs from other neurons. In a neural circuit, the axon connection between a presynaptic and a postsynaptic neuron induces time synchrony in the responses of the connected neurons, i.e. the firing of the presynaptic neuron will typically lead to a time-delayed firing in the postsynaptic neuron. This neural activity can be represented in a lower-dimensional latent space, known as the neural manifold \cite{Langdon2023-cr}, which enables the study of phenomena such as synchrony \cite{Gine2013}. In this latent space, the time-delayed firing induced by synaptic connections reflects a form of synchronization between the activity of connected neurons with a temporal offset.

The degree of synchronization in neural activity is referred to as \emph{functional connectivity}~\cite{petersen2015,sporns2015,Bullmore2011,Bullmore2009,papo2014}, which can be measured from coupled time series data, such as fMRI or EEG, using a variety of methods. A common approach is to calculate the pairwise cross-correlation between time series. This functional connectivity can then be used to construct a functional network, where nodes represent neurons or brain regions, and edges are weighted according to the degree of functional connectivity. Such networks can be analyzed using techniques from graph theory or network science\footnote{Note that one can also study coupled time series using other tools, e.g. \cite{tirabassi2015,nakamura2016,sun2014}.} \cite{petersen2015,sporns2015,Bullmore2011,Bullmore2009,papo2014}.
However, because neurons often participate in multiple circuits, functional connectivity alone is typically insufficient to distinguish between different circuits and thus directly infer structural connections.

The highly structured nature of brain circuits allows us to utilize knowledge about neuronal organization, such as their partition into densely-connected communities, to facilitate the inference of structural connections.
Let us consider a motivating example that demonstrates principles of neuronal organization. Cortical processes of perception and sensation are active in most background brain activities. The neural activity generated by these cortical processes involves sensory cells first transducing their sensory inputs to the primary cortical areas followed by the synchronization of cortical columns.
These columns consist of neurons whose activity continuously represents the perceived space
    (known as topographic mapping\footnote{Mapping is understood here in the biological sense: each neuron is activated by a specific part of the space of stimuli, called a receptive field. Formally, we could write the neural activity map as $f: X \times Y \rightarrow \R$, where $X$ is some embedding of the cortex and $Y$ the space of stimuli. Neighbourhoods of neurons are activated by neighbourhoods of receptive fields, hence this mapping of the stimuli space to the activity of neuron is continuous.} in neuroscience\cite{Kaas2001}). For example, in auditory perception the perceived space is the space of frequencies perceivable by a human. This space is represented by cortical columns in the primary auditory cortex\cite{Purves2001}; each column is activated by the presence of a specific frequency in the perceived sound. Those columns continuously represent lower to higher frequencies, preserving the structure of the frequency space. Due to this continuous representation, neurons in the same column respond similarly to similar inputs. Upon activation, neurons in the same cortical column activate downstream neurons in other cortical regions,
synchronizing their activity over time. Untangling the activity of synchronized neurons into the specific cortical columns they belong to is highly relevant to a mechanistic understanding of processing of sensory information in the brain. We can think of neurons in the same cortical column as communities which are densely connected to each other anatomically while being sparsely connected to neurons in other cortical columns. These communities possess a structural connectivity given by the underlying synaptic circuitry and additional observed functional connectivity given by the synchronization of their activities.
Ideally, we would want to estimate entire underlying circuits from observed functional activity of those anatomical circuits. 
The primary aim of the present work is to develop a novel \emph{structural community estimation} (SCE) algorithm which identifies underlying communities from multivariate time series whose dynamics are determined by unknown underlying structural connections.

\subsection{Kuramoto Stochastic Block Model}

To provide a rigorous analytical study of our method, we focus our work on the Kuramoto model, a popular mathematical model in neuroscience which exhibits dynamical properties similar to real neural systems\cite{frank2000, Luke2012}.
Here, we use the Kuramoto model as a source of synthetic time series data to avoid effects of noise that is often present in experimental data and the complexity inherent to modelling spiking neuron dynamics\cite{Gerstner2014}. 
We consider a specific version of the Kuramoto model where the structural coupling is governed by a class of random graphs: the \textit{stochastic block model} (SBM)~\cite{Abbe2018}.

SBMs are graphs, where the set of nodes is partitioned into communities and the probability of an edge between two nodes depends only on their respective communities and not their individual identities. Similarly to cortical columns, we consider \emph{assortative structures}, where edges between nodes occur preferably inside the same node community, and far less frequently between different node communities. Coupling between oscillators drives alignment of phases over time, which typically results in frequency synchronization. 
We refer to this model as the \textit{Kuramoto Stochastic Block Model} (KSBM). It is a simple yet rich model of synchronized and community-clustered processes (see Figure~\ref{fig:overview}) that has been previously used, e.g., in~\cite{Stolz2017-lw}~\cite{bassett2013robust}~\cite{Leo2023}, and exhibits {the characteristic dynamics of the Kuramoto model: synchronization of oscillators at steady state when their coupling strength is chosen above a critical threshold.}

Previous work by Arenas \emph{et al.}~\cite{arenas2006-sync-top-comp-net,arenas_synchronization_2006,arenas_synchronization_2007,arenas_synchronization_2008} studied the hierarchical structure of fixed networks using the Kuramoto model by observing the dynamics at random initial conditions, including with spectral methods. 
These articles, along with several others~\cite{li_synchronization_2008,boccaletti_detecting_2007} developed community detection algorithms based on the Kuramoto dynamics. The focus of these early articles is to leverage known Kuramoto dynamics, as a \emph{method} to detect intrinsic community structure in fixed graphs. {For example, Li \emph{et al.}~\cite{li_synchronization_2008} and Boccaletti \emph{et al.}~\cite{boccaletti_detecting_2007} focused on identifying communities in graphs derived from real-world data by simulating the Kuramoto model, using the graph structure to define coupling between oscillators. Communities are then estimated from the synchronization behavior of oscillators. In contrast, our work considers the KSBM as a \emph{generating model} of time series with oscillators coupled according to an underlying community structure that we wish to recover. We determine communities directly from the resulting time series without knowledge of the underlying structural graph.}
The stability of intra-community synchronization in Kuramoto models with underlying communities\footnote{In the context of SBM, communities are also often referred to as clusters.} has previously been studied~\cite{Favaretto2017-wd, Menara2019-ok, Menara2020-ul}.
In particular, Achterhof \emph{et al.}~\cite{Achterhof2021-I, Achterhof2021-II} investigated the phase diagram of synchronized steady states in the context of two communities.
More recently, Timofeyev and Patania~\cite{timofeyev_cluster_2025} studied the relationship between Laplacian spectra and Kuramoto dynamics of almost equitable partitions and weighted generalizations.

\subsection{Contribution 1: Temporal Regimes and Dynamics of the KSBM}
In this article, we build upon this previous work, and show analytically that KSBM time series can be partitioned into previously observed distinct temporal regimes, and provide a detailed study of the dynamical behavior within these regimes. 
Although the KSBM is a widely used toy model with well-characterized dynamics, the separation of its time series into distinct temporal regimes has not, to our knowledge, been formally shown. synchronization properties depending on the underlying network structure have been investigated formally, see e.g. recent work by Nagpal \emph{et al.}\cite{nagpal2024}, Townsend \emph{et al.}\cite{townsend2020}, and Kassabov \emph{et al.} \cite{kassabov2021}, but these do not consider KSBMs and in particular do not investigate different temporal regimes. Indeed, Nagpal \emph{et al.}\cite{nagpal2024} point to the KSBM as being a case of interest for analytical studies.

A visual representation of the temporal regimes is shown in Figure~\ref{fig:overview}A.
Within this version of the Kuramoto model, synchronization occurs more quickly and frequently among oscillators within the same community, while oscillators from different communities require more time to synchronize, if they synchronize at all. As a result, each community of oscillators forms a cohesive, synchronized unit that initially operates independently from others, giving rise to lower-dimensional dynamics governed by the community structure.
While the functional connectivity observed in this model reflects the underlying SBM structural network, this behavior contrasts with the standard Kuramoto model, where all oscillators eventually form a single synchronized unit, lacking any notion of community structure.

We use approximations to analytically study the dynamics of the temporal regimes; to our knowledge, we are the first to explore KSBM dynamics analytically on more than two communities~\cite{Achterhof2021-I, Achterhof2021-II} and outside of steady state. In the initial clusterization regime, we use Gaussian approximation~\cite{Sonnenschein2013-mh} of communities to reduce the dynamics to a model that includes only the mean and the variance of the oscillator ensemble.
Once the communities are clustered, the dynamics are driven by the average over the oscillators of each of the communities, where we use the mean-field theory developed in~\cite{Chiba2019-et, Kaliuzhnyi-Verbovetskyi2018-qr} to study the dynamics. 
These approximations are especially useful for estimating error bounds and assessing the accuracy of clustering methods applied to time series from KSBMs.

\subsection{Contribution 2: Structural Community Estimation}

Equipped with a theoretical understanding of the dynamics and temporal regimes that lead to synchronization, we propose a novel method, called \textit{structural community estimation}, to identify structural communities from observed dynamics. A fundamental component of our approach is to leverage the \emph{path signatures}~\cite{chevyrev_primer_2016, giusti_iterated_2020}, a characterization of paths (up to time-reparametrization), which has recently been used to study both stochastic differential equations~\cite{lyons_differential_2007,friz_multidimensional_2010} and time series data in machine learning~\cite{lyons_signature_2022, lee_signature_2023}. The path signature encodes information between collections of time series, {which can capture more complex interactions in high-dimensions than existing methods. For example, covariance based methods may struggle with causal interactions which depend on the coordinated activity of triplets of neurons, or (trial-) averaging may destroy fine temporal structure for events with variable timing and duration (e.g., a mouse licking)}. We focus on specific pairwise components of the path signature that are summarised in the \emph{lead matrix}~\cite{baryshnikov_cyclicity_2016}. The lead matrix encodes pairwise lead-lag behavior between time series, e.g., how changes in one time series determine changes in the other.
Lead matrices have the additional advantage of being time-reparametrization invariant. This property allows the investigation of cyclic (non-periodic) and noisy dynamics, such as those found spiking neural networks.

Readily applying lead matrices or covariance matrices of coupled time series generated from the KSBM results in poor estimators of communities (Figure~\ref{fig:overview}B). By partitioning the time series into separate temporal regimes, we can remove most of the {heterogeneity} and, in particular, {separate communities that are structurally distinct but functionally equivalent at steady state}. We show that lead matrices behave very differently in each temporal regime, and by building lower-dimensional models of the KSBM in each regime, we are able to reconstruct our communities from the lead matrices. 
An advantage of the Kuramoto model in this context, is that synchronization of oscillators has a straightforward interpretation for lead matrices. Indeed, in the case of synchronized oscillators the lead corresponds to time offsets between oscillators.
However, our approach is not limited to KSBMs, but generalizes to similar problems in other fields, e.g., in physics or biology. {In particular, we show that in real neural recordings from mice in a sensory-motor task our method is able to identify clusters of neurons with structured activity in this task.} 
We define a notion of \emph{block-clustering} for matrices, and develop an algorithm for \textit{structural community estimation} (Figure~\ref{fig:overview}C) which leverages the lead matrix to infer communities from time series with similar temporal patterns, revealing underlying structural connections in the data.
While there are other approaches for inferring structural connections from functional data \cite{tirabassi2015, rubido2014, shandilya2011, timme2007, timme2014, levnajic2011}, we are the first to propose path signatures and the lead matrix for this task.

\subsection{Outline} 
In Section~\ref{sec:KSBM-dynamics}, we develop mean-field models (Theorem \ref{thm:MF-KSBM}) for the transient and steady state regimes and Gaussian low-rank models (Theorem \ref{thm:gaussian-assortative-KSBM}) for the clusterization regime in the KSBM. Using these models, we provide expected times of transitions between regimes which we call the \emph{transition time} (Lemma \ref{lemma:transition-time-assortative-KSBM}). We show that the resulting regime split is particularly relevant since clusterization hinders community estimation while the mean-field regime (corresponding to the transient and steady state regimes) enables it.
In Section~\ref{sec:S-community-detection}, we compute analytical expressions for path signatures at synchronized steady state for the KSBM oscillators' time series $\Theta$ and $\sin(\Theta)$. Motivated by the structure of these expressions, we construct a metric which we call \emph{block-clustering} for functional connectivity matrices. We use this metric to develop a new community detection algorithm tailored to path signatures, called the structural community estimation algorithm (Alg.\ref{alg:structural-community-estimation}). We show that our algorithm can recover communities in numerical experiments when considering distinct time regimes.

Finally, in Section~\ref{sec:experiments}, we provide experiments to numerically verify our results. In particular, we show in Section~\ref{ssec:gauss_mf} that our analytic approximations agree with numerical results, and we demonstrate the efficacy of our structural community estimation algorithm in Sections~\ref{ssec:exp_community} and \ref{ssec:stochastic-ksbm}. {To conclude, we show that in real neural data from mice our algorithm identifies clusters of neurons with activity specific to a sensory-motor-task.}

\section{Kuramoto Stochastic Block Model and its Dynamics}
\label{sec:KSBM-dynamics}
\subsection{Kuramoto Stochastic Block Model}

The Kuramoto Stochastic Block Model (KSBM) is a version of the generalized Kuramoto model\footnote{In the rest of this paper, we will refer to the generalized Kuramoto model as the Kuramoto model; whenever we need to make the distinction with the original model, we will simply refer to the non-generalized version as the standard Kuramoto model.}\cite{rodrigues_kuramoto_2016} in which the coupling is given by an underlying stochastic block model (SBM) graph\cite{Abbe2018}, and where the intrinsic frequency of each oscillator is drawn from a distribution specific to its community in the SBM. In general, the Kuramoto model consists of $N$ oscillators $\theta_i$ valued in the unit circle $S^1 := \R \text{ mod } 2\pi$. For simplified notation, we define the set $[N] := \{1, ..., N\}$. The oscillators in the Kuramoto model satisfy the following system of differential equations,

\begin{equation}
\label{eq:generalized-kuramoto}
    \dot{\theta_i}(t) = \omega_i + \sum_{j\in[N]}{\oC_{ij}\sin(\theta_j(t) - \theta_i(t))}.
\end{equation}
We denote the collection of all phases by $\Theta(t) \coloneqq (\theta_i(t))_{i=1}^N$.
The intrinsic frequencies $\omega_i$ correspond to the angular speed of the oscillators when not influenced by other oscillators, i.e. when the oscillators are not coupled. Coupling between oscillators is defined by the coupling matrix $\oC \in  \R^{N \times N}$. The factor $\sin(\theta_j(t) - \theta_i(t))$ drives the phase alignment of oscillators $\theta_i$ and $\theta_j$ whenever the coupling is positive, i.e., $\oC_{ij} > 0$.

The coupling matrix $\oC$ encodes structural information underlying the Kuramoto model. In particular, if $\oC_{ij} = 0$, then oscillator $\theta_j$ has no influence on oscillator $\theta_i$. We can isolate this structural information using an $N \times N$ binary matrix where $A_{ij} = 1$ if and only if $\oC_{ij} \neq 0$. This is a directed adjacency matrix of an underlying graph. Here, a parallel to neuroscience can be drawn and the adjacency matrix can be interpreted as underlying structural connections between the oscillators.

The KSBM is a special case of the Kuramoto model where $A$ is determined by an SBM. An SBM is a random graph model, where the nodes are partitioned into \emph{communities}, and the probability of an edge existing depends only on the communities of the two endpoints. Note that random graph models can be equivalently defined as a random adjacency matrix.

We consider SBM consisting of $n$ communities with $m$ nodes each, for a total of $N = mn$ nodes. We fix a balanced community assignment	 $\phi: [N] \to [n]$ by $\phi(i) = \left\lceil \frac{i}{n}\right\rceil$. We define the \emph{nodes in community $r$} as $G_r = \phi^{-1}(r)$ for any $r\in[n]$.
The distribution of a random adjacency matrix $A$ of an SBM is specified by a probability matrix $P \in [0,1]^{n\times n}$:
\begin{align*}
    \p{A_{ij} = 1} = P_{\phi(i), \phi(j)}.
\end{align*}
We write $G \sim SBM(n,m,P)$ to denote such a random (directed) graph.

Throughout this article, we only consider balanced SBMs to simplify the notation. However, our results can be extended to the unbalanced setting using similar methods.

We define the KSBM as follows.

\begin{definition} \label{def:ksbm}
    (Kuramoto Stochastic Block Model, KSBM)\\
    The \emph{Kuramoto SBM} $\cK = \cK(n,m,p, C, \mu, \sigma, \theta^0)$ is a random Kuramoto model defined by Equation~\ref{eq:generalized-kuramoto}, such that 
    \begin{itemize}
        \item the \emph{adjacency matrix} $A \sim SBM(n,m, P)$ is an SBM with $N$ nodes, $n$ communities of $m$ nodes each denoted by $\{G_r\}_{r=1}^n$, and probability matrix $P$;
        \item the \emph{coupling matrix} $\oC$ is defined by the \emph{community coupling matrix} $C \in \R^{n \times n}$, where $C_{rs} = C_{sr}$, by $\oC_{ij} = C_{\phi(i), \phi(j)} \cdot A_{ij}$;
        \item the \emph{intrinsic frequencies} $\omega_i$ are i.~i.~d.~normally distributed as $\omega_i \iid \gauss{\mu_r}{ \sigma^2}$ for all $r\in [n]$, where $\mu_i \in \R$ is the community mean frequency and $\sigma^2 \in \R^+$ is a fixed variance; and
        \item the \emph{initial conditions} $\theta_i(0) = \theta_i^0$ are fixed. 
    \end{itemize}
    We refer to the random variable $(A,\omega)$ representing the adjacency matrix and intrinsic frequencies as the \emph{KSBM randomness} and to a specific sample of this random variable as a \emph{realisation} of the KSBM. 
\end{definition}

We will also consider a variant of KSBMs where the analytical computations are more tractable: the \emph{assortative KSBM}. An assortative KSBM $\cAK = \cAK(n,m,\kappa, \mu, \sigma, \theta^0)$ is a random Kuramoto model defined by Equation~\ref{eq:generalized-kuramoto} such that
\begin{itemize}
    \item communities are fully connected, $P_{rr} = 1$ for any $r \in [n]$;
    \item for each node $i \in G_r$, we add one edge between node $i$ and a node $j$ in another community $j\in G_s$ with $s\neq r$, and 
    \item the coupling strengths are uniform, $\oC_{ij} = \frac{\kappa}{N}$ for all $i,j\in [n]$ for some $\kappa \in \R$;
    \item intrinsic frequencies $\omega_i$ and initial conditions $\theta_i(0)$ are set according to Definition~\ref{def:ksbm}.
\end{itemize}

{By construction, in the final graph each node may be connected to more than one node outside its community.}

The assortative KSBM\footnote{Strictly speaking, this is not the same model of the SBM as defined for the KSBM since the edge probabilities are distributed differently. For instance, if $i \in G_r$ and $j \in G_s$, then $A_{i,j}$ and $A_{ik}$ are \emph{independent} in the KSBM, but \emph{dependent} in the assortative KSBM. However, they are functionally equivalent when the number of nodes grows to infinity\cite{Abbe2018}.} is the model we use in our simulations and for the Gaussian approximation in Section~\ref{ssec:gaussian-KSBM}. We can view this as constructing $n$ fully-connected Kuramoto models and then weakly coupling them using a single edge across communities for each node. It is directly inspired by the models studied in \cite{Stolz2017-lw,Leo2023,Menara2019-ok,Menara2020-ul} with the added change of forcing the communities to be assortative for analytical simplicity.

Our primary running example is an assortative KSBM with $N = 99$ oscillators and $n = 3$ communities with high coupling ($\kappa = 100$) and low {heterogeneity} $\sigma = 0.1$. We will refer to it as the \textbf{standard configuration}, we show a simulation in Figure~\ref{fig:example-dynamic}.
Here, we qualitatively observe several time regimes. First, there is a marked qualitative change in dynamics around $t = 0.3 s$. We refer to the time regime up to $t = 0.3 s$ as \textit{clusterization}, as it represents a transition from a uniform initial state to synchronization of frequencies and phases within the structural communities of oscillators. In this first time regime, the oscillators are primarily driven by the coupling inside their own community. Thus, each community can be studied approximately independently from one another. 

In the second time regime, starting where $t > 0.3 s$, the clustered communities start influencing each other before reaching a steady state, which in our model corresponds to the frequency synchronization of all oscillators. During this time regime, oscillators within the same community behave homogeneously, so instead of treating them as separate units, we can focus on the interactions between average oscillators from each community. We refer to this time regime as the \emph{mean-field regime} where the system behaves like a Kuramoto model with each community represented as a single average oscillator.
We show this property of the KSBM in Theorem \ref{thm:MF-KSBM}. One can further split this second regime into a \textit{transient} regime followed by the \textit{steady state}. While the steady state is usually the preferred setting in which to analyze properties of Kuramoto models (see \cite{Favaretto2017-wd,Menara2020-ul}), this regime can obfuscate a lot of useful information that would allow us to identify oscillator communities. Indeed, if all intrinsic frequencies are the same, then the steady state will exhibit all oscillators converged to a single large unit rotating along the unit circle (see Figure~ \ref{fig:KSBM-dynamics-collapsed}), voiding any distinction between the underlying structural communities. Therefore, to cluster oscillators effectively, we need to study the transient effects carefully and understand how community coupling relates to observed dynamics.

\begin{figure}[ht!]
    \centering
    \includegraphics[width=1\linewidth]{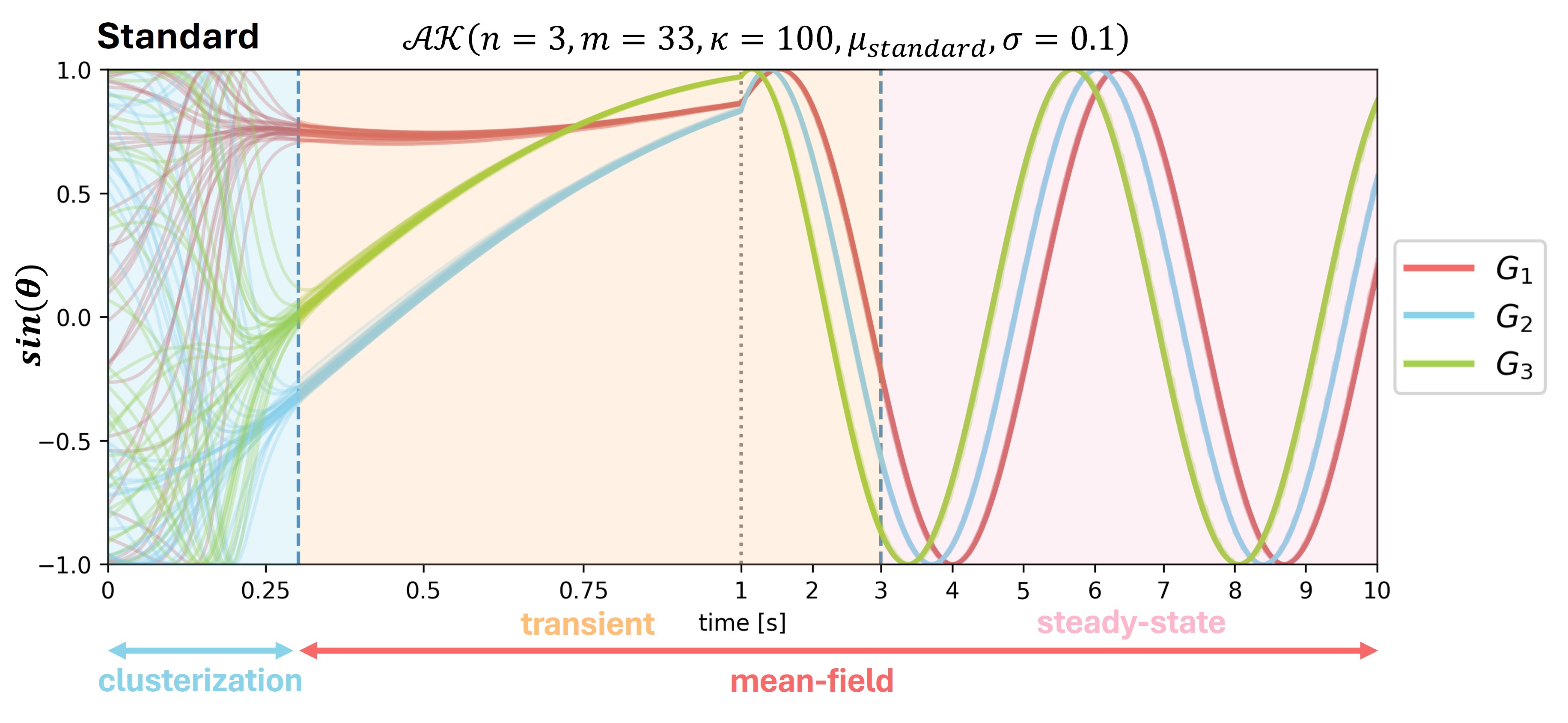}
    \caption{\textbf{Example of KSBM dynamics and temporal regimes}\\
    We show time series of the standard assortative KSBM $\cAK(n=3, m=33, \kappa=100, \mu=\{\frac{2}{3}, \frac{14}{9}, 2 ~rad/s\}, \sigma=0.1)$ with communities $G_1, G_2, G_3$. The time interval from 0 to 1 second is expanded to highlight the clustering regime, which ends around time = 0.3 seconds, transitioning into a mean-field regime where the behavior of individual oscillators is indistinguishable from the behavior of other oscillators within the same community. The mean-field regime can be further divided into transient synchronizing dynamics followed by a steady state at around time = $3$ seconds.}
    \label{fig:example-dynamic}
\end{figure}

\subsection{Analysis of KSBM Dynamics}

In this section, we analytically study the behavior of KSBMs in both the mean-field and clusterization regimes. We begin by focusing on the \emph{mean-field regime}, where we approximate a system of $N$ oscillators as a system of $n$ mean oscillators, which we call the mean-field KSBM (Theorem \ref{thm:MF-KSBM}). Next, we move on to the \emph{clusterization regime}, where we consider the specific case of the assortative KSBM, and use Gaussian approximations to study the transient variance of phases, which we call the \emph{Gaussian KSBM} (Theorem \ref{thm:gaussian-assortative-KSBM}). 

We first cover some useful terminology to define our regimes. For $\epsilon > 0$ and $t > 0$, a subset $U \subseteq [N]$ of oscillators $(\theta_i)_{i \in U}$ is \emph{$\epsilon$-clustered at time $t$} if
\[
    \max_{i,j \in U}{\abs{\theta_i(t) - \theta_j(t)}} \leq \epsilon,
\]
and is \emph{$\epsilon$-variance-clustered at time $t$} if
\[
    \Var_{i \in U}(\theta_i(t)) \leq \epsilon.
\]

\begin{definition}
\label{def:KSBM-clusterization}
    (KSBM Clusterization)\\
    Let $\epsilon > 0$ and $T > 0$. 
    A KSBM is \emph{$\epsilon$-(variance-)clustered by time $T$} if for each community, all oscillators within are $\epsilon$-(variance-)clustered for all time $t \geq T$.\\
    Moreover, it is \emph{$\epsilon$-strong-(variance-)clustered by time $T$} if in addition, for all time $t \geq T$
    \[
    \min_{r,s \in [n]}{\abs{\theta_{G_r}(t)-\theta_{G_s}(t)}} > \epsilon.
    \]
\end{definition}

\begin{definition}
    (Synchronization)\\
    We say that a KSBM is \emph{synchronized} at time $t$ if $\dot{\theta}_i(t) = \dot{\theta}_j(t)$ for all $i, j \in [N]$.
    We say it is in \emph{synchronized steady state by time $T$} if it is synchronized for all time $t \geq T$.
\end{definition}

\subsubsection{Mean-Field KSBM}

In this section, we consider a collection of KSBMs
\[
\cK_m = \cK(n,m,P,\mu,\sigma, \theta^0)
\]
and their realisations $\Theta^{(m)}$, with varying number of nodes $m$ in each community where all other parameters are fixed. Furthermore, we consider a specific instance of the intrinsic frequencies $\{\omega^{(m)}_i\}_{i=1}^N$, which is fixed throughout this section. In particular, we only consider the graph randomness inherent in the random SBM.

The dynamics in Figure~\ref{fig:example-dynamic} beyond the clusterization regime are qualitatively governed by the mean of the oscillators in each community.
This suggests the use of \emph{mean-field theory}\cite{Kaliuzhnyi-Verbovetskyi2018-qr}, where we approximate the dynamics of such mean oscillators, i.e. the mean phases and intrinsic frequencies in each community.
We denote the mean phase and intrinsic frequency in community $G_r$ as 
\begin{align*}
\theta_{G_r} := \avg{\theta_i}_{i \in G_r}
\quad \text{and} \quad \omega_{G_r} := \avg{\omega_i}_{i \in G_r}  
\end{align*}
respectively. Here, our aim is to show that in mean-field regime, the mean oscillators $\theta_{G_r}$ behave deterministically. 

Our main tool is the mean-field theory for Kuramoto models established in~\cite{Chiba2019-et, Kaliuzhnyi-Verbovetskyi2018-qr}. Fix a number $m \in N$ of nodes in each community. We define a \emph{deterministic} Kuramoto model with $N$ nodes, $\hTheta^{(m)} =\{\htheta^{(m)}_i\}_{i=1}^N$ where $\htheta^{(m)}_i: [0,T] \to S^1$, and governed by the dynamics
\begin{align} \label{eq:mean_field_deterministic}
    \dot{\htheta}^{(m)}_i(t) = \omega_i^{(m)} + \sum_{j=1}^N P_{\phi(i), \phi(j)} C_{\phi(i), \phi(j)} \sin(\htheta^{(m)}_i(t) - \htheta^{(m)}_j(t))
\end{align}
with the same initial conditions as our KSBM, $\htheta^{(m)}_i(0) = \theta^0_i$. In particular, this is a fully connected Kuramoto model, where the intrinsic frequency is given by the mean intrinsic frequency $\mu_r$ of the KSBM model and the \emph{deterministic} coupling strength is scaled by the corresponding edge probabilities. {Next, we state an important result by Kaliuzhnyi and Verbovetskyi~\cite{Kaliuzhnyi-Verbovetskyi2018-qr}[Lemma 1.1], originally proved in~\cite{Chiba2019-et}, in the specific case of the KSBM linking it to the deterministic Kuramoto model of Equation~\ref{eq:mean_field_deterministic}.}

\begin{theorem}{\cite[Lemma 1.1]{Kaliuzhnyi-Verbovetskyi2018-qr}} \label{thm:mean_field_kuramoto}
    Let $\cK_m = \cK(n,m,P,\mu,\sigma, \theta^0)$ be a collection of KSBMs where $n, P, \mu, \sigma$ and $\theta^0$ are fixed. Let $\hTheta^{(m)}$ be the Kuramoto model defined in Equation~\ref{eq:mean_field_deterministic}. Then, 
    \begin{align}
        \lim_{m \to \infty} \sup_{t \in [0,T]} \frac{1}{\sqrt{N}}\left\| \Theta^{(m)}(t) - \hTheta^{(m)}(t)\right\|_2 = 0
    \end{align}
    almost surely (with respect to the randomness of SBM).
\end{theorem}

Applying Theorem~\ref{thm:mean_field_kuramoto}, we show that in the large $m \gg 1$ limit, the mean oscillators approximately behave deterministically with respect to a Kuramoto model on $n$ oscillators. 

\begin{theorem}
\label{thm:MF-KSBM}
    (Mean-Field KSBM)\\
    Let $\Theta^{(m)}$ be a realization of the KSBM $\cK_m$ and $\hTheta^{(m)}$ be the Kuramoto model defined in Equation~\ref{eq:mean_field_deterministic}. Suppose $\Theta^{(m)}$ is $\epsilon$-clustered with $\epsilon \ll 1$, $\bar{C} := \avg{C_{rs}P_{rs}}_{r\neq s\in [n]}=O(1/N)$  and $m \gg 1$. Then for all $r\in [n]$,
    \[
    \dot{\theta}^{(m)}_{G_r}(t) = \mu_r + m\sum_{s=1}^n{P_{rs}C_{rs}\sin(\htheta^{(m)}_{G_s}(t) - \htheta^{(m)}_{G_r}(t))} + O(\epsilon)
    \]
    almost surely (with respect to SBM randomness).
\end{theorem}
\begin{proof}
    Appendix \ref{proof:thm:MF-KSBM}
\end{proof}

\subsubsection{Gaussian KSBM}\label{ssec:gaussian-KSBM}

In this subsection, we focus on the \emph{clusterization regime}: the transient process of oscillators going from an unclustered state to a clustered state.
In simulations, this process occurs for a sufficiently strong coupling $\kappa$ (Figure~\ref{fig:example-dynamic}). Here, we will analytically study the clusterization regime of the assortative KSBM. Instead of the mean-field approximation, we will now consider a \emph{Gaussian approximation}\cite{Sonnenschein2013-mh}, where we assume that the oscillators within a community are i.i.d.~normally distributed at every point in time, and we call this the \emph{Gaussian KSBM}. generalizationWe derive differential equations describing the mean and variance for these communities and use this to study the clusterization of this approximation. 
In particular, we find that for any $\epsilon > 0$, we can scale the coupling strength $\kappa$ so that KSBM is clustered with factor $\epsilon$ (Theorem~\ref{thm:dominated-gaussian-assortative-KSBM}).

\begin{hypothesis}
\label{hypothesis:gaussian-phase}
    (Gaussian Assumption)\\
    We assume that at time $t$ for any $r \in [n]$,
    \[
    \theta_i(t) \iid \gauss{\theta_{G_r}(t)}{V_{G_r}(t)},
    \]
    for all $i \in G_r$, where $V_{G_r}:[0,T] \rightarrow \R_{\geq 0}$ is the variance of $\theta_i$ in cluster $G_r$ at time $t$.
\end{hypothesis}

Although the distribution of $\theta_i$ for $i \in G_r$ in a KSBM is not Gaussian, communities of oscillators concentrate around their mean given sufficiently large intra-community coupling, qualitatively giving the distribution of each community a bell shape after a few time steps~\cite{Sonnenschein2013-mh}, which we \emph{approximate} by Gaussians. Within this model, a community of oscillators is characterized by its mean and variance, both as functions of time. We will see that the Gaussian KSBM is consistent with the Mean-Field KSBM and can therefore be viewed as a generalization of the mean-field analysis of the KSBM.

We will restrict our analysis to the Gaussian approximation of the assortative KSBM with positive coupling $\kappa > 0$, in which we denote the phases by $\Theta^{A}$. We leave the generalization of all theorems and results below to the general (non-assortative) KSBM as future work; these generalizations would require more involved necessary conditions on the coupling strengths. 

In this Section we will focus on a simpler model where we can understand the clusterization mechanisms and the minimal variance it can reach. {We leave the generalization to the full model to the appendix (Thm.~\ref{thm:gaussian-assortative-KSBM})}. We will therefore introduce an additional simplifying assumption:

\begin{hypothesis}
\label{hypothesis:intra-community-dominated}
    (Intra-Community Dominated Dynamics)\footnote{Notice that under this hypothesis, the KSBM consists of $n$ independent fully-connected Kuramoto models.}\\
    We assume
    \[
    \dot{\theta}_i^{A}(t) = \omega_i + \frac{\kappa}{N}\sum_{j\in G_r}{\sin(\theta^{A}_j(t)-\theta_i^{A}(t))}.
    \]
\end{hypothesis}

Hypothesis \ref{hypothesis:intra-community-dominated} is justified by the assortative structure of the coupling graph: $\theta^A_i$ connects to the $m$ other oscillators in $G_r$, and only finitely many other nodes in $\bigcup_{s\neq r}G_s$ almost surely as $m \to \infty$.  It follows that during clusterization, $\dot{\theta}^{A}_i$ is driven by $\sum_{j\in G_r}{\sin(\theta^{A}_j(t)-\theta_i^{A}(t))}$ since it has infinitely many non-zero terms. This assumption ceases to hold almost surely once the clusterization process ends, since the oscillators inside the community will have synchronized in phase and thus the terms of the intra-community sum vanish.

\begin{theorem}
\label{thm:dominated-gaussian-assortative-KSBM}
    (Dominated Gaussian Assortative KSBM)\\
    Under Hypothesis \ref{hypothesis:gaussian-phase} and \ref{hypothesis:intra-community-dominated} when $m \gg 1$ and $n\sigma \ll \kappa$, the assortative KSBM intra-community variance $V(t):=V^{A}_{G_r}(t)$ for any $r \in [n]$ follows,
    $$\frac{dV(t)}{dt} = \epsilon-\frac{2\kappa}{n}V(t) e^{-V(t)}$$
    where $\epsilon>0$ is a constant such that $\epsilon  \leq \frac{\sigma^2\pi n}{\kappa}$.
    Furthermore, if $e\sigma^2\pi n^2 < 2\kappa^2$, the variance has a stable steady state $V^{SS}$ which satisfies $V^{SS} \leq v^*$ where $v^* \geq 0$ is the smallest fixed point of,
    $$v = \frac{\pi}{2}\left(\frac{\sigma n}{\kappa}\right)^2e^{v}.$$
\end{theorem}
\begin{proof}
    Appendix \ref{proof:thm:dominated-gaussian-assortative-KSBM}
\end{proof}

This result gives us a bound on the steady state intra-community phase variance (when neglecting inter-community couplings) given the strength of the coupling $\kappa$ and the intrinsic frequency {heterogeneity} $\sigma$. This implies that the approximate KSBM reaches a clustered state by the time it enters the steady state.

\begin{lemma}
\label{lemma:dominated-gaussian-assortative-KSBM-is-clustered}
    Under Hypothesis \ref{hypothesis:intra-community-dominated}, the dominated Gaussian assortative KSBM is $\sqrt{\frac{\pi}{2}}\frac{\sigma n}{\kappa}$-clustered with high probability by some time $T > 0$ if $m \gg 1$ and $n\sigma \ll \kappa$.
\end{lemma}
\begin{proof}
    Appendix \ref{proof:lemma:dominated-gaussian-assortative-KSBM-is-clustered}
\end{proof}

\subsubsection{Transition Time}
In this section, we estimate the duration of such a clusterization process, and when it ends, a time point we call the {transition time} $t_{\trans}$.
In particular this is the time at which the intra- and inter-community coupling have the same magnitude in expectation.

We will show that a {transition time} describes the time point at which a KSBM finishes clusterization and enters mean-field regime.

\begin{definition}
\label{def:transition-time}
    (Transition Time)\\
    Let $\Theta$ be a realisation of a KSBM. We define the \emph{intra- and inter-community coupling} as 
    \begin{align*}
        C_{\intra}(r, t) &\coloneqq C_{rr}\sum_{j \in G_r}{\sin(\abs{\theta_j(t) - \theta_i(t)})\I\{j \sim i\}},\\
        C_{\inter}(r, t) &\coloneqq \sum_{s \neq r}^n{C_{rs}\sum_{j \in G_s}{\sin(\abs{\theta_j(t) - \theta_i(t)})\I\{j \sim i\}}}.
    \end{align*}
    The \emph{transition time} $t_{\trans}$ is defined to be the time at which
    \[
        \sum_{r=1}^n C_{\intra}(r, t_{\trans}) = \sum_{r=1}^n C_{\inter}(r, t_{\trans}).
    \]
\end{definition}

The relative magnitude between intra- and inter-communities coupling varies from intra-community dominated dynamics in the high entropy initial dynamics (see Theorem \ref{thm:dominated-gaussian-assortative-KSBM}) to an inter-community dominated dynamics in the clustered state (see Theorem \ref{thm:MF-KSBM}).\\
\\
We remark that initially, $\frac{C_{\intra}(r)}{C_{\inter}(r)} > 1$, but this ratio will then decrease as the communities cluster when strong-clustered. If the communities all synchronize in phase together, the ratio is undefined, hence the need to assume strong-clustered. Nonetheless, in an assortative KSBM, community clusterization happens before community synchronization (when communities could merge), thus for our purpose we can consider $C_{\inter}$ being given by communities of oscillators offset by some phase. Using this, we can relate $t_{\trans}$ to the variance of our Gaussian assortative KSBM,\\

\begin{lemma}
\label{lemma:gaussian-transition-time-bound}
    For a Gaussian assortative KSBM with intra-community variance $V(t)$, when $m \gg 1$, $t_{\trans} = O(t^*) \text{ almost surely}$ where $t^*$ is the earliest time such that for some $\nu > 0$,
    $$V(t^*) \leq (\frac{1}{m})^{2 + \nu}.$$
\end{lemma}
\begin{proof}
    Appendix \ref{proof:lemma:gaussian-transition-time-bound}
\end{proof}

{This result also holds for KSBMs in general (Cor.~\ref{cor:transition-time-bound}), therefore, Gaussian KSBMs can be used to estimate their {transition times}.}
In particular, we use the dominated identical Gaussian assortative KSBM and take the empirical estimator,
$$\hat{t}_{\trans} = \inf_t\{t : V(t) \leq (\frac{1}{m})^2\}$$
which we derive from Lemma~\ref{lemma:gaussian-transition-time-bound} by taking $\nu = 0$. From this {transition time}, one can separate the (Gaussian) KSBM dynamics into the clusterization process occurring from $0 \leq t < t_{\trans}$ and the mean-field dynamics taking place for $t > t_{\trans}$. Indeed from Lemma \ref{lemma:dominated-gaussian-assortative-KSBM-is-clustered}, we know that for sufficient coupling then the KSBM will tend toward a clustered steady state with factor $\epsilon = \sqrt{\frac{\pi}{2}}\frac{\sigma n}{\kappa}$. If $\kappa$ is large enough such that $\epsilon \leq (\frac{1}{m})^{2 + \nu}$, then we know that $\epsilon \ll 1$. Hence for any $t \geq t_{\trans}$, our (Gaussian) KSBM follows the mean-field approximation since Theorem \ref{thm:MF-KSBM} holds. Whenever $t < t_{\trans}$, we can approximate the dynamics by the intra-community coupling when $N \gg 1$, thus the dominated (Gaussian) KSBM is a good model of clusterization.\\
\\
Finally, we show that the Gaussian assortative KSBM is consistent with both the dominated Gaussian assortative KSBM when $t \ll t_{\trans}$ and with the mean-field assortative KSBM when $t \gg t_{\trans}$.\\

\begin{lemma}
\label{lemma:consistency-gaussian-assortative-KSBM}
    (Consistency of Gaussian Assortative KSBM)\\
    Suppose that $m \rightarrow \infty$ and $n\sigma \ll \kappa$ for a Gaussian assortative KSBM $\Theta^{A}$, then,
    \begin{itemize}
        \item if $\sigma^2 \ll \kappa$ and $t \gg t_{\trans}$: $\dot{\Theta}^{A}(t) = \dot{\Theta}^{MF}(t)$,\\
        where $\Theta^{MF}$ is the Mean-Field KSBM of Theorem \ref{thm:MF-KSBM}.
        \item if $t \ll t_{\trans}$: $\dot{\Theta}^{A}(t) = \dot{\Theta}^{A,dom}(t)$,\\
        where $\Theta^{A,dom}$ is the dominated Gaussian assortative KSBM as in Theorem \ref{thm:dominated-gaussian-assortative-KSBM}.
    \end{itemize}
\end{lemma}
\begin{proof}
    Appendix \ref{proof:lemma:consistency-gaussian-assortative-KSBM}
\end{proof}

It is important to note that the variance increasing component $\epsilon$ in Theorem \ref{thm:gaussian-assortative-KSBM} is driven by {heterogeneity} in the intrinsic frequency $\sigma$. This intrinsic {heterogeneity} does not disappear when $N \rightarrow \infty$, hence oscillators are not exactly identical in contrast to typical assumptions in mean-field models. In practice, for $\kappa$ large enough or $\sigma$ small enough, this drive is negligible and thus the mean-field assumption holds. We can also show that for a (not necessarily Gaussian) assortative KSBM, $t_{\trans}$ marks the transition to mean-field dynamics (Lemma~\ref{lemma:transition-time-assortative-KSBM}).

\section{Path Signatures for Community Detection}\label{sec:S-community-detection}

We now study the dynamics of the KSBM with path signatures, a structured characterization of paths as an infinite sequence of tensors~\cite{chevyrev_primer_2016,lyons_differential_2007}.
Formally, path signatures are defined as a collection of iterated integrals of a path, but here we will provide an explicit definition in terms of components, i.e.~individual entries of the tensors.

\begin{definition}
    (Path Signatures)\\
    Let $\gamma = (\gamma_1, \ldots, \gamma_N): [0,T] \to \R^N$ be a piecewise smooth path. For all $M \in \N$, a \emph{multi-index} is an ordered sequence $I = (i_1, \ldots, i_M)$ where $i_j \in \{1, \ldots, N\}$. The \emph{path signatures of $\gamma$ with respect to $I$} are
    \[
    S_I(\gamma) = \int_0^T\int_0^{t_M}...\int_0^{t_2}\dot{\gamma}_{i_1}(t_1)...\dot{\gamma}_{i_M}(t_M)\, dt_1...dt_m.
    \]
    We refer to the number $M$ as the \emph{level}.
\end{definition}

A fundamental property of the path signature is that they characterize piecewise smooth paths up to \emph{tree-like equivalence}~\cite{chen_integration_1958}, a generalized notion of reparametrization which includes retracing. 

\subsection{Lead Matrix}

While we state our main results in terms of the full path signatures, we focus our later experimental work on level two signature terms which encode lead-lag dynamics. Suppose $\gamma = (\gamma_1, \ldots, \gamma_N) : [0,T] \to \R^N$ such that $\gamma(0) = 0$. In this case, given $i,j \in [N]$, the level two signature terms have the form
\[
    S_{(i,j)}(\gamma) = \int_0^T \gamma_{i}(t) \dot{\gamma}_j(t) dt.
\]
Such terms can capture lead-lag behavior of time series~\cite{baryshnikov_cyclicity_2016}. Indeed, if $S_{(i,j)}(\gamma) > 0$, this implies that $\gamma_i$ and $\dot{\gamma}_j$ are positively correlated; if $S_{(j,i)}(\gamma) < 0$, then $\gamma_j$ and $\dot{\gamma}_i$ are negatively correlated. Together, these two effects capture that $\gamma_i$ is \emph{leading} $\gamma_j$. Following~\cite{baryshnikov_cyclicity_2016}, we define the \emph{lead matrix}:

\begin{definition}
    (Lead Matrix)\\
    Let $\gamma = (\gamma_1, \ldots, \gamma_N): [0,T] \to \R^N$ be a piecewise smooth path. The \emph{lead matrix} of $\gamma$, $L(\gamma) \in \R^{N \times N}$, is defined by
    \[
        L_{ij}(\gamma) \coloneqq \frac{1}{2} \left( S_{(i,j)}(\gamma) - S_{(j,i)}(\gamma)\right). 
    \]
\end{definition}

The factor of $\frac{1}{2}$ is used to relate the entry $L_{ij}(\gamma)$ of the lead matrix to the \emph{signed area} of the projected path $(\gamma_i, \gamma_j)$; see~\cite{giusti_iterated_2020} for further details.
In our applications to the KSBM, we will consider paths $\gamma = f(\Theta)$ defined by a transformation $f: (S^1)^N \to \R^N$ (or $f: (S^1)^N \to \C^N$) of the Kuramoto dynamics. In particular, we will consider $\gamma = \Theta$ and $\gamma = \sin(\Theta)$. An important remark is that in experiments $\Theta$ is not defined  numerically on $S^1$ but on $\R$. Thus for $\gamma = \Theta$, the transformation $f$ is the lift from $S^1$ to $\R$ where each period around $S^1$ results in an additional $2\pi$ added to the value in $S^1$. We choose $f = \sin$ as this is a conventional projection for angular dynamics in addition to allowing the map from $S^1$ to $\R$ to be continuous.

\subsection{Path Signatures for KSBM}

Path signatures computed from KSBM time series change across different temporal regimes and depend on the underlying community structure of the KSBM. We demonstrate that this dependence can be described, which will be very helpful in constructing community estimators from the path signature later on.

\subsubsection{Regime-Split Path Signatures}

As we saw in the introduction (see Figure~\ref{fig:overview}B), the properties displayed by lead and covariance matrices computed from time series of the KSBM can vary drastically over each regime. For example, community-dependent block patterns are only exhibited in mean-field regimes. In this Section we give theoretical results for lead matrices and path signatures in specific regimes and discuss their dependence on underlying communities. In particular, we suggest to compute the lead matrices over the regime split time series. Since lead matrices are defined as time integrals, they accumulate different patterns over the distinct regimes, in particular during \emph{clusterization}, which can obfuscate community recovery.

We first define the lead matrices and path signatures over time series of a KSBM that we split into regimes.

\begin{definition}
    (Clusterization/Mean-Field Split)\\
    Let $\Theta$ be a realisation of a KSBM $\cK$. Let $t_{\trans}$ be its transition time and $t_{SS}$ be the time by which it is in synchronized steady state. We define the following time regimes:
    \begin{itemize}
        \item $[0, t_{\trans}]$: \textbf{clusterization}, and denote $S^C(\gamma) := S(\gamma|_{[0, t_{\trans}]})$
        \item $[t_{\trans}, t_{SS}]$: \textbf{transient}, and denote $S^{TR}(\gamma) := S(\gamma|_{[t_{\trans}, t_{SS}]})$; and 
        \item $[t_{SS}, t_{SS} + T]$: \textbf{steady state}, and denote $S^{SS}(\gamma)(T) := S(\gamma|_{[t_{SS}, t_{SS} + T]})$.
    \end{itemize}
    We define the lead matrices in these regimes, $L^C, L^{TR}, L^{SS}$, in a similar manner.
\end{definition}

We first show that path signatures\footnote{This convergence also holds for lead matrices since they are expressed as a difference of signatures.} from time series $\Theta$ of an \emph{assortative} KSBM in the mean-field regime (transient or steady state) converge to the path signatures over their community average $\bar{\Theta} := (\theta_{G_1},...,\theta_{G_n})$. 

\begin{lemma}
\label{lemma:convergence-S-assortative-KSBM}
    (Convergence of Path Signatures in the Assortative KSBM)\\
    Let $\Theta$ be a realisation of an assortative KSBM $\cAK$ and $\bar{\Theta} := (\theta_{G_1},...,\theta_{G_n})$ the corresponding community average. Consider $f \in C^2(\prod_{i\in[N]}S^1, \R^N)$ and define $\gamma = f(\Theta)$ and $\bar{\gamma} = f(\bar{\Theta})$. If  $n\sigma \ll \kappa$, then 
    \begin{equation*}
        S^{TR/SS}_{i_1...i_m}(\gamma) \xrightarrow{m \to \infty} S^{TR/SS}_{G_{r_1}...G_{r_m}}(\bar{\gamma}) a.s.,
    \end{equation*}
    for $i_j \in G_{r_j}, r_j \in [n]$ for all $j \in [m], m\geq 1$.
\end{lemma}
\begin{proof}
    Appendix \ref{proof:lemma:convergence-S-assortative-KSBM}
\end{proof}

{In particular, this implies that for any $i \in G_r, j \in G_s$, the lead matrix $L_{ij}(f(\Theta))$ converges to $L_{G_rG_s}(f(\bar{\Theta}))$ as $m \rightarrow \infty$ and $f$ a sufficiently smooth function}.
For lead matrices with $f=\sin$ from the steady state time series of a KSBM (not necessarily assortative) {the first and second moments can be expressed in terms of the community average $\bar{\Theta}$ (Lemma~\ref{lemma:expectation-variance-lead-matrix})}. In particular, the error between the lead matrix of individual oscillators in the KSBM and the community vanishes as $O(\frac{\sigma \omega T}{m})$ (Lemma~\ref{lemma:frequency-noise-lead-matrix}).

In the case of $\gamma = \Theta$ where $f$ is the lift from $S^1$ to $\R$ (or $f=e^{\imath-}$), the steady state lead matrices are zero everywhere (see Lemma \ref{lemma:S-theta} and \ref{lemma:S-exp-i-theta}). In contrast, $f=\sin$ results in steady state lead matrices which are explicitly dependent on the offset between oscillators (Lemma~\ref{lemma:L-sin-theta}), and hence on the community structure, as shown by Lemma \ref{lemma:deviation-synchronized-steady-state}. Lead matrices at steady state can therefore be used to distinguish communities with distinct intrinsic frequencies, i.e. communities that will not have merged into a single synchronized unit at steady state.

\begin{lemma}
\label{lemma:L-sin-theta}
     (Steady State Lead Matrix of the Sinusoid)\\
    Let $\Theta$ a realisation of a KSBM $\cK$ in steady state, then for any $i,j \in [N]$ the offset between oscillators' phase $\Delta\theta_{ij} := \theta_i(t)-\theta_j(t)$ is constant and,
    $$L^{SS}_{ij}(\sin(\Theta)) = \frac{\sin(\Delta\theta_{ij})}{2}(\omega T + \sin(\omega T)).$$
\end{lemma}
\begin{proof}
    Appendix \ref{proof:lemma:L-sin-theta}
\end{proof}

Notice that the lead is subject to fluctuations of $\sin(\omega T)$, hence for small $T$ we expect the community-dependent lead to vanish periodically.

While lead matrices at steady state are able to encode communities with distinct intrinsic frequencies, it is not possible to distinguish communities, which share the same intrinsic frequencies and are therefore merged when synchronized (Lemma \ref{lemma:L-sin-theta}).

In contrast, lead matrices for time series in the transient regime can allow us to distinguish such communities.
Indeed, the transient regime follows the {clusterization where initial phases are uniformly sampled and highly heterogeneous}, and therefore the initial conditions of the mean-field regime for each average community oscillator are random. Hence, communities that would merge in steady state follow distinct trajectories in the transient regime which still depend on the community structure by Lemma \ref{lemma:convergence-S-assortative-KSBM}.

Finally, when considering the time series in the clusterization regime, we know by Lemma \ref{lemma:consistency-gaussian-assortative-KSBM} that Hypothesis \ref{hypothesis:intra-community-dominated} holds, hence the dynamics are approximately independent from the community structure as the intra-community coupling dominates. Thus, the lead accumulated in this regime contributes only to {heterogeneity} with respect to community identification, and therefore must be removed from the time series. 

\subsection{Community Detection}

In this section, we define the \emph{block-clustering metric} for lead matrices, which quantifies the extent to which communities form blocks of homogeneous intra- and distinct inter-community values in the matrix (\emph{block pattern}), and will serve a basis for a community detection algorithm: the structural community estimation algorithm (Alg. \ref{alg:structural-community-estimation}).
We consider a generalization to the full path signatures in Appendix \ref{appendix:sec:generalization-S}.

\subsubsection{Block-Clustering Metric}
In the previous section, we have seen that oscillators in the same community share similar values in the lead matrix, and have distinct values across different communities (for instance, see Lemma~\ref{lemma:L-sin-theta} and ~\ref{lemma:deviation-synchronized-steady-state}).
Our aim is to use these properties to perform community detection. We begin by quantifying the notions of \emph{community homogeneity} and \emph{community discriminativity} for matrices.

{Consider a matrix $B \in \R^{N\times N}$ and a community assignment in the form of a partition $\coprod_{r\in [n]}G_r = [N]$.
We define \emph{community homogeneity} of $B$ as,
$$h(B|G) = \frac{1}{n^2}\sum_{r,s \in [n]}{\Var_{i \in G_r, j\in G_s}(B_{ij})}$$
and \emph{community discriminativity} of $B$ as,
$$d(B|G) = \frac{1}{n^2}\sum_{r,s \in [n]}{(B_{G_rG_s}-B_{G_rG_r})^2 + (B_{G_rG_s}-B_{G_sG_s})^2},$$
where $B_{G_rG_s} = \E[i \in G_r, j\in G_s]{B_{ij}}$.}

We wish to apply community homogeneity and discriminativity to the lead matrix computed from KSBM time series to quantify the extent to which oscillators are clustered. In particular, we wish to detect when homogeneity is \emph{small} and discriminativity is \emph{large}.

\begin{definition}
\label{def:block-clustering}
    (Block Clustering)\\
    For a matrix $B \in \R^{N\times N}$ with community assignment $\coprod_{r\in [n]}G_r = [N]$. If $h(B|G) \neq 0$, we define the \emph{block clustering} of $B$ is,
    $$g(B|G) = \frac{d(B|G)}{h(B|G)}$$
\end{definition}

We note that empirically, $h(B|G) > 0$ due to {heterogeneity}; however one can add a small additive term to obtain a more stable definition if necessary.
If $g(B|G) > 1$, then $d(B|G) > h(B|G)$ which means that the difference in average values between submatrices exceeds the variance within the submatrices themselves. Hence, we say that $B$ is \emph{clustered} when $g(B|G) > 1$. Conversely, if $g(B|G) < 1$, then internal variance dominates, i.e. we cannot effectively distinguish oscillators across communities. Block clustering is {non-negative and scale invariant (Prop.~\ref{prop:clustering})}.\\

Since block clustering is scaling invariant, it is perfectly suited to contrast between lead matrices, regime-split lead matrices and covariance matrices since they live in different ranges of values. In the following section, we consider the maximization of block clustering to perform community detection. However, block clustering increases as the number of communities $n$ increases. Thus, we heuristically employ a normalization $g(-|G)/n$ and study its efficacy in detecting communities in the KSBM.

\subsubsection{Structural Community Estimation Algorithm}
Given any matrix $B \in \R^{N\times N}$, we develop an algorithm to perform community estimation $\hat{G}(B)$ such that the normalized block clustering $g(B|\hat{G}(B))/n$ is maximized. This algorithm begins with one community, iteratively adds communities, uses a method similar to $K$-medoids\cite{kaufman1990} for community assignment, and iterates until the normalized block clustering is maximized. We provide pseudo-code for our algorithm in Alg.\ref{alg:structural-community-estimation}, and summarize the main steps as follows. We use $\hat{G} \coloneqq \hat{G}(B)$ to simplify notation.

\begin{enumerate}
    \item We begin by assuming a single community with constant community assignment $\phi^{(1)}: [N] \to [1]$.
    \item Assume that we have a collection of $k$ communities, $\hat{G}^{(k)}$, with assignment $\phi^{(k)}: [N] \to [k]$, where each community is equipped with a medoid, given by $\xi^{(k)}: [k] \to [N]$.
    \item Let $v_r = (B_{i,r})_{i=1}^N \in \R^N$, and define a distance matrix $D_{ij} = \|v_i - v_j\|_2$.
    \item Choose the most dissimilar pair of nodes contained within the same community
    \[
        (i_*,j_*) = \argmax_{i,j \,: \, \phi^{(k)}(i) = \phi^{(k)}(j)} D_{ij}
    \]
    and assign $i_*$ to be the new medoid of $\phi^{(k)}(i_*)$, and the other to be the medoid of the new community to define $\xi^{(k+1)}: [k+1] \to [N]$.
    \item Define a new community assignment $\phi^{(k+1)}: [N] \to [k+1]$ by proximity to these medoids,
    \[
        \phi^{(k+1)}(i) = \argmin_{r} \|v_i - v_{\xi^{(k+1)}(r)}\|_2.
    \]
    \item If $g(B|\hat{G}^{(k)})/k > g(B|\hat{G}^{(k+1)})/(k+1)$, then return community assignment $\phi^{(k)}$, otherwise, repeat from step 2.
\end{enumerate}

An example of this algorithm is depicted in Figure \ref{fig:algorithm} which is based on the $l_2$ distance matrix $D$.

\begin{figure}[ht!]
\centering
\includegraphics[width=1\linewidth]{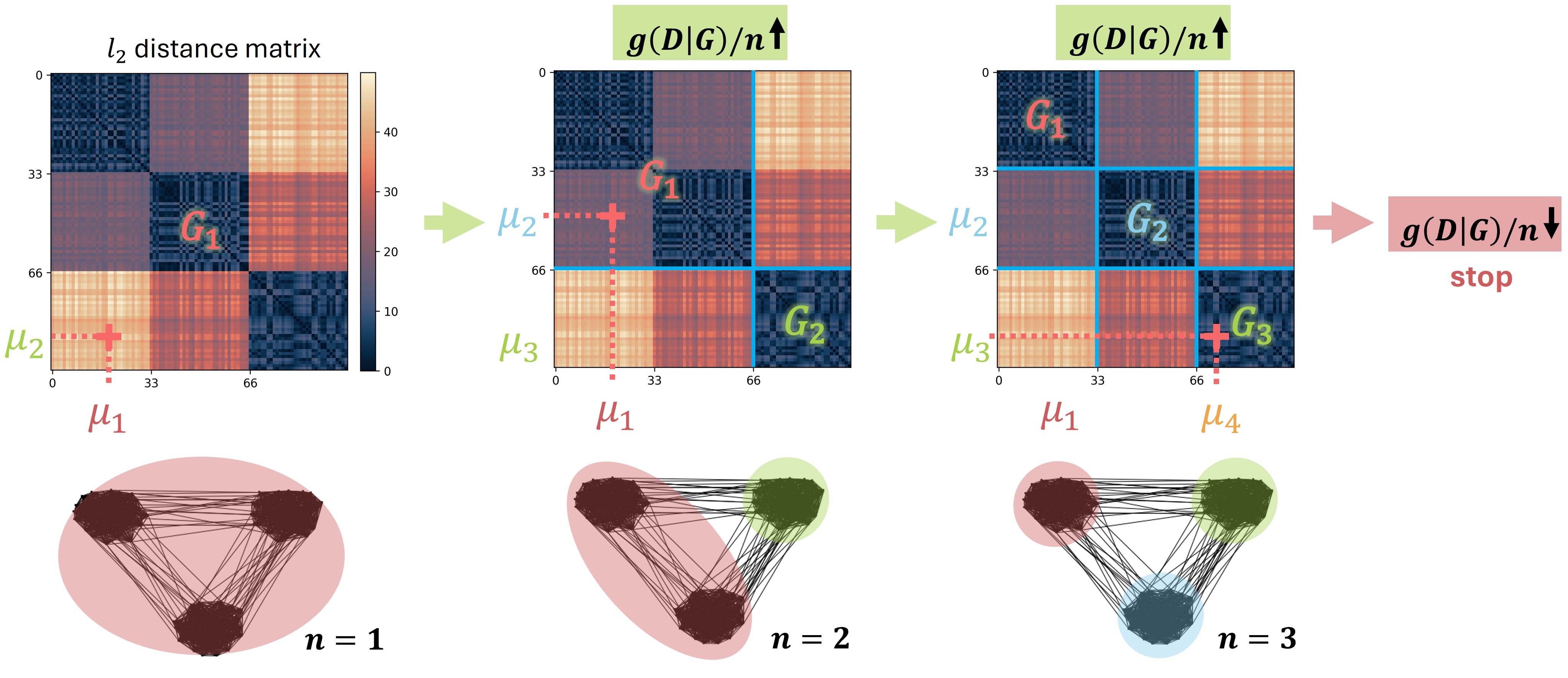}
\caption{\textbf{Clustering algorithm}\\
 Overview of the clustering algorithm (Alg. \ref{alg:structural-community-estimation}) for a lead matrix. Using the most dissimilar pair inside the communities, the algorithm creates new medoids and goes from $n=1,2,3$ communities while increasing $g(D|G)/n$ until it starts decreasing when reaching $n=4$, outputting the communities found at $n=3$.}
\label{fig:algorithm}
\end{figure}

{Our structural community estimation (SCE) algorithm is a variant of the $K$-medoids~\cite{kaufman1990} algorithm with a purity measure~\cite{Tibshirani2001}, the normalized block clustering $g(-|G)/n$, to select the number of communities in the form. In this sense, our algorithm is very close to traditional approaches but uses our additional knowledge of the structure of the lead matrix to empirically determine the number of communities.}

One crucial property for our algorithm to work is that $g(B|G^{(k)})$ increases at each step. Intuitively, this is the case since for every iteration we find the most dissimilar pair of oscillators $(i,j)$ inside of a community $G_r$ and split it into two communities. The new communities thus have lower variances (decreased homogeneity) while also increasing discriminativity (by using the most dissimilar pair as the new medoids).

\section{Data}

\subsection{Numerical Experiments}

We performed all numerical experiments using the time series output of an assortative KSBM with $m = 33$ oscillators per community with different combinations of coupling strength $\kappa$, mean intrinsic frequencies $\omega$, {heterogeneity} in frequency $\sigma$, and number of communities $n=3$ or $n=6$. We measure time $t$ in seconds $s$ and all frequencies in $rad/s$. The input to computing our lead matrices are linearly interpolated time series which we obtain from a discretization of $500$ time steps within time interval $[0,10]s$ with the exception of Noisy and Large KSBM (see below).\\
\\
We consider four different configurations:
\begin{enumerate}
    \item \textbf{Standard KSBM}: $n=3$ communities, $\sigma=0.1$ (low {heterogeneity}) and distinct mean intrinsic frequencies $\mu_r$ evenly spaced points in the range $[\frac{2}{3}, 2 \ rad/s]$, $\kappa=100$ (strong coupling), simulated up to $t = 10s$;
    \item \textbf{Collapsed KSBM}: same configuration as the standard KSBM, but we fix the mean intrinsic frequencies in all communities to be $\mu_r=\frac{2}{3}$;
    \item \textbf{Noisy KSBM}: $n=3$ communities, $\sigma=1$ (high {heterogeneity}), distinct mean intrinsic frequencies $\mu_r$ evenly spaced points in the range $[\frac{1}{3}, 1 \ rad/s]$, $\kappa=10$ (weak coupling), simulated up to $t = 50s$;
    \item \textbf{Large KSBM}: same configuration as the standard KSBM but with $n=6$ communities, mean intrinsic frequencies $\mu_r$  evenly spaced points in the range $[\frac{1}{6}, 1 \ rad/s]$, and simulated dynamics up to $t = 19s$.
\end{enumerate}
We identified splitting times for the regimes visually and checked them against the numerical prediction (see Figure\ref{fig:transition-times} and \ref{fig:dom-id-GKSBM}), specifically the {transition time} from the Gaussian KSBM for each coupling strength and number of clusters.

\subsection{Neural Recordings}
\label{ssec:neural-recordings}

{We applied our SCE algorithm to neural recordings of mice performing a sensory-motor task\cite{Esmaeili2021} in which they need to lick in response to an audio cue if they received a whisker stimulus one second before (Figure~\ref{fig:sensory-motor-task}A).}

{We focused on a single session with neuropixel probes in the wS1 and ALM cortical areas, for a total of 93 units and 568 trials. The data consists of spike times for each area which were binned in $dt=2ms$ steps, and further processed with a normalized exponential decay filter with intrinsic time $\tau = 40ms$ (Figure~\ref{fig:sensory-motor-task}B,D). The time series that we used to compute the lead and covariance matrices were the concatenation of the filtered spike trains of all trials\footnote{Since the filter size is short compared to the time between trial and the original signal is binary, the order of concatenation does not change the resulting lead matrices.}.}

{A trial is labeled \emph{Hit} (\emph{Miss}) if the mouse licks (does not lick) in response to a whisker stimulus and audio cue. A trial is labeled \emph{False Alarm (FA)} (\emph{Correct Rejection (CR)}) if the mouse licks (does not lick)  in response to only an audio cue.}

\section{Results} \label{sec:experiments}

\subsection{KSBM Dynamics}

The dynamics and full/regime-split lead matrices for each configuration of the assortative KSBM are depicted in Figure\ref{fig:KSBM-dynamics-standard},\ref{fig:KSBM-dynamics-collapsed}-\ref{fig:KSBM-dynamics-large}.

\begin{figure}[!ht]
\centering
\includegraphics[width=0.9\linewidth]{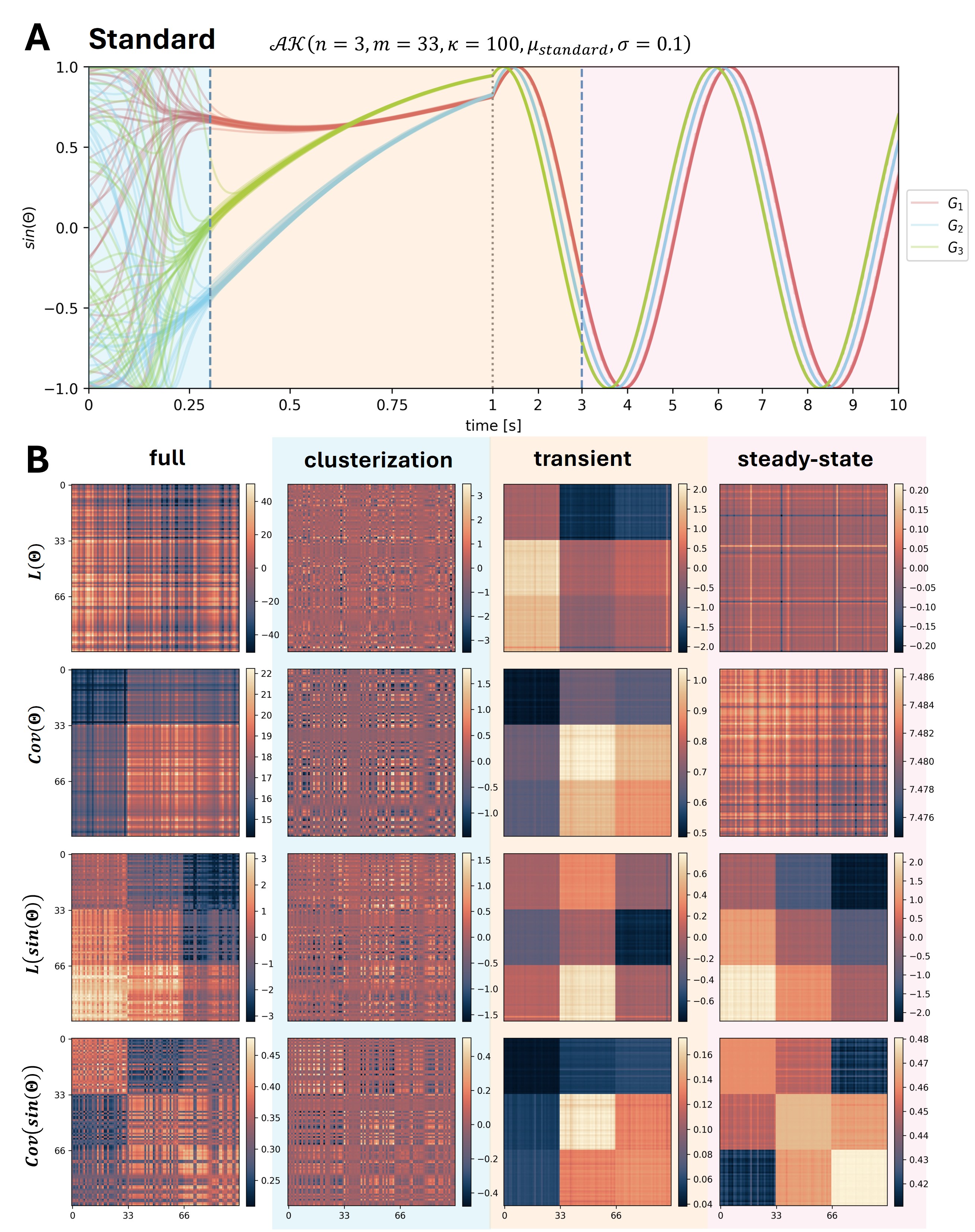}
\caption{\textbf{Standard KSBM time series and lead matrices}\\
(A) Time series output from the standard KSBM $\cAK(n=3, m=33, \kappa=100, \mu=\{\frac{2}{3}, \frac{14}{9}, 2 ~rad/s\}, \sigma=0.1)$ with communities $G_1, G_2, G_3$. (B) Lead and Covariance matrices for the full time series and each time regime for both $\Theta$ and $\sin(\Theta)$. All matrices are clustered in the transient regime, but differ in pattern from the steady state, while clusterization only contributes {heterogeneity}.}
\label{fig:KSBM-dynamics-standard}
\end{figure}

Across all reasonable coupling strengths $\kappa \geq 10$, the KSBMs (see Figure\ref{fig:KSBM-dynamics-standard}A,\ref{fig:KSBM-dynamics-collapsed}A-\ref{fig:KSBM-dynamics-large}A) follow our empirically estimated split into clusterization and mean-field regimes (first blue dashed line) in their time series. We visually determine the time point at which clusterization ends and use it to split the time series from which we compute the lead and covariance matrices. In the next subsection, we shall confirm those splits through {transition times} $t_{\trans}$. For lower coupling strengths, e.g., in the noisy KSBM, the mean-field regime can only be approximated by a mean-field KSBM as the minimal variance in the communities is not vanishingly small (see conditions of Theorem \ref{thm:MF-KSBM}). At high coupling strength $\kappa=100$, we observe the second split (second blue dashed line) which marks the dynamics entering a synchronized steady state.\\
\\
When we split the time series into clusterization and mean-field regimes, we observe that the lead and covariance matrices are {heterogeneous} during clusterization while they exhibit a block pattern in the mean-field regimes (see Figure\ref{fig:KSBM-dynamics-standard}B,\ref{fig:KSBM-dynamics-collapsed}B-\ref{fig:KSBM-dynamics-large}B). If we further split the mean-field regime into transient and steady state, the block pattern sometimes disappears in steady state for $L(\Theta)$ or in all matrices for the Collapsed KSBM (see Figure~\ref{fig:KSBM-dynamics-collapsed}B).
Those disappearing block patterns in steady state can be explained for $L^{SS}(\Theta)$ by Lemma \ref{lemma:S-theta}. In the Collapsed KSBM, by Lemma \ref{lemma:L-sin-theta}, $L^{SS}(\sin(\Theta))$ is proportional to the offset between oscillators which is zero at steady state when all oscillators are synchronized in phase. Transient dynamics do not directly depend on those differences in frequencies (otherwise an anti-diagonal gradient would also appear), but are expressive in both $\Theta$ and $\sin(\Theta)$. The persistent block patterns in the transient regime is independent of the steady state matrices not exhibiting any block pattern, e.g., in the Collapsed KSBM.

\subsection{Gaussian and Mean-Field Models} \label{ssec:gauss_mf}

We compared our time series of our different versions of the assortative KSBM against the mean-field (Theorem \ref{thm:MF-KSBM}) and (dominated) Gaussian (see Theorem \ref{thm:dominated-identical-gaussican-assortative-KSBM}, \ref{thm:dominated-gaussian-assortative-KSBM} \& \ref{thm:gaussian-assortative-KSBM}) KSBM predictions. We first evaluate the variance prediction of the dominated and identical Gaussian KSBM (see Theorem \ref{thm:dominated-identical-gaussican-assortative-KSBM} for definition; Figure~\ref{fig:dom-id-GKSBM}), and then cover how the mean-field KSBM (Figure~\ref{fig:MF-KSBM}) matches up to the KSBM in the mean-field regime. Finally, we show that our Gaussian KSBM (Figure~\ref{fig:Gauss-KSBM}) captures both variance and mean of the oscillators' phases, i.e. the predicted values match the time series closely in all regimes. As a result of the early {high-variance} dynamics the predicted dynamics can enter the mean-field regime at different time points in the time series. These time points, however, are very close to each other.

\subsubsection{Transition Times and Dominated Gaussian KSBM}

\begin{figure}[ht!]
  \centering
  \includegraphics[width=1\linewidth]{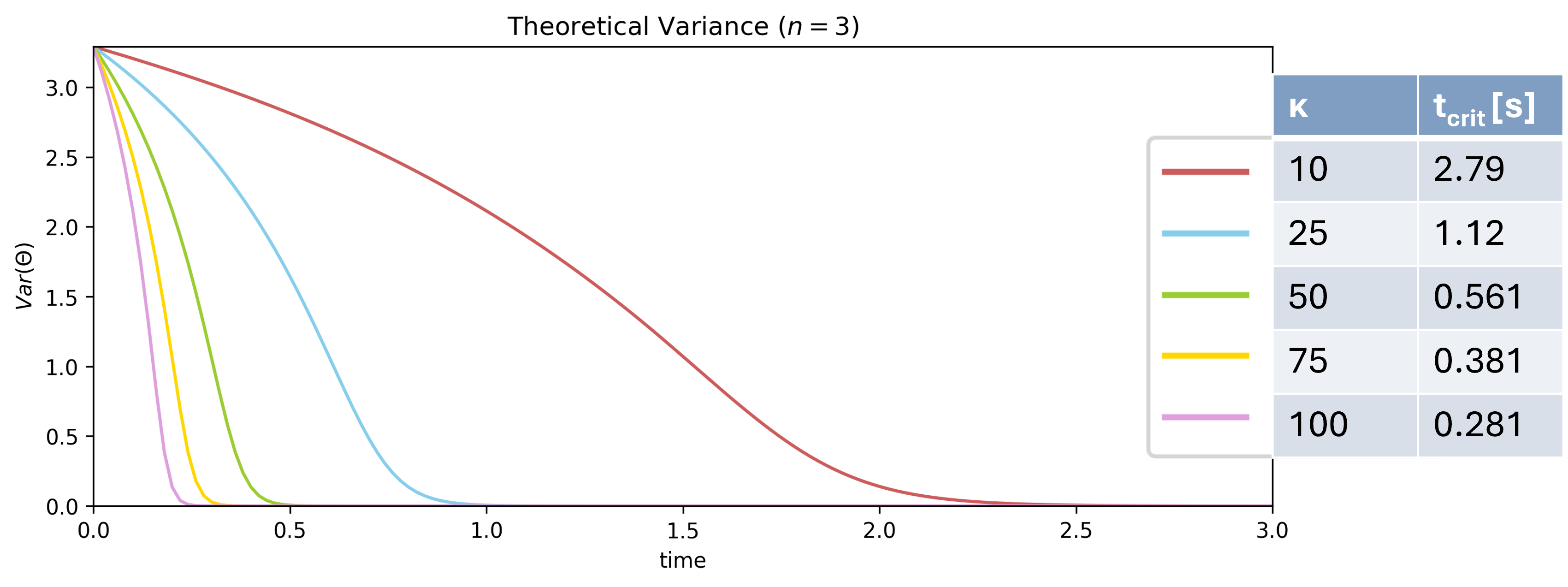}
  \caption{\textbf{Transition times}\\
  Dominated and identical Gaussian KSBM theoretical variance in time for various coupling $\kappa$ and $n=3$ with initial variance $\frac{\pi^2}{3}$, the square of the variance of the uniform distribution on $S^1$.}
  \label{fig:transition-times}
\end{figure}

\begin{figure}[ht!]
  \centering
  \includegraphics[width=1\linewidth]{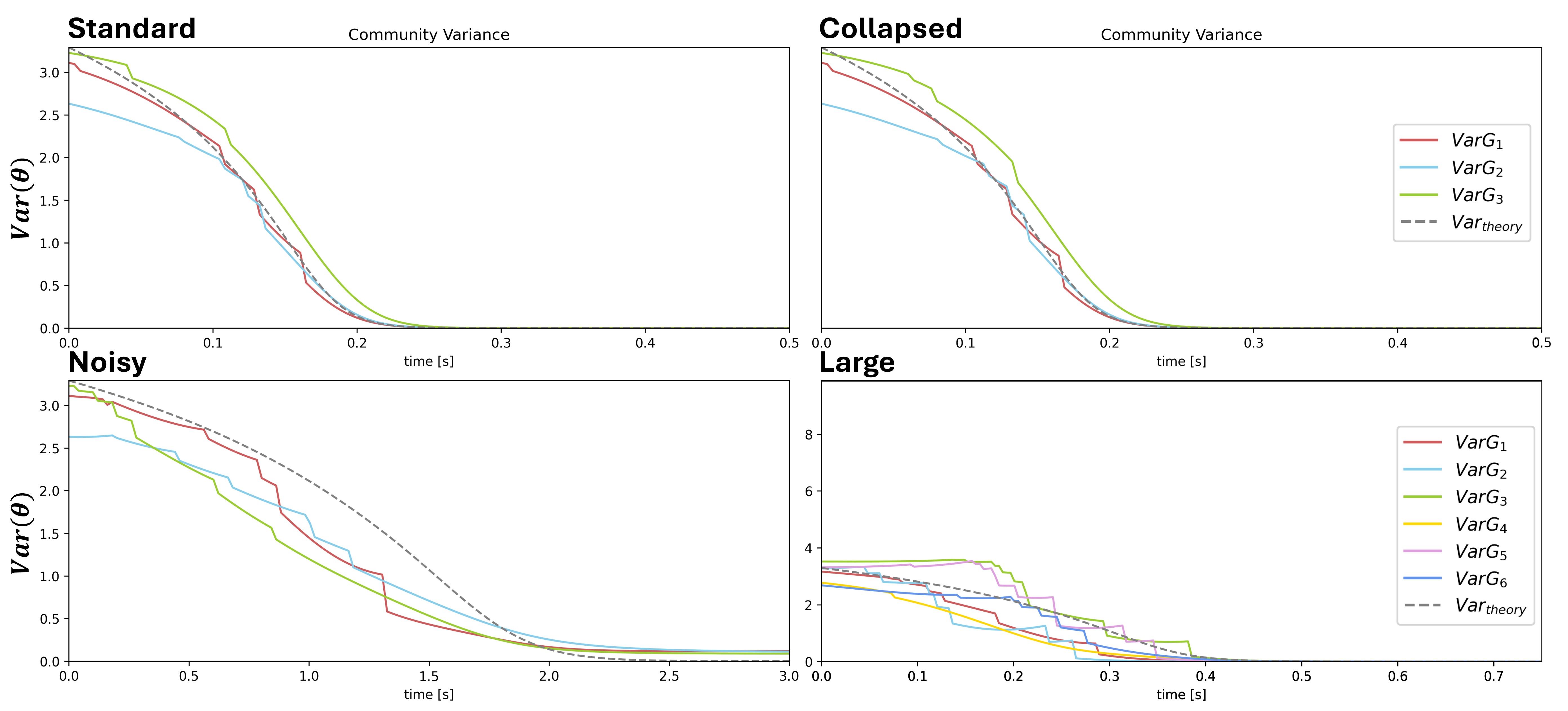}
  \caption{\textbf{Dominated and identical Gaussian KSBM}\\
  Dominated and identical Gaussian KSBM $\Var_{\text{theory}}$ match the collapse of the communities variances $\Var G_\cdot$ in all KSBM configurations: (clockwise) standard, collapsed, noisy, large. Transition times are  $t_{\trans} \approx 0.281, 0.281, 2.79, 0.558s$ respectively.}
  \label{fig:dom-id-GKSBM}
\end{figure}

We computed the dominated (and identical) Gaussian KSBM for different values of $\kappa$, $n=3$ or $6$ with initial variance $\frac{\pi^2}{3}$ (this corresponds to the square of the critical variance at which a Gaussian is approximately uniform). We show their resulting time series and the corresponding {transition times} $t_{\trans}$ in Figure~\ref{fig:transition-times}. We overlay the theoretical variance with the empirical variance of the dynamics of each KSBM considered in Figure~\ref{fig:dom-id-GKSBM}.\\
\\
The collapse estimated by our dominated (and identical) Gaussian KSBM matches up in time with the KSBM dynamics. Hence, our {transition times} (Figure~\ref{fig:transition-times}) agree with the transition point estimated in Figure\ref{fig:dom-id-GKSBM}.

\subsubsection{Mean-Field KSBM}

\begin{figure}[ht!]
  \centering
  \includegraphics[width=1\linewidth]{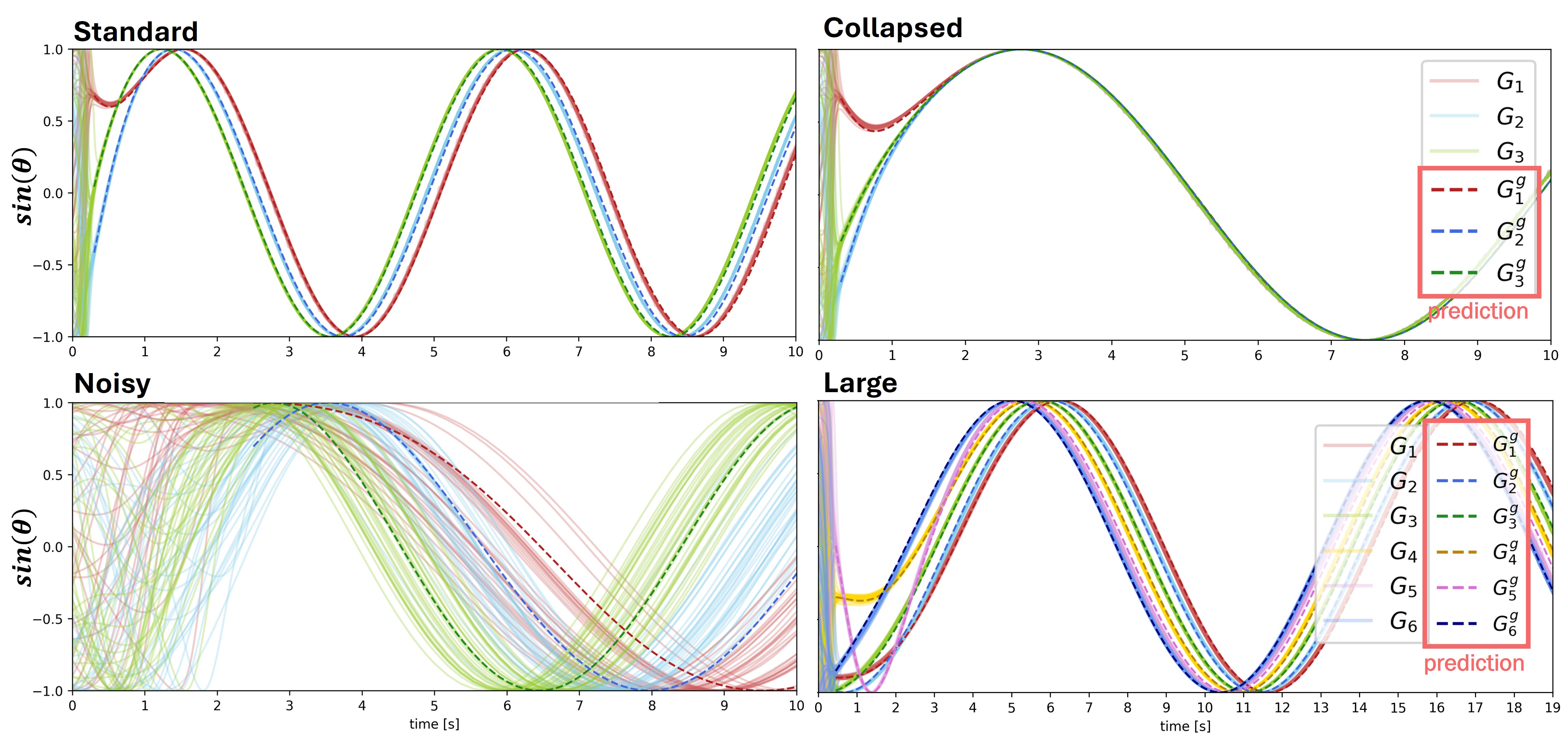}
  \caption{\textbf{Mean-field KSBM}\\
  Mean-Field KSBM (MF) time series plotted against the KSBM dynamics from Figure~\ref{fig:KSBM-dynamics-standard},\ref{fig:KSBM-dynamics-collapsed}-\ref{fig:KSBM-dynamics-large}, started at the empirical {transition times} for each KSBM configurations: standard, collapsed, noisy and large (clockwise).}
  \label{fig:MF-KSBM}
\end{figure}

Recall that the mean-field KSBM needs to be initialized when $\Theta$ is sufficiently clustered by Theorem \ref{thm:MF-KSBM}. Therefore, for the starting point of the mean-field KSBM, we chose the empirical end of clusterization times of the dynamics. As we saw in the previous subsection, these agree with the {transition time} $t_{\trans}$ established in Lemma \ref{lemma:transition-time-assortative-KSBM}.\\
\\
In Figure~\ref{fig:MF-KSBM}, we show the simulated mean-field KSBM of Theorem \ref{thm:MF-KSBM} with the assortative KSBM dynamics. As per Theorem \ref{thm:MF-KSBM}, the KSBM dynamics and mean-field KSBM dynamics agree whenever the communities are sufficiently clustered. In the Noisy KSBM, the mean-field KSBM predictions don't match as closely as the other models. This can be explained by the prediction not accounting for {heterogeneity} and $m$ not being sufficiently large. In particular, the mean-field KSBM is still a good fit up to time $t = 10s$ for weakly clustered $\Theta$ as is the case with the high-{heterogeneity}, weak coupling of the noisy KSBM. This may imply that the $\epsilon \ll 1$ clustered condition of Theorem \ref{thm:MF-KSBM} can be relaxed.

\subsubsection{Gaussian KSBM}

\begin{figure}[ht!]
  \centering
  \includegraphics[width=1\linewidth]{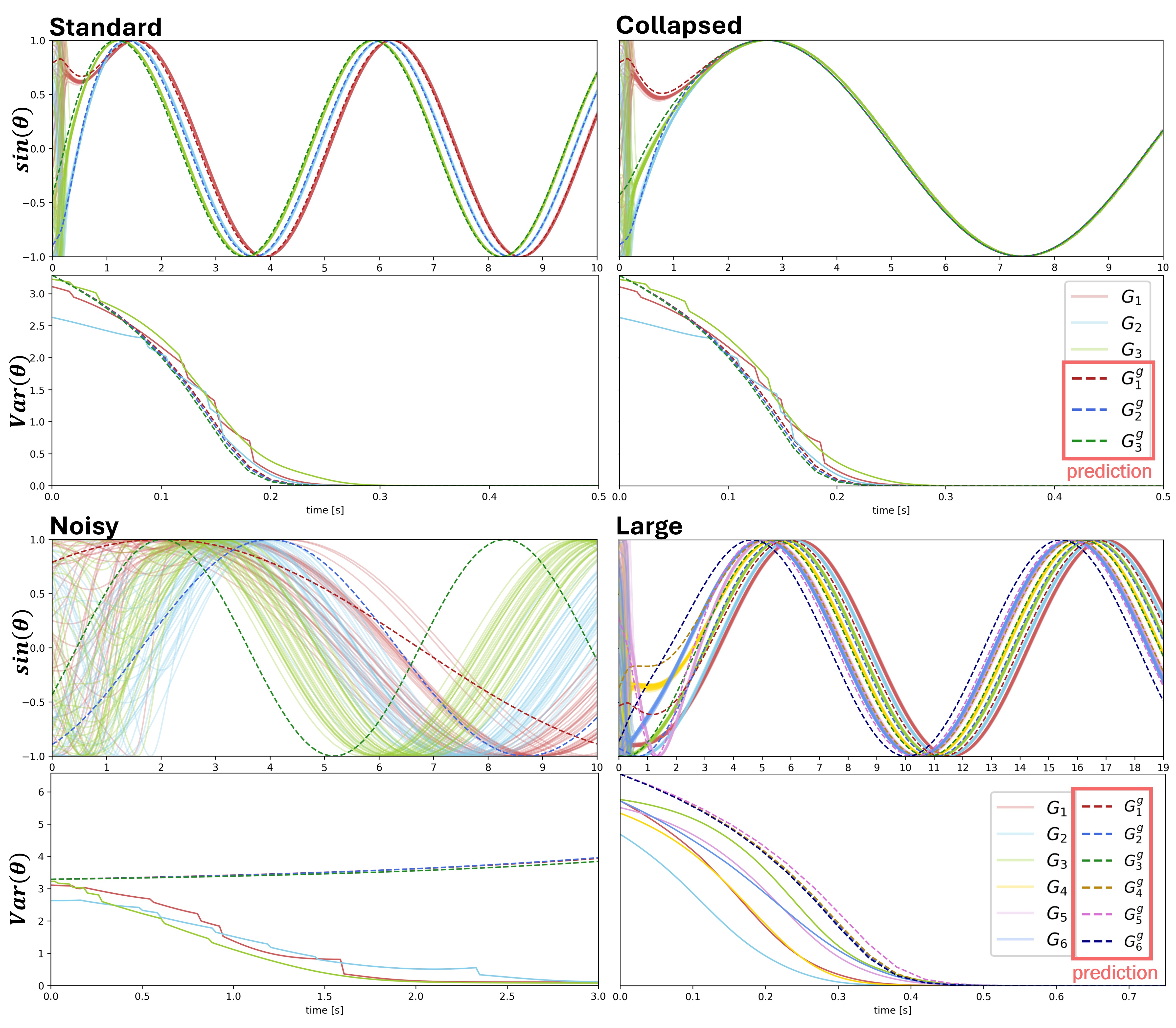}
  \caption{\textbf{Gaussian KSBM}\\
  Gaussian KSBM for each configuration: (clockwise) standard, collapsed, noisy and large. Top: Gaussian KSBM mean oscillators $\theta_{G_\cdot}$ (labelled $G^g_\cdot$, dashed line) against KSBM time series of Figure~\ref{fig:KSBM-dynamics-standard},\ref{fig:KSBM-dynamics-collapsed}-\ref{fig:KSBM-dynamics-large} (labelled $G_\cdot$). Bottom: Variance variables of Gaussian KSBM ($\Var G^g_\cdot$, dashed line) against the true community variances ($\Var G_\cdot$). Notice that the variance for noisy Gaussian KSBM does not vanish.}
  \label{fig:Gauss-KSBM}
\end{figure}

We show the Gaussian KSBM simulations starting from $t=0s$ in Figure~\ref{fig:Gauss-KSBM}. We observe that our bound on the {heterogeneity-driven} $\epsilon$ of Theorem \ref{thm:gaussian-assortative-KSBM} is too high in the noisy KSBM, and thus, while the noisy KSBM is clustered to some extent, the Gaussian KSBM does not capture it and remains in the uniform maximal entropy stage. While this bound allows us to give a sufficient condition for the clusterization occurring, it fails to predict it does not occur.

\subsection{Recovering Community Structure} \label{ssec:exp_community}

We limit ourselves to evaluating our structural community estimation algorithm on regime-split time series. We show the block clustering $g(-|G)$ of the lead and covariance matrices for the regime-split and full time series in Figure~\ref{fig:regime-split-clustering}. We further show the agreement $A(G,G_{est})$ between the true label $G$ and estimated label $G_{est}$ of our algorithm in Figure~\ref{fig:regime-split-agreement}. Recall that the minimal agreement for community estimation on $n$ communities is $\frac{1}{n}$, which corresponds to random assignment. Since our estimation may contain more communities than the true number of communities, the agreement may fall below this threshold. If the agreement is below the threshold, recovery fails completely.

\begin{figure}[ht!]
  \centering
  \includegraphics[width=1\linewidth]{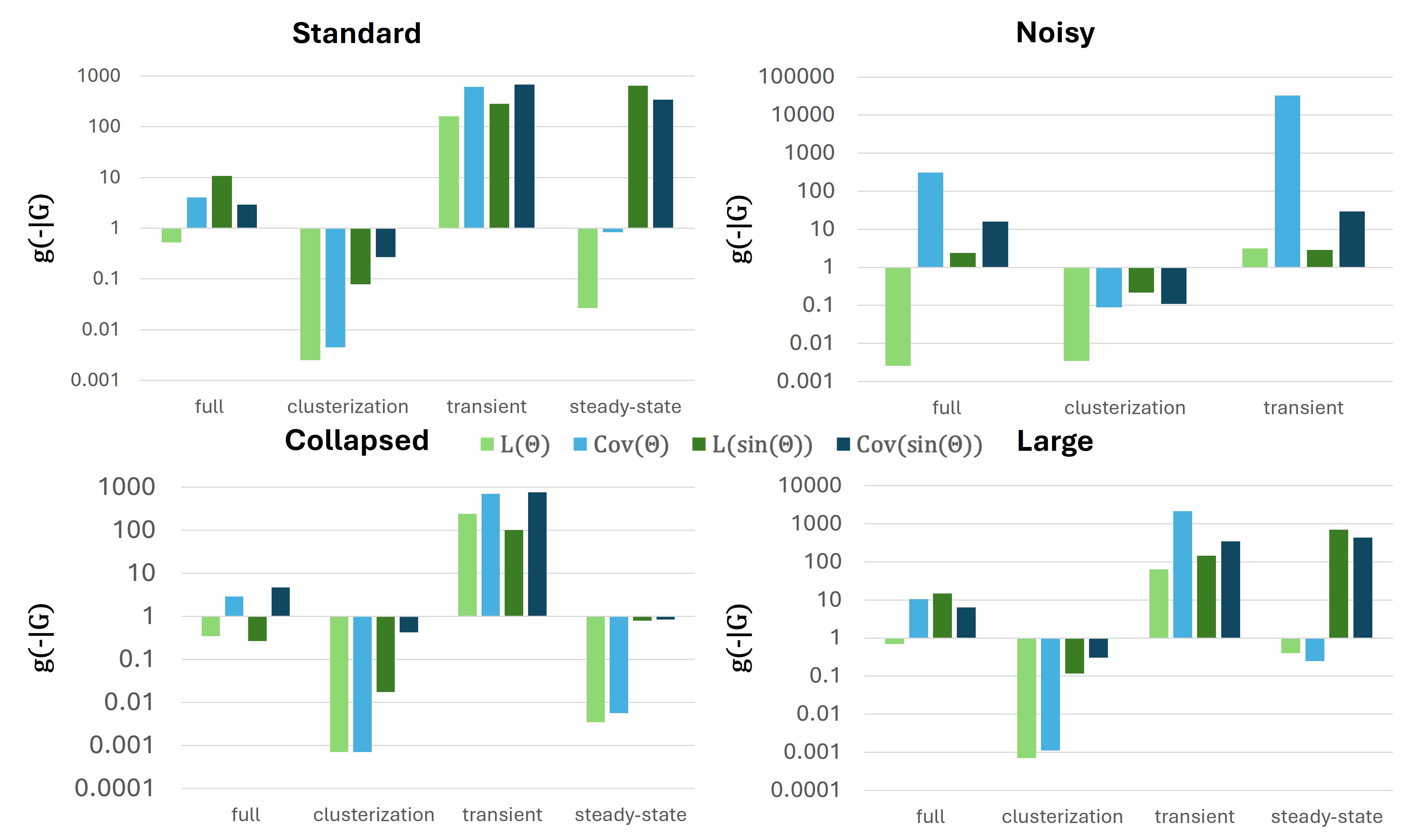}
  \caption{\textbf{Clustering of KSBMs}\\
  Block clustering $g(-|G)$ of the lead and covariance matrices over the full time series and time regime splits for each KSBM. The matrices for $\sin(\Theta)$ are consistently clustered (value larger than $1$) in the transient regime and steady state (save for noisy). As expected, noisy and zero matrices are weakly and non-clustered respectively.}
  \label{fig:regime-split-clustering}
\end{figure}
\begin{figure}[ht!]
  \centering
  \includegraphics[width=1\linewidth]{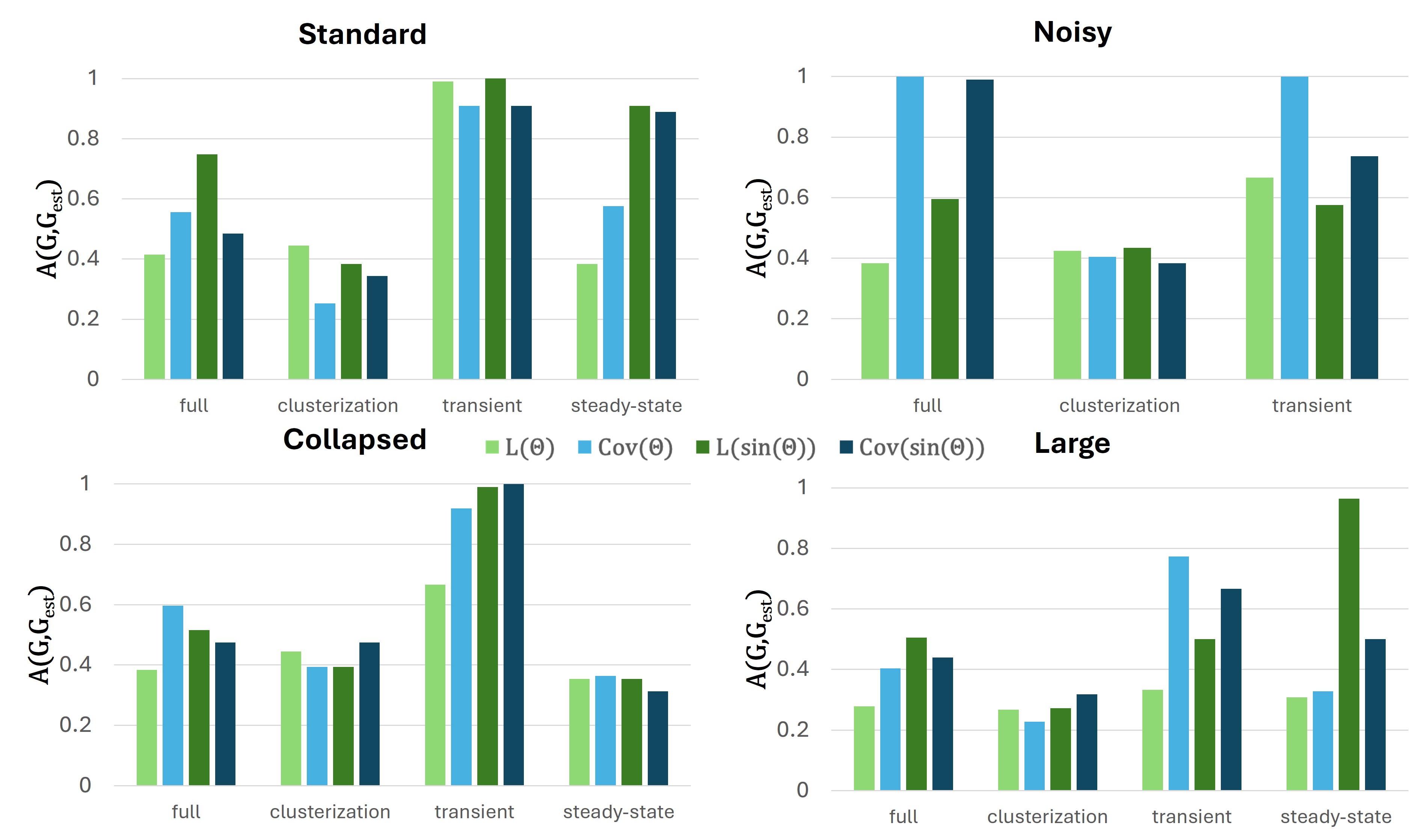}
  \caption{\textbf{Performance of the structural community estimation algorithm in regime-split KSBMs}\\
  Agreement $A(G,G_{est})$ between the true community assignment $G$ and structural community identification algorithm $G_{est}$ over the full dynamic and each regime-split for each KSBM. Exact recovery corresponds to an agreement of $1$ while for random assignment of the true number of communities ($n$) the agreement is $1/n$, but can be lower if the number of communities estimated is larger than $n$. Exact-recovery is possible in all transient regimes for at least one of the matrices, in general $L(\sin(\Theta))$.}
  \label{fig:regime-split-agreement}
\end{figure}

We observe exact community recovery in several cases for all KSBMs (see Figure~\ref{fig:regime-split-agreement}). In particular, the communities are exactly recovered in the transient regime for a subset of lead and covariance matrices. Notably, $L(\sin(\Theta))$ is suited for recovery in the transient regime with reasonable {heterogeneity} $\sigma$ ($n\sigma^2 \ll \kappa$, see Lemma \ref{lemma:consistency-gaussian-assortative-KSBM}), but fails completely if {heterogeneity} is too high. As expected, recovering communities at steady state is impossible whenever communities merge together, as is the case with the collapsed KSBM (see Figure~\ref{fig:KSBM-dynamics-collapsed}). Further, community estimation fails in the clusterization regime since all lead and covariance matrices consist only of {heterogeneity}. We also remark that the block clustering of the matrices is a good indication of whether the algorithm will be able to recover the communities, at least partially. However, high values of block clustering do not guarantee exact recovery. High block-clustering can arise from very clear block patterns in the lead or covariance matrices, where one block corresponds to several communities. Hence, recovery will fail partially to separate such communities.\\
\\
Overall, we see that our analytical regime split of the time series guarantees the feasibility of community recovery in the corresponding lead and covariance matrices. In particular, we observe an increase of performance of the algorithm on the regime split in comparison to the full time series.

\subsection{Hierarchical KSBM}
{We investigate the performance of our SCE algorithm on a two level hierarchical community structure similar to the one studied by Arenas \emph{et al.}~\cite{arenas2006-sync-top-comp-net}. In fact, their Kuramoto model which is coupled by an underlying hierarchical network can be formulated as a submodel of the KSBM. See Appendix~\ref{appendix:sec:hierarchical-ksbm} for details and results.}

\subsection{Stochastic KSBM}
\label{ssec:stochastic-ksbm}
{Before tackling real data - where a block structure might not be present - we investigate an intermediary model, which includes more realistic noise. Specifically, we consider a version of the KSBM with Brownian noise added to the oscillator dynamics. We evaluate the performance of our SCE and traditional clustering algorithm on community estimation in the standard KSBM configuration for different levels of Brownian noise scaling $b$ ranging from low ($b=0.1$) to high ($b=5$). See Appendix~\ref{appendix:sec:stochastic-ksbm} for details and results.}

\subsection{Sensory-Motor Task}
\label{ssec:sensory-motor-task}
{We apply our method to real neural recordings of mice, as described in Section~\ref{ssec:neural-recordings}. In this setting, we expect the neurons to be highly heterogeneous in their activity as those areas are involved in many processes. In particular, the lead and covariance matrices do not display a block structure with respect to cortical areas or types (Figure~\ref{fig:sensory-motor-task}C). Although the recording involves the primary area wS1, all the animal's whiskers except one are clipped such that only this remaining whisker can trigger activity in wS1. This could explain the heterogeneity within area and absence of a separation between most wS1 and ALM neurons as only some of the wS1/ALM neurons are specific to this single whisker/task and the rest are only weakly to non-responding (Figure~\ref{fig:sensory-motor-task}B).}

\begin{figure}[ht!]
  \centering
  \includegraphics[width=1\linewidth]{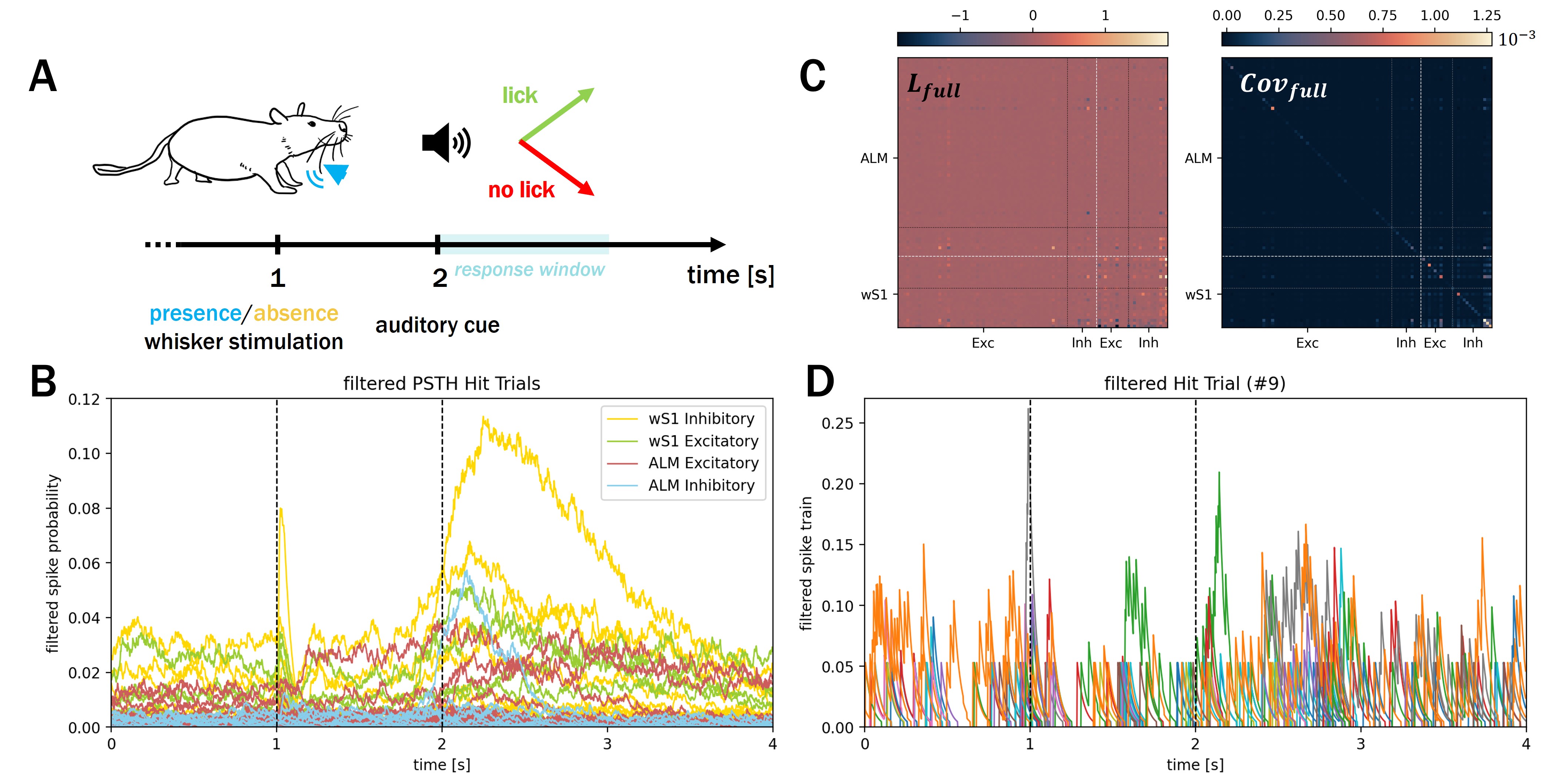}
  \caption{\textbf{Sensory-Motor Task}\\
  (A) Sensory-motor task with whisker stimulation and auditory cue. Trials start at $t = 0s$, whisker stimulus and audio cue occur at $t=1s$ and $t=2s$ respectively with a response-window of one second. (B) Peri-Stimulus Time Histogram (PSTH) for Hit trials of each neuron filtered spike trains labeled by cortical areas and type. (C) Lead and covariance matrices for the full concatenated time series. (D) Example filtered spike trains of a Hit trial.}
  \label{fig:sensory-motor-task}
\end{figure}

{We did not expect our SCE to perform particularly well on identifying cortical areas or neural types since the block-clustering metric is specifically tailored to block structure. Indeed, the resulting clusters (Figure~\ref{fig:sce-clusters}A) are small and contain a mix of different brain areas. The largest cluster includes a large portion of inactive neurons with close to zero values in the lead or covariance matrices. We therefore consider the neurons which are part of this largest cluster to be classified as non-specific to the task by our SCE algorithm. For simplicity, we refer to the non-maximal clusters (and their corresponding neurons) as being identified by the algorithm, i.e., they are classified as being different to non-specific neurons.}

{Nonetheless, the clusters identified by the SCE correspond to neurons, which exhibit a distinct activity pattern compared to the non-specific majority. For the lead matrix, cluster \#2 identifies a neuron, which consistently spikes during whisker stimulation and has sustained activity after the audio cue. Cluster \#4 picks up two neurons, which are very active after the audio cue and then slowly recover to their baseline. Finally, cluster \#3 neurons have a mix of both trends. In particular, the neurons identified correspond to peaks in modulus of the first few eigenvectors of the lead matrix (Figure~\ref{fig:sce-eigenmodulus}) which are related to neurons involved in cyclical components in Cyclicity Analysis\cite{baryshnikov_cyclicity_2016}.}

{For the covariance matrix, only a small cluster of two neurons is identified with a similar trend as in the previously mentioned cluster \#4. If we look at the signals that are classified as non-specific by the SCE algorithm (Figure~\ref{fig:sce-clusters}B), we observe that our algorithm misses signal where the PSTH response to the whisker stimulation is a transient dip in the baseline instead of an increase. Similarly, ALM excitatory neurons with sustained regular activity after whisker stimulation are not identified.}

\begin{figure}[ht!]
  \centering
  \includegraphics[width=1\linewidth]{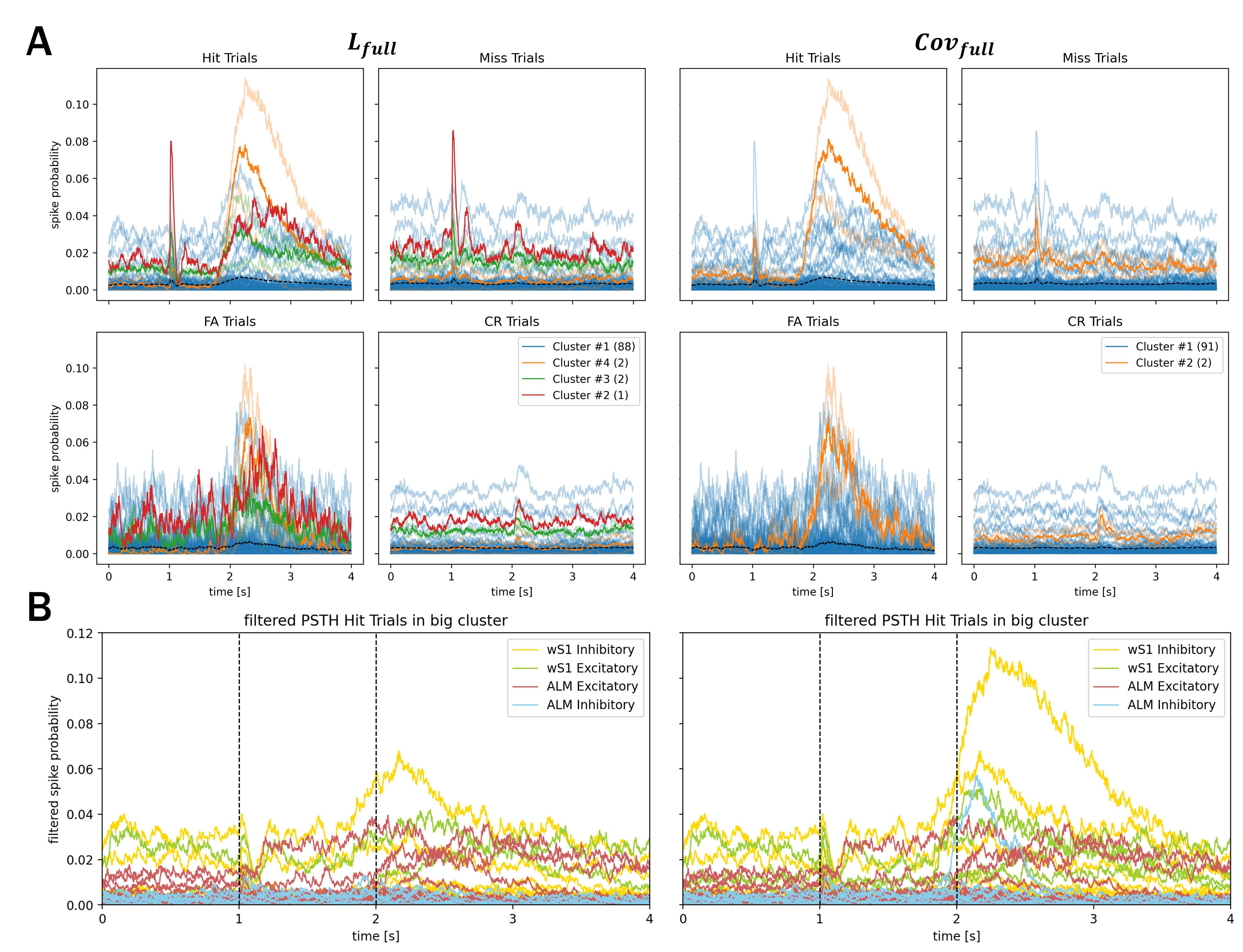}
  \caption{\textbf{SCE Clusters}\\
  (A) PSTH for each trial type labeled with clusters identified by the SCE algorithm on (Left) lead and (Right) covariance matrices, non-transparent line correspond to the average PSTH within cluster and dashed line to the population average PSTH. (B) Filtered PSTH of neurons belonging to the largest (trivial) cluster from the SCE algorithm on (Left) lead and (Right) covariance matrices.}
  \label{fig:sce-clusters}
\end{figure}

{Traditional algorithms such as $K$-means and hierarchical clustering are not able to identify cortical areas or neural types from both matrices (for $K$-means see Figure ~\ref{fig:algo-k3-clusters},\ref{fig:algo-k6-clusters}), but capture trends in the neural activities.}

{For a larger number of clusters, $K=6$, $K$-means recovers most of the neurons with structured activity during the task using the lead matrix, whereas the clusters identified in the covariance matrix are smaller and miss a large number of active excitatory neurons in ALM and wS1. The choice of matrix also leads to different features being used to form the clusters. Typically, the covariance matrix tends to result in small clusters which separate the neurons based on their time-averaged activity (see the Miss and CR PSTH). In contrast, the lead matrix assigns more importance to co-occurring transient increases in activity which results in clusters with more than one neuron. For $K=3$, the results are similar to the SCE algorithm for the lead matrix.}

\section{Discussion}

Community structures are ubiquitous in neuroscience and give rise to community-dependent neural dynamics. 
In this work, we find that separating temporal regimes of dynamics in community-structured networks is crucial to structure inference when using the lead matrices (and path signatures in general), whereas common approaches to path signatures would typically use the full time series.
Different temporal regimes exhibit dynamics, which have distinct dependencies on the underlying community structure.
We first showed this differentiated dependence of the community structure with an analytical study using mean-field and Gaussian approximations of the various regimes (see Section~\ref{sec:KSBM-dynamics}).
We then leveraged these theoretical insights to develop a novel algorithm, \emph{structural community estimation}, to recover communities empirically from both lead and covariance matrices. In particular, we have shown that performance increases when time series are split over the theoretically predicted regimes. {This increase is specifically relevant to lead matrices and (path signatures in general) since traditional applications tend to compute them on entire time series risking erasing transient structure.}
When extending these results to experimental neuroscience settings, our results suggest that connectivity recovery should incorporate time series in non-steady state settings to distinguish between functionally equivalent but structurally distinct regions.
This could be accomplished in practice by global randomized stimulation or inhibition, and in the specific example of the KSBM, through randomized pulses to the oscillators.
Our theoretical results have also shown the ability of path signatures and lead matrices to detect community structure, as their expression in the steady state regime explicitly depends on this structure.
Our mean-field and Gaussian approximations are also relevant to other dynamical community estimation methods. Those lower-dimensional models expressed directly in terms of the communities allow for the analytical study using other methods.

Related investigations on Kuramoto oscillators have not studied different temporal regimes systematically. For example, Tirabassi \emph{et al.} \cite{tirabassi2015} test and compare similarity measures for time series that can successfully infer structural connectivity. Specifically, they study cross-correlation, mutual information, and mutual information of the time series ordinal patterns. The authors apply the three measures to construct functional networks from three different quantities derived from time series output of coupled Kuramoto and Rössler oscillators. The success of their recovery of structural connections depends on the specific quantities derived from the time series, interaction strengths of oscillators in the models, and the choice of threshold for the functional networks. In contrast, our approach does not require the choice of a threshold. Moreover, Tirabassi \emph{et al.} also only consider temporal effects by truncating time series to 0.25 of their full length. Finally, the authors studied different cases of homogeneous network architectures rather than SBMs on only up to 50 oscillators.

Our observations on the temporal regimes are in line with related work in other systems. Das \emph{et al.} \cite{Das2020} showed that recovery of connectivity from time series of recurrent spiking neural networks for an imposed connectivity pattern was highly impacted by the temporal regime in which the network was observed. In particular, the authors showed experimentally that steady state regimes lead to non-vanishing bias in the recovered connectivity due to functionally equivalent connectivity patterns. This effect was worsened when only observing a fraction of the neurons. They demonstrated that perturbations to the network, such as pushing it transiently into another regime led to time series from which an unbiased connectivity pattern could be inferred. This is similar to what we have observed in the case of the collapsed KSBM, where communities can be recovered from the time series in the transient mean-field regime but not at steady state. Similar observations were made for other systems \cite{rubido2014}. Our approach allows for those same conclusions to be drawn from theoretical results in Kuramoto models.

Our analytical work is also related to several recent directions in the theoretical study of Kuramoto models. For example, Nagpal \emph{et al.} \cite{nagpal2024} study synchronization in the continuum limit of random graphs in terms of graphons equipped with sufficient regularity conditions. As suggested in that work, it would be interesting to investigate graphons of lower regularity in order to include the case of KSBMs. In another direction, the stochastic Kuramoto model~\cite{Sonnenschein2013-mh} incorporates temporal randomness in the intrinsic frequency by modeling it as Gaussian white noise. It would be interesting to connect this to work on using the theory of rough paths (where path signatures play a fundamental role) to study the continuum limit of stochastic interacting particle systems~\cite{cass_evolving_2015}. 

Our experimental observations are limited to the four versions of the assortative KSBM that we considered. We in particular restricted ourselves to balanced communities with inter-community couplings that are scaled identically. An interesting direction for future study would be to systematically investigate effects of coupling strengths and different numbers/sizes of communities.
We would also like to extend the Gaussian assortative KSBM and some of our results which are limited to the assortative KSBM to the general KSBM. Such extensions would allow us to properly study the effects of varying inter-community coupling.

{Compared to the work in Arenas \emph{et al.}~\cite{arenas2006-sync-top-comp-net} on the Kuramoto model with underlying hierarchical networks, our findings on intra-community synchronization followed by inter-community synchronization can be interpreted as a hierarchical structure where the inter-community couplings form a second level community. This phenomenon explains the clear block structure in our lead and covariance matrices as well as the correlation matrices Arenas \emph{et al.}~\cite{arenas2006-sync-top-comp-net} used to estimate communities. In Appendix~\ref{appendix:sec:hierarchical-ksbm}, we showed that our SCE algorithm is able to recover most of the community structure. In contrast to Arenas \emph{et al.}~\cite{arenas2006-sync-top-comp-net}, we did not average time series over different initializations; our method works on a single instance of a time series. This is especially relevant since there exist time series where the varying time structure would not allow for such averaging (such as random licking events across different neural recordings).}

{To assess the potential of using our proposed methods for the analysis of biological data, we first showed that the SCE and traditional algorithms based on lead and covariance matrices are robust to medium scale noise directly included in the oscillators of a model that exhibits block structure (see Appendix~\ref{appendix:sec:stochastic-ksbm}). Moreover, recovering community structure works significantly better ($p < 0.05$) on lead matrices compared to covariance matrices computed on the full time series. We observe this for all clustering algorithms that we considered}. It would be interesting to explore effects of missing data, e.g., the robustness of our results when only data from a fraction of oscillators within the communities is available. Identifying thresholds at which exact or partial community recovery is impossible under partial observation would be particularly relevant. 

{In our analysis of neural data, we saw that for a simple sensory-motor task no block structure exists with respect to cortical areas or neural types, due to only a handful of neurons in each cortical area exhibiting structured activity. Nonetheless, our SCE algorithm is still able to pick out some functional clusters, which separate different activity trends without any prior knowledge of the number of clusters. While traditional algorithms such as $K$-means are also viable, the explicit choice of the number of communities is arbitrary and can result in very different clusters. In particular, we observe that the lead matrices (compared to covariance matrices) result in more interesting clusters which focus on different features of neural activity such as transient peaks of activity.}

In our current investigations, we have not studied our algorithm analytically. We have shown intuitively and experimentally why it results in the maximization of the block-clustering or why high values of this metric result in communities being only partially recoverable. We leave the conditions for exact or partial recovery of communities with our algorithm as future work.
Another extension would be to combine community estimates from the different regimes and to create a combined estimate in our algorithm. At present, while communities can be recovered exactly in some regimes, the algorithm does not tell us which of the community estimates is correct among the different matrices and regimes.
Finally, the presented oscillator models can be brought closer to neural dynamics by considering spike-threshold dependent coupling. 
Other models of oscillators which are closer to plausible neural activity could also be considered, such as Rössler oscillators (see \cite{tirabassi2015}) or the FitzHugh-Nagumo model\cite{Nagumo1962} of a neuron's membrane potential. Finally, effects of negative coupling would be interesting to investigate as they correspond to inhibitory connections that are often found in biological systems.

\section*{Code Availability}

All the code used to simulate KSBM and generate the figures is available on \url{https://github.com/arthurion98/KSBM-path-signatures}.

\section*{Acknowledgements}
\label{sec:Acknowledgements}

We are grateful to Léo Lévy for making the code from his Bachelor thesis available to us. We further want to thank Kathryn Hess Bellwald, Walter Senn, Kelly Maggs, Tim Gentner, Johanni Brea, and Mason Porter for helpful comments, discussions and/or feedback over the course of the project. We specifically thank Carl Petersen and Vahid Esmeaili for the neural recordings.  Finally, we thank EPFL IT-SV for their technical support during the project. TJN was supported by the Swiss National Science Foundation grants 200020-207426 and 200021-236436. DL was supported by the Hong Kong Innovation and Technology Commission (InnoHK Project CIMDA).

\bibliographystyle{plain}
\bibliography{literature, kuramoto_darrick}

\begin{thebibliography}{10}

\bibitem{Abbe2018}
Emmanuel Abbe.
\newblock Community detection and stochastic block models: Recent developments.
\newblock {\em Journal of Machine Learning Research}, 18(177):1--86, 2018.

\bibitem{Achterhof2021-I}
Stefan Achterhof and Janusz~M Meylahn.
\newblock Two-community noisy kuramoto model with general interaction
  strengths. {I}.
\newblock {\em Chaos}, 31(3):033115, 2021.

\bibitem{Achterhof2021-II}
Stefan Achterhof and Janusz~M Meylahn.
\newblock Two-community noisy kuramoto model with general interaction
  strengths. {II}.
\newblock {\em Chaos}, 31(3):033116, 2021.

\bibitem{arenas_synchronization_2007}
Alex Arenas and Albert {D{\'i}az-Guilera}.
\newblock Synchronization and modularity in complex networks.
\newblock {\em The European Physical Journal Special Topics}, 143(1):19--25,
  2007.

\bibitem{arenas_synchronization_2008}
Alex Arenas, Albert {D{\'i}az-Guilera}, Jurgen Kurths, Yamir Moreno, and
  Changsong Zhou.
\newblock Synchronization in complex networks.
\newblock {\em Physics Reports}, 469(3):93--153, 2008.

\bibitem{arenas_synchronization_2006}
Alex Arenas, Albert {D{\'i}az-Guilera}, and Conrad~J {P{\'e}rez-Vicente}.
\newblock Synchronization processes in complex networks.
\newblock {\em Physica D: Nonlinear Phenomena}, 224(1):27--34, 2006.

\bibitem{arenas2006-sync-top-comp-net}
Alex Arenas, Albert {D{\'i}az-Guilera}, and Conrad~J {P{\'e}rez-Vicente}.
\newblock Synchronization reveals topological scales in complex networks.
\newblock {\em Physical Review Letters}, 96(11), 2006.

\bibitem{baryshnikov_cyclicity_2016}
Yuliy Baryshnikov and Emily Schlafly.
\newblock Cyclicity in multivariate time series and applications to functional
  {{MRI}} data.
\newblock In {\em 2016 {{IEEE}} 55th {{Conference}} on {{Decision}} and
  {{Control}} ({{CDC}})}, pages 1625--1630, 2016.

\bibitem{bassett2013robust}
Danielle~S Bassett, Mason~A Porter, Nicholas~F Wymbs, Scott~T Grafton, Jean~M
  Carlson, and Peter~J Mucha.
\newblock Robust detection of dynamic community structure in networks.
\newblock {\em Chaos}, 23(1):013142, 2013.

\bibitem{boccaletti_detecting_2007}
Stefano Boccaletti, Mikhail Ivanchenko, Vito Latora, Alessandro Pluchino, and
  Andrea Rapisarda.
\newblock Detecting complex network modularity by dynamical clustering.
\newblock {\em Physical Review E}, 75(4):045102, 2007.

\bibitem{Bullmore2009}
Ed~Bullmore and Olaf Sporns.
\newblock Complex brain networks: graph theoretical analysis of structural and
  functional systems.
\newblock {\em Nature Reviews Neuroscience}, 10(3):186--198, 2009.

\bibitem{Bullmore2011}
Edward~T Bullmore and Danielle~S Bassett.
\newblock Brain graphs: graphical models of the human brain connectome.
\newblock {\em Annual Review of Clinical Psychology}, 7(1):113--140, 2011.

\bibitem{cass_evolving_2015}
Thomas Cass and Terry Lyons.
\newblock Evolving communities with individual preferences.
\newblock {\em Proceedings of the London Mathematical Society}, 110(1):83--107,
  2015.

\bibitem{chen_integration_1958}
Kuo-Tsai Chen.
\newblock Integration of paths -- a faithful representation of paths by
  noncommutative formal power series.
\newblock {\em Trans. Amer. Math. Soc.}, 89(2):395--407, 1958.

\bibitem{chevyrev_primer_2016}
Ilya Chevyrev and Andrey Kormilitzin.
\newblock A primer on the signature method in machine learning.
\newblock {\em arXiv:1603.03788}, 2016.

\bibitem{Chiba2019-et}
Hayato Chiba and Georgi~S Medvedev.
\newblock The mean field analysis of the {K}uramoto model on graphs {I}. the
  mean field equation and transition point formulas.
\newblock {\em Discrete and Continuous Dynamical Systems}, 39(1):131--155,
  2019.

\bibitem{Das2020}
Abhranil Das and Ila~R Fiete.
\newblock Systematic errors in connectivity inferred from activity in strongly
  recurrent networks.
\newblock {\em Nature Neuroscience}, 23(10):1286–1296, 2020.

\bibitem{Esmaeili2021}
Vahid Esmaeili, Keita Tamura, Samuel~P. Muscinelli, Alireza Modirshanechi,
  Marta Boscaglia, Ashley~B. Lee, Anastasiia Oryshchuk, Georgios Foustoukos,
  Yanqi Liu, Sylvain Crochet, Wulfram Gerstner, and Carl~C.H. Petersen.
\newblock Rapid suppression and sustained activation of distinct cortical
  regions for a delayed sensory-triggered motor response.
\newblock {\em Neuron}, 109(13):2183--2201.e9, 2021.

\bibitem{Favaretto2017-wd}
Chiara Favaretto, Angelo Cenedese, and Fabio Pasqualetti.
\newblock Cluster synchronization in networks of {K}uramoto oscillators.
\newblock {\em IFAC-PapersOnLine}, 50(1):2433--2438, 2017.

\bibitem{frank2000}
Till~D Frank, Andreas Daffertshofer, C~E Peper, Peter~J Beek, and Hermann
  Haken.
\newblock Towards a comprehensive theory of brain activity: Coupled oscillator
  systems under external forces.
\newblock {\em Physica D: Nonlinear Phenomena}, 144(1):62--86, 2000.

\bibitem{friz_multidimensional_2010}
Peter~K Friz and Nicolas~B Victoir.
\newblock {\em Multidimensional {{Stochastic Processes}} as {{Rough Paths}}:
  {{Theory}} and {{Applications}}}.
\newblock Cambridge {{Studies}} in {{Advanced Mathematics}}. Cambridge
  University Press, 2010.

\bibitem{Gerstner2014}
Wulfram Gerstner, Werner~M Kistler, Richard Naud, and Liam Paninski.
\newblock {\em Neuronal Dynamics: From Single Neurons to Networks and Models of
  Cognition}.
\newblock Cambridge University Press, 2014.

\bibitem{Gine2013}
Jaume Giné.
\newblock The stable limit cycles: A synchronization phenomenon.
\newblock {\em Journal of the Franklin Institute}, 350(7):1649--1657, 2013.

\bibitem{giusti_iterated_2020}
Chad Giusti and Darrick Lee.
\newblock Iterated {{integrals}} and {{population time series analysis}}.
\newblock In Nils~A. Baas, Gunnar~E. Carlsson, Gereon Quick, Markus Szymik, and
  Marius Thaule, editors, {\em Topological {{Data Analysis}}}, pages 219--246,
  Cham, 2020. Springer International Publishing.

\bibitem{Kaas2001}
Jon~H Kaas.
\newblock Topographic maps in the brain.
\newblock In Neil~J. Smelser and Paul~B. Baltes, editors, {\em International
  Encyclopedia of the Social \& Behavioral Sciences}, pages 15771--15775.
  Pergamon, Oxford, 2001.

\bibitem{Kaliuzhnyi-Verbovetskyi2018-qr}
Dmitry Kaliuzhnyi-Verbovetskyi and Georgi~S Medvedev.
\newblock The mean field equation for the {K}uramoto model on graph sequences
  with non-{L}ipschitz limit.
\newblock {\em SIAM Journal on Mathematical Analysis}, 50(3):2441--2465, 2018.

\bibitem{kassabov2021}
Martin Kassabov, Steven~H Strogatz, and Alex Townsend.
\newblock Sufficiently dense {K}uramoto networks are globally synchronizing.
\newblock {\em Chaos: An Interdisciplinary Journal of Nonlinear Science},
  31(7), 2021.

\bibitem{kaufman1990}
Leonard Kaufman and Peter~J Rousseeuw.
\newblock {\em Partitioning around medoids (Program PAM)}, chapter~2, pages
  68--125.
\newblock John Wiley \& Sons, Ltd, 1990.

\bibitem{Langdon2023-cr}
Christopher Langdon, Mikhail Genkin, and Tatiana~A Engel.
\newblock A unifying perspective on neural manifolds and circuits for
  cognition.
\newblock {\em Nature Review Neuroscience}, 24(6):363--377, 2023.

\bibitem{lee_signature_2023}
Darrick Lee and Harald Oberhauser.
\newblock The {{signature kernel}}.
\newblock {\em arXiv:2305.04625}, 2023.

\bibitem{levnajic2011}
Zoran Levnaji{\'c} and Arkady Pikovsky.
\newblock Network reconstruction from random phase resetting.
\newblock {\em Physical review letters}, 107(3):034101, 2011.

\bibitem{Leo2023}
Leo Levy.
\newblock Kuramoto model and path-signatures.
\newblock Technical report, EPFL, 2023.

\bibitem{li_synchronization_2008}
Daqing Li, Inmaculada Leyva, Juan~A Almendral, Irene {Sendi{\~n}a-Nadal},
  Javier~M Buld{\'u}, Shlomo Havlin, and Stefano Boccaletti.
\newblock Synchronization {{Interfaces}} and {{Overlapping Communities}} in
  {{Complex Networks}}.
\newblock {\em Physical Review Letters}, 101(16):168701, 2008.

\bibitem{Luke2012}
Tanushree Luke, Ernest Barreto, and Paul So.
\newblock A dynamical study of pulse-coupled oscillators in the brain.
\newblock {\em BMC Neuroscience}, 13(S1), 2012.

\bibitem{lyons_differential_2007}
Terry Lyons, Michael Caruana, and Thierry L{\'e}vy.
\newblock {\em Differential {{equations driven}} by {{rough paths}}}.
\newblock {\'E}cole d'{{{\'E}t{\'e} Probabibilit\'e}} {{St}}.-{{Flour}}.
  Springer-Verlag, Berlin Heidelberg, 2007.

\bibitem{lyons_signature_2022}
Terry Lyons and Andrew~D McLeod.
\newblock Signature {{methods}} in {{machine learning}}.
\newblock {\em arXiv:2206.14674}, 2022.

\bibitem{Menara2019-ok}
Tommaso Menara, Giacomo Baggio, Danielle~S Bassett, and Fabio Pasqualetti.
\newblock Exact and approximate stability conditions for cluster
  synchronization of {K}uramoto oscillators.
\newblock In {\em 2019 American Control Conference ({ACC})}. IEEE, 2019.

\bibitem{Menara2020-ul}
Tommaso Menara, Giacomo Baggio, Danielle~S Bassett, and Fabio Pasqualetti.
\newblock Stability conditions for cluster synchronization in networks of
  heterogeneous {K}uramoto oscillators.
\newblock {\em IEEE Transaction on Control of Network Systems}, 7(1):302--314,
  2020.

\bibitem{nagpal2024}
Shriya~V Nagpal, Gokul~G Nair, Steven~H Strogatz, and Francesca Parise.
\newblock Synchronization in random networks of identical phase oscillators: A
  graphon approach.
\newblock {\em arXiv:2403.13998}, 2024.

\bibitem{Nagumo1962}
Jinichi Nagumo, Suguru Arimoto, and Shuji Yoshizawa.
\newblock An active pulse transmission line simulating nerve axon.
\newblock {\em Proceedings of the IRE}, 50(10):2061–2070, 1962.

\bibitem{nakamura2016}
Tomomichi Nakamura, Toshihiro Tanizawa, and Michael Small.
\newblock Constructing networks from a dynamical system perspective for
  multivariate nonlinear time series.
\newblock {\em Physical Review E}, 93(3):032323, 2016.

\bibitem{papo2014}
David Papo, Massimiliano Zanin, Jos{\'e}~Angel Pineda-Pardo, Stefano
  Boccaletti, and Javier~M Buld{\'u}.
\newblock Functional brain networks: great expectations, hard times and the big
  leap forward.
\newblock {\em Philosophical Transactions of the Royal Society B: Biological
  Sciences}, 369(1653):20130525, 2014.

\bibitem{petersen2015}
Steven~E Petersen and Olaf Sporns.
\newblock Brain networks and cognitive architectures.
\newblock {\em Neuron}, 88(1):207--219, 2015.

\bibitem{Purves2001}
Dale Purves, George~J Augustine, David Fitzpatrick, C~Katz~Lawrence,
  Anthony-Samuel Lamantia, James~O McNamara, and S~Mark Williams.
\newblock {\em Neuroscience 2nd edition}, chapter The Auditory Cortex.
\newblock Sinauer Associates, 2001.

\bibitem{rodrigues_kuramoto_2016}
Francisco~A Rodrigues, Thomas K~DM Peron, Peng Ji, and J{\"u}rgen Kurths.
\newblock The {{Kuramoto}} model in complex networks.
\newblock {\em Physics Reports}, 610:1--98, 2016.

\bibitem{rubido2014}
Nicol{\'a}s Rubido, Arturo~C Mart{\'\i}, Ezequiel Bianco-Mart{\'\i}nez, Celso
  Grebogi, Murilo~S Baptista, and Cristina Masoller.
\newblock Exact detection of direct links in networks of interacting dynamical
  units.
\newblock {\em New Journal of Physics}, 16(9):093010, 2014.

\bibitem{shandilya2011}
Srinivas~Gorur Shandilya and Marc Timme.
\newblock Inferring network topology from complex dynamics.
\newblock {\em New Journal of Physics}, 13(1):013004, 2011.

\bibitem{Sonnenschein2013-mh}
Bernard Sonnenschein and Lutz Schimansky-Geier.
\newblock Approximate solution to the stochastic kuramoto model.
\newblock {\em Phys. Rev. E Stat. Nonlin. Soft Matter Phys.}, 88(5):052111,
  2013.

\bibitem{sporns2015}
Olaf Sporns.
\newblock {\em Graph-theoretical analysis of brain networks}.
\newblock Elsevier, 2015.

\bibitem{Stolz2017-lw}
Bernadette~J Stolz, Heather~A Harrington, and Mason~A Porter.
\newblock Persistent homology of time-dependent functional networks constructed
  from coupled time series.
\newblock {\em Chaos}, 27(4):047410, 2017.

\bibitem{sun2014}
Xiaoran Sun, Michael Small, Yi~Zhao, and Xiaoping Xue.
\newblock Characterizing system dynamics with a weighted and directed network
  constructed from time series data.
\newblock {\em Chaos: An Interdisciplinary Journal of Nonlinear Science},
  24(2), 2014.

\bibitem{Tibshirani2001}
Robert Tibshirani, Guenther Walther, and Trevor Hastie.
\newblock Estimating the number of clusters in a data set via the gap
  statistic.
\newblock {\em Journal of the Royal Statistical Society Series B: Statistical
  Methodology}, 63(2):411–423, 2001.

\bibitem{timme2007}
Marc Timme.
\newblock Revealing network connectivity from response dynamics.
\newblock {\em Physical review letters}, 98(22):224101, 2007.

\bibitem{timme2014}
Marc Timme and Jose Casadiego.
\newblock Revealing networks from dynamics: an introduction.
\newblock {\em Journal of Physics A: Mathematical and Theoretical},
  47(34):343001, 2014.

\bibitem{timofeyev_cluster_2025}
Tobias Timofeyev and Alice Patania.
\newblock Cluster {{Synchronization}} via {{Graph Laplacian Eigenvectors}}.
\newblock {\em arXiv:2503.18978}, March 2025.

\bibitem{tirabassi2015}
Giulio Tirabassi, Ricardo Sevilla-Escoboza, Javier~M Buld{\'u}, and Cristina
  Masoller.
\newblock Inferring the connectivity of coupled oscillators from time-series
  statistical similarity analysis.
\newblock {\em Scientific reports}, 5(1):10829, 2015.

\bibitem{townsend2020}
Alex Townsend, Michael Stillman, and Steven~H Strogatz.
\newblock Dense networks that do not synchronize and sparse ones that do.
\newblock {\em Chaos: An Interdisciplinary Journal of Nonlinear Science},
  30(8), 2020.

\end{thebibliography}

\appendix
\section{Figures}

\begin{figure}[H]
\centering
\includegraphics[width=0.92\linewidth]{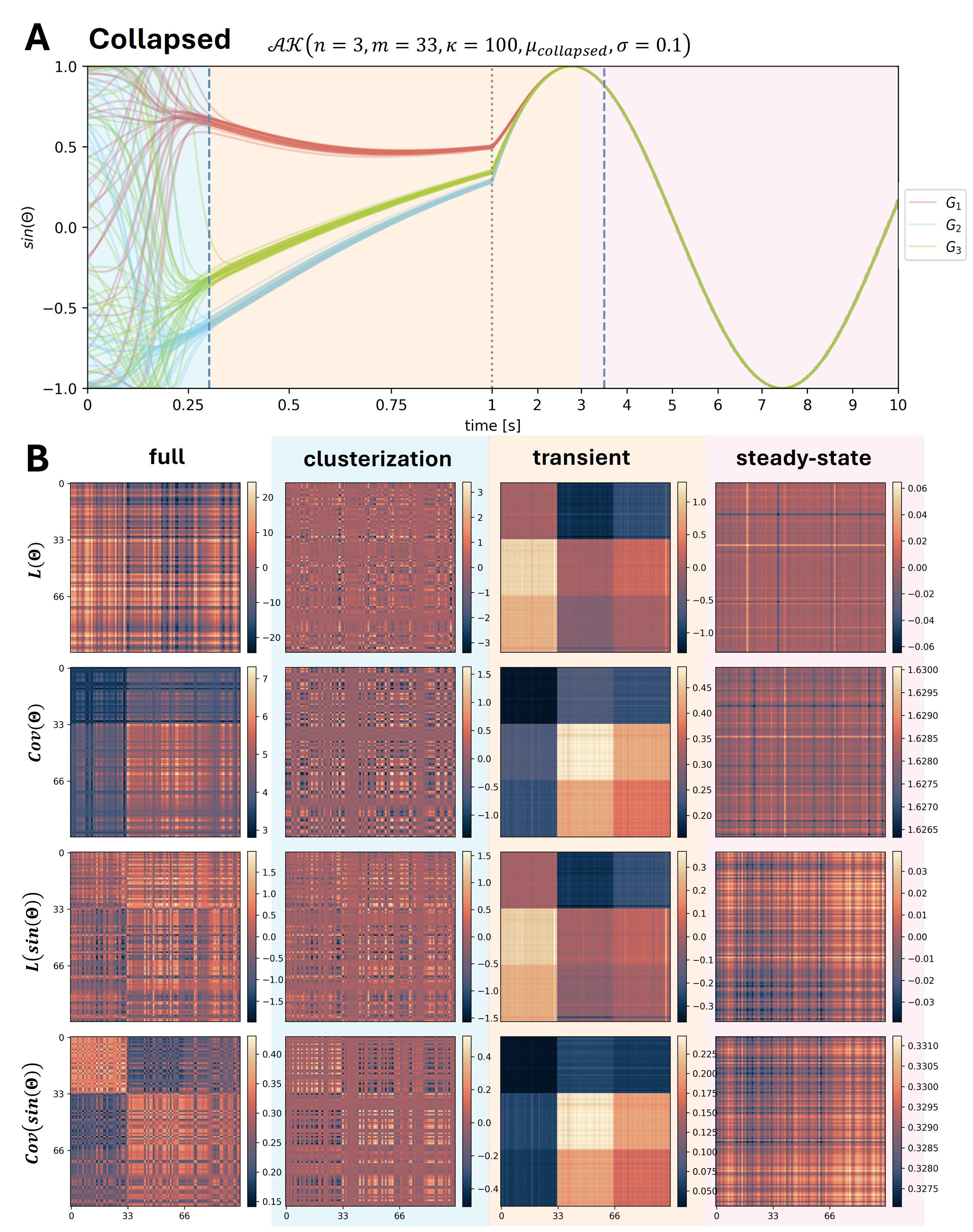}
\caption{\textbf{Collapsed KSBM time series and lead matrices}\\
(A) Time series of the collapsed KSBM $\cAK(n=3, m=33, \kappa=100, \mu=\{\frac{2}{3}, \frac{2}{3}, \frac{2}{3} ~rad/s\}, \sigma=0.1)$ with communities $G_1, G_2, G_3$. (B) Lead and covariance matrices computed from the full time series and each temporal regime for $\Theta$ and $\sin(\Theta)$. Both steady state and clusterization matrices are almost zero everywhere.}
\label{fig:KSBM-dynamics-collapsed}
\end{figure}

\begin{figure}[H]
\centering
\includegraphics[width=0.92\linewidth]{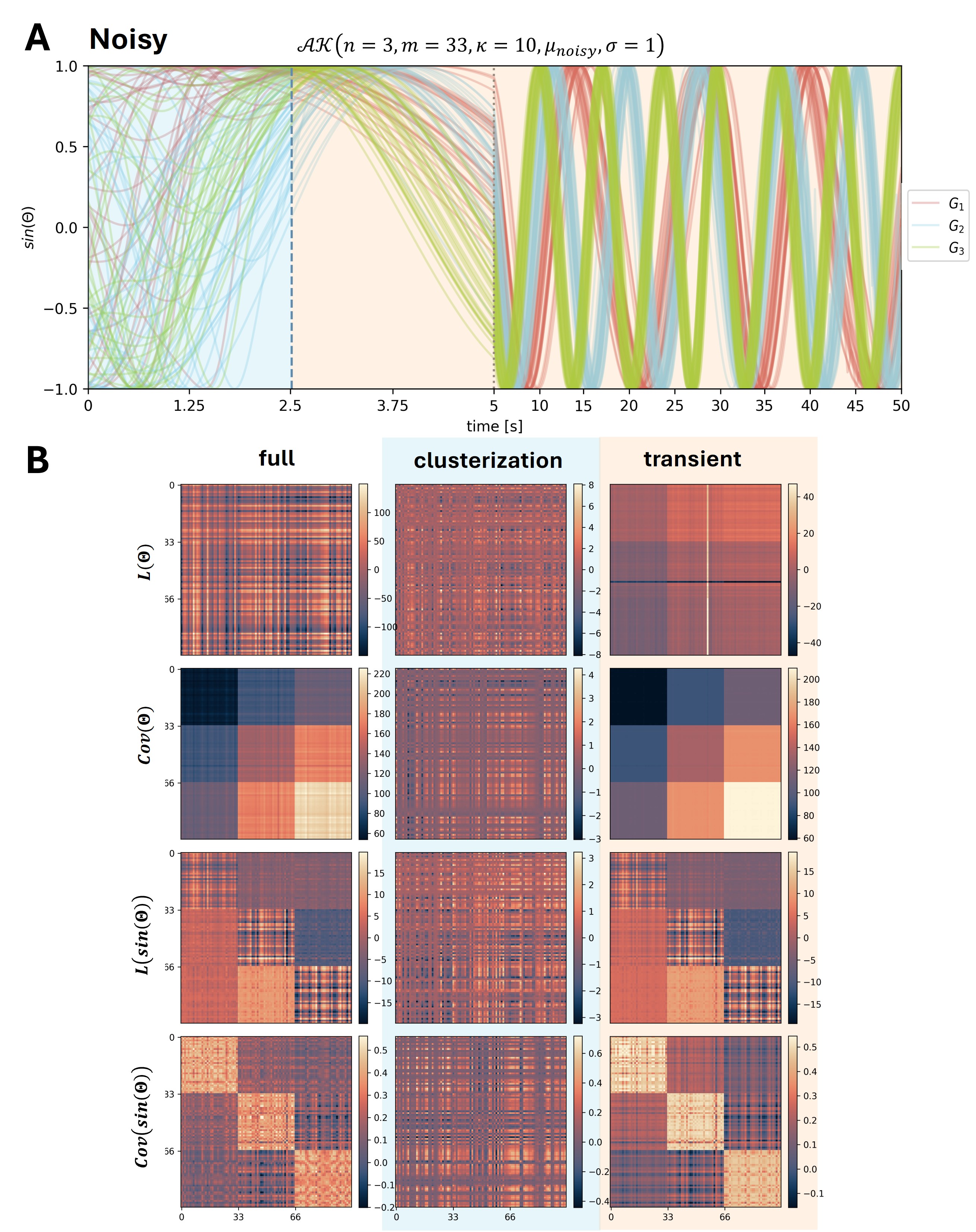}
\caption{\textbf{Noisy KSBM time series and lead matrices}\\
(A) Time series of the Noisy KSBM $\cAK(n=3, m=33, \kappa=10, \mu=\{\frac{1}{3}, \frac{2}{3}, 1 ~rad/s\}, \sigma=1)$ with communities $G_1, G_2, G_3$. (B) Lead and covariance matrices computed from the full time series and each temporal regime for $\Theta$ and $\sin(\Theta)$. The {heterogeneity} and weak coupling prevent the convergence of oscillators to the community average oscillator and in particular prevent the existence of a steady state.}
\label{fig:KSBM-dynamics-noisy}
\end{figure}

\begin{figure}[H]
\centering
\includegraphics[width=0.92\linewidth]{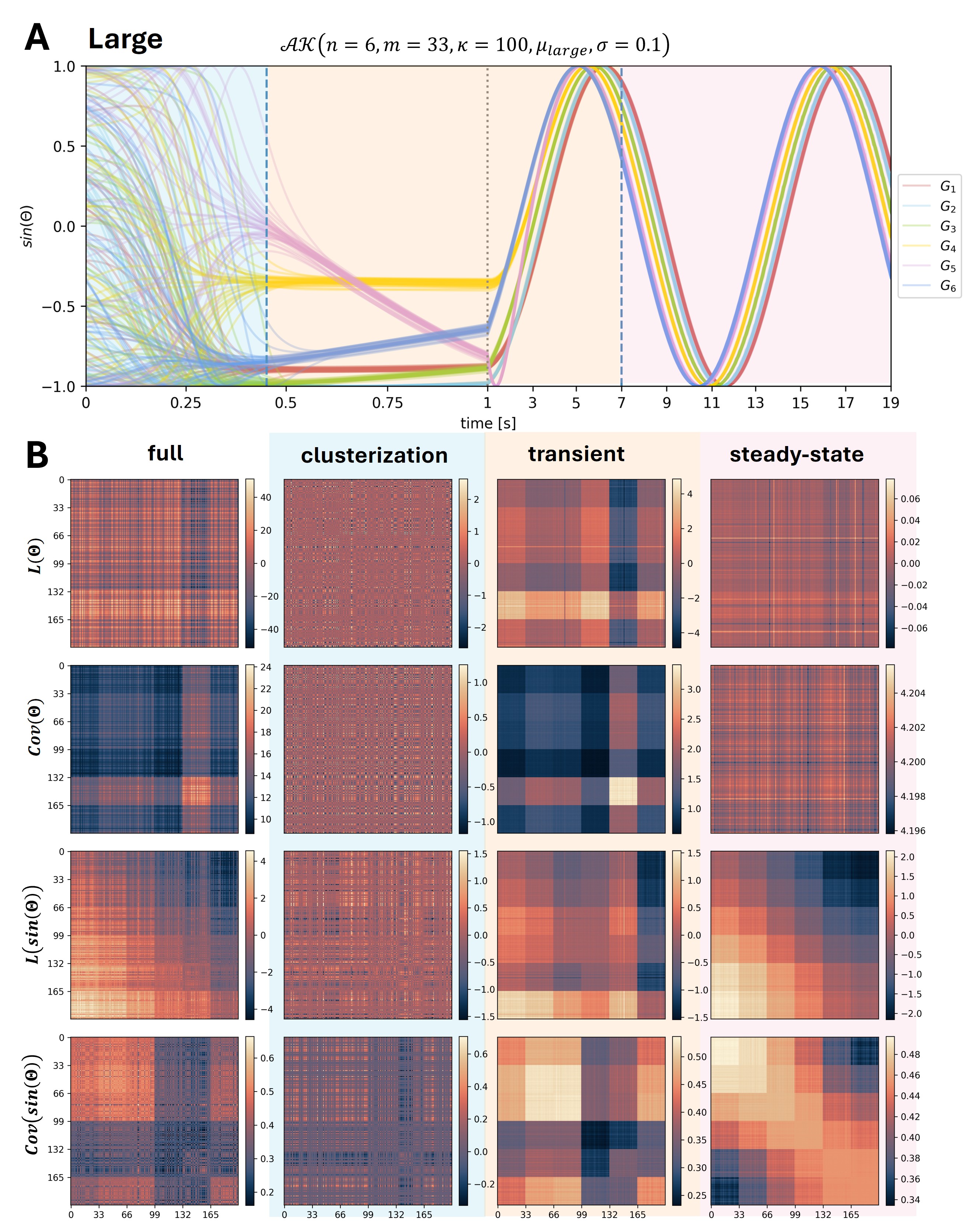}
\caption{\textbf{Large KSBM time series and lead matrices}\\
(A) Time series of the large KSBM $\cAK(n=6, m=33, \kappa=100, \mu=\{\frac{1}{6}, \frac{1}{3}, \frac{1}{2}, \frac{2}{3}, \frac{5}{6}, 1 ~rad/s\}, \sigma=0.1)$ with communities $G_1, G_2, G_3, G_4, G_5, G_6$. (B) Lead and covariance matrices computed from the full time series and each regime for $\Theta$ and $\sin(\Theta)$. Some communities cannot be distinguished in the lead and covariance matrices computed from the full time series in contrast to the transient and steady state regimes.}
\label{fig:KSBM-dynamics-large}
\end{figure}

\begin{figure}[ht!]
  \centering
  \includegraphics[width=0.92\linewidth]{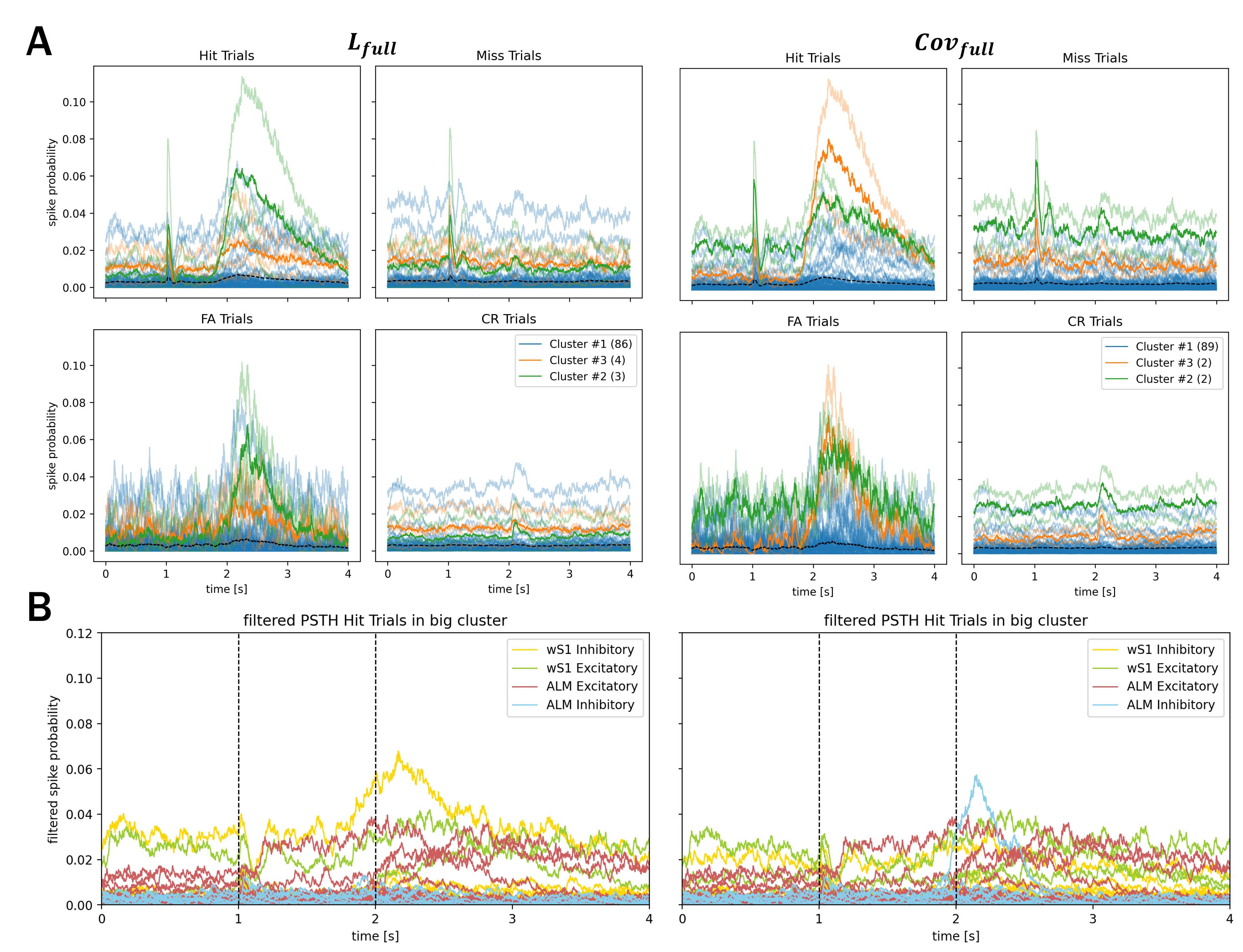}
  \caption{\textbf{$K$-means ($K=3$) Clusters}\\
  (A) PSTH for each trial type labeled with clusters identified by $K$-means with $K=3$ on (Left) lead and (Right) covariance matrices, non-transparent line correspond to the average PSTH within cluster and dashed line to the population average PSTH. (B) Filtered PSTH of neurons belonging to the largest (trivial) cluster from $K$-means on (Left) lead and (Right) covariance matrices.}
  \label{fig:algo-k3-clusters}
\end{figure}

\begin{figure}[ht!]
  \centering
  \includegraphics[width=0.92\linewidth]{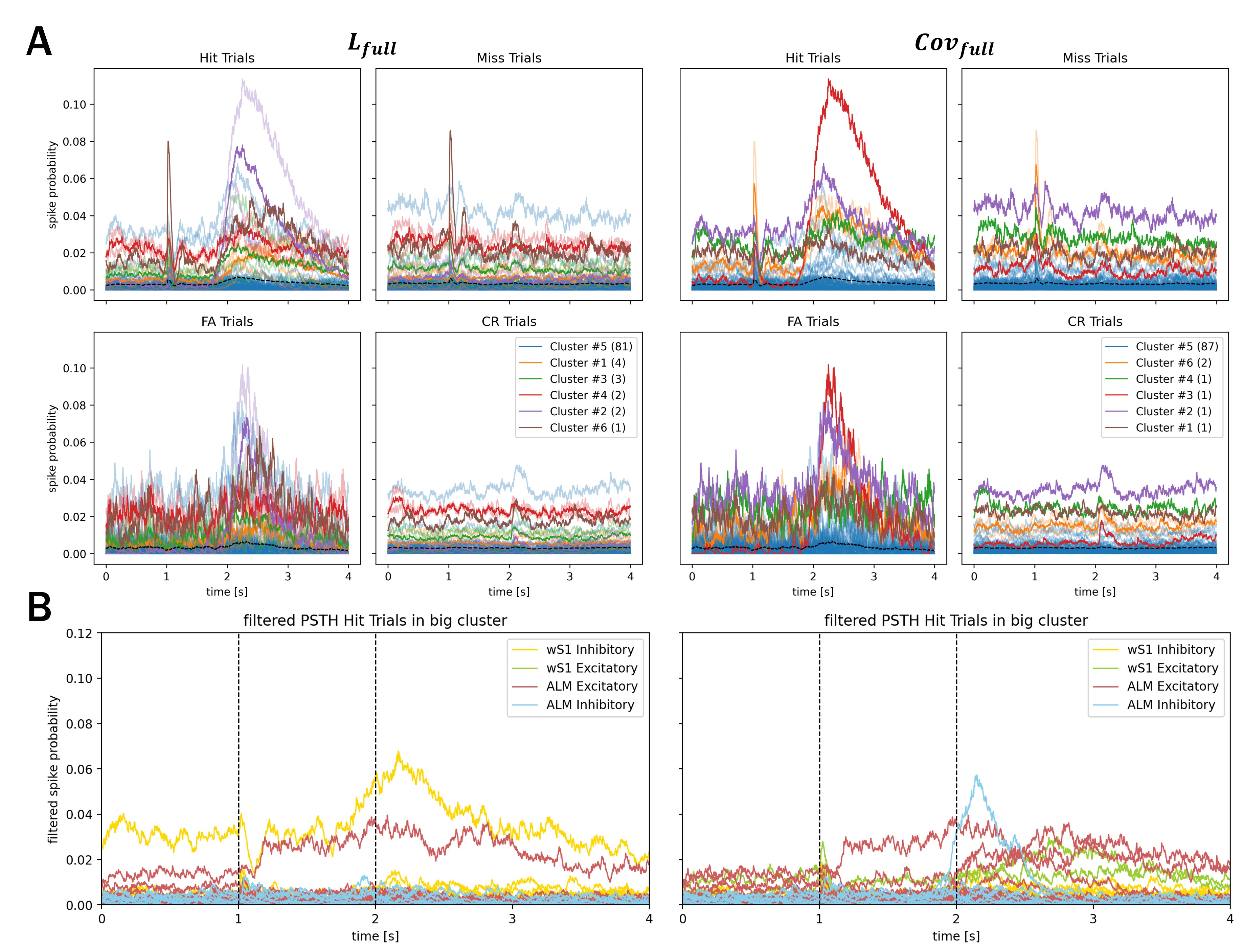}
  \caption{\textbf{$K$-means ($K=6$) Clusters}\\
  (A) PSTH for each trial type labeled with clusters identified by $K$-means with $K=6$ on (Left) lead and (Right) covariance matrices, non-transparent line correspond to the average PSTH within cluster and dashed line to the population average PSTH. (B) Filtered PSTH of neurons belonging to the largest (trivial) cluster from $K$-means on (Left) lead and (Right) covariance matrices.}
  \label{fig:algo-k6-clusters}
\end{figure}

\begin{figure}[ht!]
  \centering
  \includegraphics[width=0.92\linewidth]{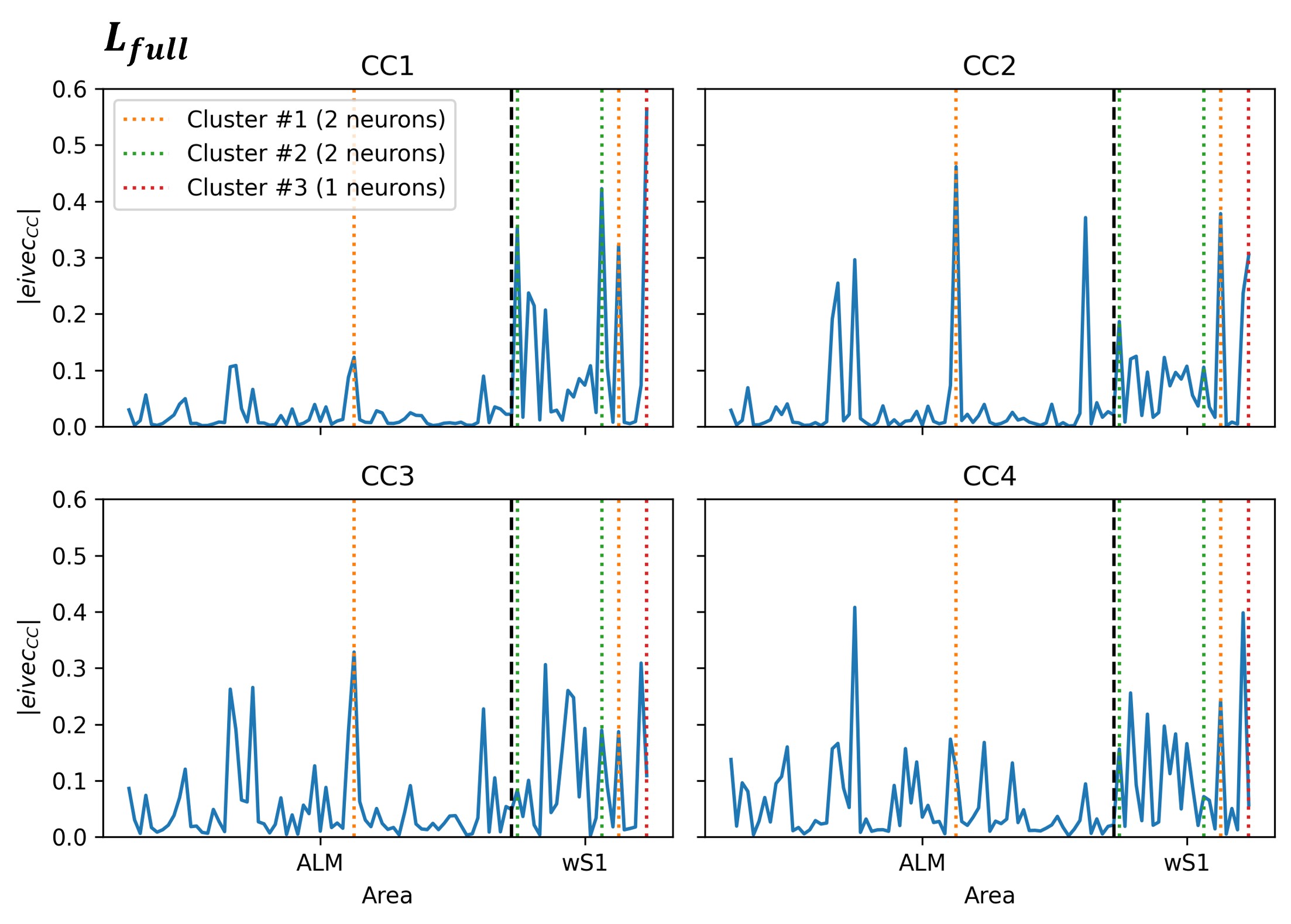}
  \caption{\textbf{SCE clusters and lead matrix eigenvectors modulus}\\
  Modulus of the first four (pairs of) eigenvectors of the lead matrix over the full time series for the neural recordings with neurons identified by the small clusters overlain.}
  \label{fig:sce-eigenmodulus}
\end{figure}

\section{Proofs and Technical Results} \label{apx:proofs}

\subsection{Kuramoto Stochastic Block Model} \label{apxssec:ksbm}

The probability of two oscillators being coupled in the assortative KSBM is given by the following lemma.

\begin{lemma}
\label{lemma:probability-assortative-KSBM}
    In an assortative KSBM, for any $r \neq s \in [n]$,
    $$P_{rs} = \frac{2}{m(n-1)} + O(\frac{1}{m^2})$$
    while $P_{rr} = 1$.
\end{lemma}
\begin{proof}
    First recall that $P_{rs}$ for $r \neq s$ in the assortative KSBM is the probability that a given pair of nodes in $G_r$ and $G_s$ are coupled. Since any node in $G_r$ has a single directed edge to a node in $\bigcup_{s'\neq r}G_{s'}$, it follows that $i \sim j$ if and only if $i$ connects to $j$ ($i \rightarrow j$) or $j$ connects to $i$ ($j \rightarrow i$). Those events are independent, hence, $\p{i \sim j} = \p{i \rightarrow j} + \p{j \rightarrow i} - \p{i \rightarrow j}\p{j \rightarrow i}$.
    Furthermore those two connection probabilities are equal by symmetry in the assortative KSBM. Without loss of generality, $\p{i \rightarrow j} = \frac{1}{m(n-1)}$ since $i$ has a directed edge to a single oscillator chosen uniformly among $m(n-1)$ oscillators. Hence, $P_{rs} = \frac{2}{m(n-1)} - \frac{1}{m^2(n-1)^2}$, $P_{rr} = 1$ simply follows from the fully connected communities. When $m \gg 1$, the term $-\frac{1}{m^2(n-1)^2}$ vanishes.
\end{proof}

\subsection{Mean-Field KSBM} \label{apxssec:mf_ksbm}

This next technical Lemma is very useful as it allows us to go from a statement on $\epsilon$-variance-clustered to a probabilistic statement on $\epsilon$-clustered. The converse result also exists but we will not need it in the following.

\begin{lemma}
\label{lemma:variance-clustered-implies-clustered-whp}
    For any $U \subseteq [N]$, if $(\theta_i)_{i \in U}$ is $\epsilon$-variance-clustered, then it is $\epsilon^{\frac{1}{2} - \nu}$-clustered with high probability for any $\nu > 0$ if $\epsilon \ll 1$.
\end{lemma}
\begin{proof}
    We first remark that,
    \begin{equation*}
    \begin{split}
        \max_{i,j \in U}{\abs{\theta_i(t) - \theta_j(t)}} & = \max_{i,j \in U}{\abs{(\theta_i(t)-\avg{\theta_k(t)}_{k \in U})-(\avg{\theta_k(t)}_{k \in U}-\theta_j(t))}},\\
        & \leq 2\max_{i,j \in U}{\abs{\theta_i(t)-\avg{\theta_k(t)}_{k \in U}}}.
    \end{split}
    \end{equation*}
    Hence, it suffices to bound the maximal distance between an oscillator and its average in $U$ by $\epsilon/2$ to obtain $\epsilon$-clusterization. Using Chebyshev's inequality, $\p{\abs{\theta_i(t)-\avg{\theta_k(t)}_{k \in U}} \geq a} \leq \frac{\Var(\theta_i(t))}{a^2} = \frac{\epsilon}{a^2}$ and therefore, $\p{\abs{\theta_i(t)-\avg{\theta_k(t)}_{k \in U}} \leq a} \geq 1-\frac{\epsilon}{a^2}$. The expression for the probabilities follows by taking $a = \frac{1}{2}\epsilon^{\frac{1}{2}-\nu}$ for any $\nu > 0$. The result is thus a consequence of the term $O(\epsilon^{2\nu})$ vanishing whenever $\epsilon \ll 1$.
\end{proof}

\begin{remark}
    (Equivalent Statement of the KSBM)\\
    In the following proofs, it will be easier to reformulate the KSBM differential equation (Equation~\\ref{eq:generalized-kuramoto}) so that the sum over communities and dependence on two oscillators being connected is more apparent.
    For $\Theta(t)$ a realisation of a KSBM $\cK$ of $N$ oscillators in $n$ balanced communities, then for any $i \in G_r$ for any $r=1,...,n$, Equation~\\ref{eq:generalized-kuramoto} can be stated equivalently as,
    \begin{equation*}
        \dot{\theta}_i(t) = \omega_i + \sum_{s=1}^n{C_{rs}\sum_{j \in G_s}{\sin(\theta_j(t) - \theta_i(t))\I\{j \sim i\}}},
    \end{equation*}
    where $j \sim i$ denotes that $j$ connects to $i$ (i.e. $A_{ij} = 1$).
\end{remark}

\begin{remark}
    In the following, we will denote the sample average as $\avg{-}$ and, when applicable, $\E[\theta_i(t)]{-}$ denotes the expectation over the distribution of oscillator $\theta_i$ at time $t$. When we consider Gaussian models, then all oscillators within a given community $G_r$ for $r \in [n]$ will be i.i.d. Gaussian at time $t$, hence we introduce the short-hand notation $\E[i\in G_r]{-} := \E[\theta_{i^*}(t)]{-}$ where $i^* \in G_r$ is any oscillator in $G_r$ and we drop $t$ for conciseness. This notation makes the link to the population average $\avg{-}_{i \in G_r}$ explicit which converges by the Law of Large Numbers (LLN) to $\E[i\in G_r]{-}$ whenever oscillators are i.i.d. and $m \gg 1$.
\end{remark}

\begin{lemma}
\label{lemma:expectation-sin-theta-is-0}
    $\avg{\sin(\theta_j(t) - \theta_i(t))}_{i,j \in G_r} = 0$.
\end{lemma}
\begin{proof}
    By definition of the sample average
    $$\avg{\sin(\theta_j(t) - \theta_i(t))}_{i,j \in G_r} \propto \sum_{i \in G_r}{\sin(\theta_i(t) - \theta_i(t))}
        + \sum_{i<j \in G_r}{(\sin(\theta_j(t) - \theta_i(t)) + \sin(\theta_i(t) - \theta_j(t)))} = 0.$$
\end{proof}
\begin{corollary}
    The above result also holds for (conditional) expectations on the oscillators when assuming i.i.d. oscillators,
    $$\E[i,j\in G_r]{\sin(\theta_j(t)-\theta_i(t))} = 0 \text { and } \E[i,j\in G_r]{\sin(\theta_j(t)-\theta_i(t))|j\sim i} = 0$$
\end{corollary}
\begin{proof}
    By the Law of Large Numbers (LLN), the expectation $\E[i,j\in G_r]{\sin(\theta_j(t)-\theta_i(t))}$ is the limit when $m \rightarrow \infty$ of the sequence of averages $\avg{\sin(\theta^{(m)}_j(t) - \theta^{(m)}_i(t))}_{i,j \in G_r} = 0$, thus the expectation is necessarily zero. We remark that the conditional average on $i\sim j$ (i.e. where we only average over pairs of connected oscillators) is also zero since $i\sim j$ implies $j\sim i$,
    \begin{equation*}
        \avg{\sin(\theta_j(t) - \theta_i(t))|i\sim j}_{i,j \in G_r} \propto  \sum_{i<j \in G_r, i\sim j}{(\sin(\theta_j(t) - \theta_i(t)) + \sin(\theta_i(t) - \theta_j(t)))} = 0.
    \end{equation*}
    Thus, the conditional expectation will also be zero as the limit $m \rightarrow \infty$ of zero sequences.
\end{proof}

\begin{lemma}
\label{lemma:clustered-convergence-sine}
    If $\Theta(t)$ is $\epsilon$-clustered with $\epsilon \ll 1$, then
    $$\avg{\sin(\theta_j(t)-\theta_i(t))}_{i \in G_r, j \in G_s} = \sin(\theta_{G_s}(t) - \theta_{G_r}(t)) + O(\epsilon).$$
\end{lemma}
\begin{proof}
    For simplicity we define $\epsilon_{r,i} := \theta_i(t) - \theta_{G_r}(t)$. We can rewrite it as,
    $$\sin(\theta_j(t)-\theta_i(t)) = \sin(\theta_{G_s}(t) - \theta_{G_r}(t))\cos(\epsilon_{s,j} - \epsilon_{r,i})
        + \cos(\theta_{G_s}(t) - \theta_{G_r}(t))\sin(\epsilon_{s,j} - \epsilon_{r,i}).$$
    We now observe that $\abs{\epsilon_{r,i}} \leq \epsilon$ by $\epsilon$-clusterization, so $\abs{\epsilon_{s,j} - \epsilon_{r,i}} \leq 2\epsilon \ll 1$. Using, $\sin(\epsilon') =  O(\epsilon')$ and $\cos(\epsilon') = 1 + O(\epsilon'^2)$ for $\epsilon' \ll 1$, we get, $\sin(\theta_j(t)-\theta_i(t)) = \sin(\theta_{G_s}(t) - \theta_{G_r}(t)) + O(\epsilon)$. 
\end{proof}

\begin{proof}
\label{proof:thm:MF-KSBM}
    \textbf{Theorem \ref{thm:MF-KSBM}}\\
    Applying Theorem \ref{thm:mean_field_kuramoto}, we can replace 
    one realisation $\Theta^{(m)}$ of the random KSBM $\cK_m$ with that of the Kuramoto model $\hTheta^{(m)}$ defined in Equation~\ref{eq:mean_field_deterministic} for 
    $m \gg 1$. We obtain,
    \begin{equation*}
        \dot{\theta}^{(m)}_i(t) \sim \omega_i
        + \sum_{s=1}^n{C_{rs}P_{rs}\sum_{j \in G_s}{\sin(\htheta^{(m)}_j(t) - \htheta^{(m)}_i(t))}}.
    \end{equation*}

    We now compute the average within community $r \in [n]$:
    
    \begin{equation*}
    \begin{split}
        \dot{\theta}^{(m)}_{G_r}(t) & = \avg{\dot{\theta}_i}_{i \in G_r} \sim \mu_r
        + \sum_{s=1}^n{C_{rs}P_{rs}\sum_{j \in G_s}{\avg{\sin(\htheta^{(m)}_j(t) - \htheta^{(m)}_i(t))}_{i \in G_r}}},\\
        & = \mu_r
        + m\sum_{s=1}^n{C_{rs}P_{rs}\avg{\sin(\htheta^{(m)}_j(t) - \htheta^{(m)}_i(t))}_{i \in G_r, j\in G_s}},
    \end{split}
    \end{equation*}
    where we use the Law of Large Numbers (LLN) to obtain $\avg{\omega_i}_{i\in G_r} \xrightarrow{LLN} \mu_r$ since $\omega_i$ are independent. Applying Lemma \ref{lemma:clustered-convergence-sine} and $\epsilon$-clusterization with $\epsilon \ll 1$ and $\bar{C} = O(1/N)$, we obtain
    $$\dot{\theta}^{(m)}_{G_r}(t) \sim \mu_r + m\sum_{s=1}^n{C_{rs}P_{rs}\sin(\htheta^{(m)}_{G_s}(t) - \htheta^{(m)}_{G_r}(t))} + O(\epsilon).$$
\end{proof}

We characterize the notion of synchronization of all oscillators as a steady state and now present results concerning the deviations between oscillators within a community. These results will help us understand the effect of intrinsic frequency {heterogeneity} on the KSBM dynamics.

\begin{lemma}
\label{lemma:synchronized-steady state}
    (Synchronized steady state)\\
    Let $\Theta$ be a realisation of a KSBM $\cK$ and $\Theta(t)$ in synchronized steady state by time $T \geq 0$ then,
    $$\theta_{G_r}(t) = \omega(P,C)(t-T) + \Delta\theta_{G_r} \ \forall r\in [n],$$
    where $\omega(P,C) \in \R$ is some constant and $\Delta\theta_{G_r} = \theta_{G_r}(T)$. Furthermore, if the KSBM is symmetric, i.e. $P_{rs}C_{rs} = P_{sr}C_{sr} \ \forall r,s \in [n]$, then $\omega(P,C) = \omega = \avg{\omega_{G_r}}_{r\in [n]}$.
\end{lemma}
\begin{proof}
    In synchronized steady state all frequencies are equal, i.e., $\dot{\theta}_i = \omega(P, C)$ for all $i \in [N]$, thus,
    $$\theta_{G_r}(t) = \int_T^t{\omega(P,C) \,dt} = \omega(P,C)(t-T) + \theta_{G_r}(T).$$
\end{proof}

\begin{lemma}
\label{lemma:deviation-synchronized-steady-state}
    (Deviation at Synchronized steady state)\\
     Let $\Theta^{(m)}$ be a realisation of the KSBM $\cK_m$ under the same conditions as in Theorem \ref{thm:MF-KSBM} at synchronized steady state. Then for $i \in G_r$ and $r\in [n]$
    $$\theta^{(m)}_i(t)-\theta^{(m)}_{\bullet}(t) = \arcsin\left(\frac{1}{m}\frac{\omega_i-\omega_{\bullet}}{\sum_{s=1}^nC_{rs}P_{rs}\cos(\htheta^{(m)}_{G_s}(t)-\htheta^{(m)}_{G_r}(t))}\right) + O(\epsilon^2),$$
    where $\bullet$ denotes either any oscillator $j$ in the same community $G_r$ or the community average oscillator of $G_r$, and $\hTheta^{(m)}$ is a realisation of the mean-field KSBM from Theorem \ref{thm:MF-KSBM}.
\end{lemma}
\begin{proof}
    Recall that the conditions of Theorem \ref{thm:MF-KSBM} are that $m \gg 1$, $\Theta^{(m)}$ is $\epsilon$-clustered with $\epsilon \ll 1$ and $\avg{C_{rs}P_{rs}}_{r\neq s\in [n]}=O(1/N)$. By assuming that the system is at (synchronized) steady state, we also know that for any pair of oscillators $i,j \in [N]$ we have $\dot\theta_i^{(m)} = \dot\theta_j^{(m)}$. It follows that this also holds with average oscillators $\theta_{G_r}^{(m)}$ for any $r \in [n]$. We now drop the dependence on $t$ since all deviations are constant (since the oscillators are synchronized, i.e. derivatives are equal).
    We first consider the case when $i,j \in G_r$ for some $r\in [n]$. Our first step is to replace the 
    realisation $\Theta^{(m)}$ of the random KSBM with the deterministic model $\hTheta^{(m)}$ of Equation~\ref{eq:mean_field_deterministic} using Theorem \ref{thm:mean_field_kuramoto},
    $$\dot\theta_i^{(m)} - \dot\theta_j^{(m)} \sim \omega_i-\omega_j + \sum_{s=1}^n{C_{rs}P_{rs}}\sum_{k\in G_s}(sin(\htheta_k^{(m)}-\htheta_i^{(m)}) - sin(\htheta_k^{(m)}-\htheta_j^{(m)})).$$
    Notice that since both models converge when $m \gg 1$, then $\hTheta^{(m)}$ is also $\epsilon$-clustered with $\epsilon \ll 1$. Using this we can rework the difference of sines as,
    $$sin(\htheta_k^{(m)}-\htheta_i^{(m)}) - sin(\htheta_k^{(m)}-\htheta_j^{(m)}) =  - \cos(\htheta_k^{(m)}-\htheta_i^{(m)})sin(\htheta_i^{(m)}-\htheta_j^{(m)}),$$
    by using $\cos(\htheta_i^{(m)}-\htheta_j^{(m)}) = 1 + O(\epsilon^2)$ since $i$ and $j$ are in the same community which is $\epsilon$-clustered. We need to keep terms of order $\epsilon$ since that is the range of $\htheta_i^{(m)} - \htheta_j^{(m)}$. Using a similar argument, we can see that $\cos(\htheta_k^{(m)}-\htheta_i^{(m)}) = 1 + O(\epsilon^2)$ for any $k \in G_r$. Thus the difference $\dot\theta_i^{(m)} - \dot\theta_j^{(m)}$, up to an error $O(\epsilon^2)$, can now be expressed as
    $$\dot\theta_i^{(m)} - \dot\theta_j^{(m)} \sim \omega_i-\omega_j - sin(\htheta_i^{(m)}-\htheta_j^{(m)})(mC_{rr}P_{rr} - \sum_{s\neq r}{C_{rs}P_{rs}}\sum_{k\in G_s}\cos(\htheta_k^{(m)}-\htheta_i^{(m)}).$$
    Similarly, we can re-express the cosine as
    \begin{equation*}
    \begin{split}
        \cos(\htheta_k^{(m)}-\htheta_i^{(m)}) & \sim \cos(\htheta_k^{(m)}-\htheta_{G_r}^{(m)})-sin(\htheta_k^{(m)}-\htheta_{G_r}^{(m)})sin(\htheta_{G_r}^{(m)}-\htheta_i^{(m)}),\\
        & = \cos(\htheta_{G_s}^{(m)}-\htheta_{G_r}^{(m)})-sin(\htheta_{G_s}^{(m)}-\htheta_{G_r}^{(m)})sin(\htheta_k^{(m)}-\htheta_{G_s}^{(m)}) + O(\epsilon),\\
        & = \cos(\htheta_{G_s}^{(m)}-\htheta_{G_r}^{(m)}) + O(\epsilon).
    \end{split}
    \end{equation*}
    Notice that the $O(\epsilon)$ error originating from the sine terms becomes $O(\epsilon^2)$ in $\dot\theta_i^{(m)} - \dot\theta_j^{(m)}$, since the cosine is multiplied by $sin(\htheta_i^{(m)}-\htheta_j^{(m)}) = O(\epsilon)$, and hence is negligible. The condition $\avg{C_{rs}P_{rs}}_{r\neq s\in [n]}=O(1/N)$ ensures that the error is controlled through the summations, and thus we obtain,
    $$\dot\theta_i^{(m)} - \dot\theta_j^{(m)} \sim \omega_i-\omega_j - msin(\htheta_i^{(m)}-\htheta_j^{(m)})\sum_{s=1}^n{C_{rs}P_{rs}\cos(\htheta_{G_s}^{(m)}-\htheta_{G_r}^{(m)})} + O(\epsilon^2).$$
    At steady state, this difference is equal to zero and we can therefore rewrite the deviation $\htheta_i^{(m)} - \htheta_j^{(m)}$ as,
    $$\htheta^{(m)}_i-\htheta^{(m)}_{\bullet} = \arcsin\left(\frac{1}{m}\frac{\omega_i-\omega_{\bullet}}{\sum_{s=1}^nC_{rs}P_{rs}\cos(\htheta^{(m)}_{G_s}-\htheta^{(m)}_{G_r})}\right) + O(\epsilon^2).$$
    The claim follows from the convergence of $\theta^{(m)}$ to $\htheta^{(m)}$ in the limit $m \gg 1$ by Theorem \ref{thm:mean_field_kuramoto}. We also know that $\htheta_{G_r}^{(m)}$ for $r \in [n]$ converges to the mean-field model by Theorem \ref{thm:MF-KSBM}. The proof for the deviation $\theta^{(m)}_i(t)-\theta^{(m)}_{G_r}(t)$ is analogous.
\end{proof}

\begin{corollary}
    \label{cor:deviation-synchronized-steady-state-assortative-KSBM}
    (Deviation at Synchronized steady state for Assortative KSBM)\\
    Let $\Theta^{A,(m)}$ be a realisation of an assortative KSBM $\cAK_m$ at steady state. If $n\sigma \ll \kappa$ and $m \gg 1$, then the following holds for $i \in G_r$ and $r\in [n]$:
    $$\theta^{A,(m)}_i(t)-\theta^{A,(m)}_{\bullet}(t) = \arcsin\left(\frac{n}{\kappa}(\omega_i-\omega_{\bullet})\right) + O(\epsilon^2),$$
    where $\bullet$ denotes either any oscillator $j$ in the same community $G_r$ or the community average oscillator.
\end{corollary}
\begin{proof}
    Since we are in synchronized steady state and given $n\sigma \ll \kappa$ and $m \gg 1$, we can assume $t \gg t_{\trans}$. Hence by Lemma \ref{lemma:transition-time-assortative-KSBM}, our system converges to the mean-field KSBM. Using Lemma \ref{lemma:deviation-synchronized-steady-state}, we obtain,
    $$\theta^{A,(m)}_i-\theta^{A,(m)}_{\bullet} = \arcsin\left(\frac{1}{m}\frac{\omega_i-\omega_{\bullet}}{\frac{\kappa}{N}\sum_{s=1}^nP_{rs}\cos(\htheta^{A,(m)}_{G_s}-\htheta^{A,(m)}_{G_r})}\right) + O(\epsilon^2).$$
    For $s \neq r$, notice that $\sum_{s \neq r}P_{rs}|\cos(\htheta^{A,(m)}_{G_s}-\htheta^{A,(m)}_{G_r})| \leq \sum_{s \neq r}P_{rs} = O(\frac{1}{m})$ which is negligible in comparison to $P_{rr}\cos(\htheta^{A,(m)}_{G_r}-\htheta^{A,(m)}_{G_r}) = 1$. Hence, in the $m \rightarrow \infty$ limit we have
    $$\theta^{A,(m)}_i(t)-\theta^{A,(m)}_{\bullet}(t) = \arcsin\left(\frac{n}{\kappa}(\omega_i-\omega_{\bullet})\right) + O(\epsilon^2).$$
\end{proof}

\subsection{Gaussian KSBM} \label{apxssec:gaussian_ksbm}

The following lemma \ref{lemma:gaussian-to-uniform} is useful to understand why the Gaussian distribution of oscillators of Hypothesis \ref{hypothesis:gaussian-phase} is not incompatible with uniform initial conditions.

\begin{lemma}
\label{lemma:gaussian-to-uniform}
    Let $X \in [-\pi, \pi]$ be a random variable where $X \sim \gauss{0}{V} \text{ mod } 2\pi$, then we have the following convergence in distribution,
    $$X \xrightarrow{V \rightarrow \infty} U(-\pi, \pi).$$
\end{lemma}
\begin{proof}
    Consider the probability density function of $X$
    $$f(x) = \frac{1}{\sqrt{2\pi V}}\sum_{k \in \Z}e^{-\frac{1}{2V}(x + 2\pi k)^2}.$$
    Fix $V$ and let $k_V := \sqrt{\frac{V}{2\pi^2}}$, then notice that for $x \in [-\pi, \pi]$, $e^{-\frac{1}{2V}(x + 2\pi k)^2} \sim e^{-\frac{1}{2V}(2\pi k)^2} = e^{-\frac{1}{V}2\pi^2k^2}$ for $k \gg 1$. We then remark that, $e^{-\frac{1}{V}2\pi^2k^2} \sim 0 \text{ if } k \gg k_V$ and $e^{-\frac{1}{2V}(x+2\pi k)^2} \sim 1 \text{ if } k \ll k_V$. So asymptotically for $V \gg 1$, which implies $k_V \gg 1$,
    $$\sum_{k \in \Z}e^{-\frac{1}{2V}(x + 2\pi k)^2} = O(1 \cdot \abs{\{k : -k_V \leq k \leq k_V\}}) = O(k_v).$$
    Hence $f(x)=\frac{1}{\sqrt{2\pi V}}\sum_{k \in \Z}e^{-\frac{1}{2V}(x + 2\pi k)^2} = O(\frac{1}{\sqrt{V}}k_V) = O(1)$ which is constant with respect to $V$ and $x$. Since $\int_{-\pi}^{\pi}{f(x)\,dx}=1$, $f(x)$ being asymptotically constant when $V \gg 1$ implies that $f(x) \rightarrow \frac{1}{2\pi}$.
\end{proof}

In order to derive the dominated Gaussian assortative KSBM of Thm. \ref{thm:dominated-gaussian-assortative-KSBM}, it is simpler to first derive the dynamics in the absence of (intrinsic frequency) {heterogeneity} (Hypothesis\ref{hypothesis:identical-frequency}). This simpler model we call the dominated and identical Gaussian assortative KSBM (Thm. \ref{thm:dominated-identical-gaussican-assortative-KSBM}). This model has all intra-cluster variance decaying to zero and is particularly useful to estimate the {transition times} $t_{crit}$ using Lemma \ref{lemma:transition-time-assortative-KSBM}.

\begin{hypothesis}
\label{hypothesis:identical-frequency}
    (Identical Frequency Assumption)\\
    We assume that $\omega_i = \mu_r$ for any $i \in G_r, r \in [n]$, where $\mu_r$ is the mean intrinsic frequency for community $r$.
\end{hypothesis}

This hypothesis is a simplifying assumption corresponding to dropping the {heterogeneity} $\sigma$ in the KSBM. We will add the {heterogeneity} back when considering the (dominated) Gaussian KSBM in Theorem \ref{thm:dominated-gaussian-assortative-KSBM} \& \ref{thm:gaussian-assortative-KSBM}.

\begin{theorem}
\label{thm:dominated-identical-gaussican-assortative-KSBM}
    (Dominated and Identical Gaussian Assortative KSBM)\\
    Under Hypothesis \ref{hypothesis:gaussian-phase}, \ref{hypothesis:intra-community-dominated} and \ref{hypothesis:identical-frequency} when $m \gg 1$, the assortative KSBM intra-community variance $V(t):=V^{A}_{G_r}(t)$ for any $r \in [n]$ satisfies
    \[
        \frac{dV(t)}{dt} = -\frac{2\kappa}{n}V(t)e^{-V(t)}.
    \]
\end{theorem}
\begin{proof}
    Let $\Theta^{A}(t)$ follow an assortative KSBM, using Hypothesis \ref{hypothesis:intra-community-dominated} we have for all $i \in G_r$ and $r \in [n]$,
    $$\dot{\theta}_i^{A}(t) = \omega_i + \frac{\kappa}{N}\sum_{j\in G_r}{\sin(\theta^{A}_j(t)-\theta_i^{A}(t))}.$$
    The community average oscillator's dynamics is given by 
    $$\dot{\theta}^{A}_{G_r}(t) = \omega_{G_r} + \frac{\kappa}{mN}\sum_{i,j\in G_r}{\sin(\theta^{A}_j(t) - \theta^{A}_i(t))} \texteq{L\ref{lemma:expectation-sin-theta-is-0}}\omega_{G_r}.$$ 
    If we focus on the within community variance,
    \begin{equation*}
        V^{A}_{G_r}(t) := \Var_{i \in G_r}(\theta^{A}_i(t)) = \E[i\in G_r]{(\theta^{A}_i(t) - \theta^{A}_{G_r}(t))^2} \lln \frac{1}{m}\sum_{i\in G_r}{(\theta^{A}_i(t) - \theta^{A}_{G_r}(t))^2},
    \end{equation*}
    where its derivative can be expressed as,
    \begin{equation*}
    \begin{split}
        \dot{V}^{A}_{G_r}(t) & \sim \frac{2}{m}\sum_{i\in G_r}{(\theta^{A}_i(t)-\theta^{A}_{G_r}(t))(\omega_i-\omega_{G_r} + \frac{\kappa}{N}\sum_{j \in G_r}{\sin(\theta^{A}_j(t)-\theta_i^{A}(t))}},\\
        & \lln 2\E[i\in G_r]{(\theta^{A}_i(t)-\theta^{A}_{G_r}(t))(\omega_i-\omega_{G_r} + \frac{\kappa}{n}\E[j\in G_r]{\sin(\theta^{A}_j(t)-\theta_i^{A}(t))})}.
    \end{split}
    \end{equation*}
    Which we rewrite as,
    \begin{equation*}
        \dot{V}^{A}_{G_r}(t) \sim 2\E[i\in G_r]{(\theta^{A}_i(t)-\theta^{A}_{G_r}(t))(\omega_i-\omega_{G_r})}
        + \frac{2\kappa}{n}\E[i_\in G_r]{(\theta^{A}_i(t)-\theta^{A}_{G_r}(t))\E[j\in G_r]{\sin(\theta^{A}_j(t)-\theta_i^{A}(t))}},
    \end{equation*}
    where the first term increases the variance while the second term decreases the variance. Under Hypothesis \ref{hypothesis:identical-frequency}, the first term vanishes and using Hypothesis \ref{hypothesis:gaussian-phase}, $\theta^{A}_j(t) \iid \gauss{\theta^{A}_{G_s}(t)}{V^{A}_{G_s}(t)}$ for all $j \in G_s$ and $s \in [n]$. the second term can be directly computed as Gaussian integrals.
    $$\E[j\in G_r]{\sin(\theta^{A}_j(t)-\theta_i^{A}(t))} = -\sin(\theta_i^{A}(t) - \theta^{A}_{G_r}(t))e^{-V^{A}_{G_r}(t)},$$ and therefore,
    \begin{equation*}
    \begin{split}
        \dot{V}^{A}_{G_r}(t) & \sim -\frac{2\kappa}{n}\E[i \in G_r]{(\theta^{A}_i(t)-\theta^{A}_{G_r}(t))\sin(\theta_i^{A}(t) - \theta^{A}_{G_r}(t))}e^{-V^{A}_{G_r}(t)},\\
        & = -\frac{2\kappa}{n}V^{A}_{G_r}(t)e^{-V^{A}_{G_r}(t)},
    \end{split}
    \end{equation*}
    as desired.
\end{proof}

When we add back the {heterogeneity} into the dominated and identical Gaussian assortative KSBM of Thm. \ref{thm:dominated-identical-gaussican-assortative-KSBM}, and thus obtain Thm. \ref{thm:dominated-gaussian-assortative-KSBM}, a new term appears in the differential equation which drives the variance to increase. This term can be bounded assuming that the system reaches steady state at some point, i.e., the coupling factor $\kappa$ is sufficiently large.

\begin{lemma}
\label{lemma:variation-driving-bound}
    For $\sigma n \ll \kappa$, 
    $$\E[\omega_i]{\arcsin(\frac{n(\omega_i - \omega_{G_r})}{\kappa})(\omega_i-\omega_{G_r})} \leq \frac{\pi\sigma^2n}{2\kappa} \text{ almost surely.}$$
\end{lemma}
\begin{proof}
    First, consider the integral limited to the range of values where $arcsin$ is well defined,
    $$\frac{1}{\sqrt{2\pi\sigma^2}}\int_{-\frac{\kappa}{n}}^{-\frac{\kappa}{n}}{x\,\arcsin(\frac{nx}{\kappa})e^{-\frac{x^2}{2\sigma^2}}\,dx} = \frac{1}{\sqrt{\pi}}\int_{-\frac{\kappa}{\sqrt{2\sigma^2}n}}^{\frac{\kappa}{\sqrt{2\sigma^2}n}}{\sqrt{2\sigma^2}x\,\arcsin(\frac{\sqrt{2\sigma^2}nx}{\kappa})e^{-x^2}\,dx}.$$
    If $\frac{\kappa}{\sqrt{2\sigma^2}n} \rightarrow \infty$, i.e., $\sigma n \ll \kappa$, then,
    $$\frac{1}{\sqrt{\pi}}\int_{-\frac{\kappa}{\sqrt{2\sigma^2}n}}^{\frac{\kappa}{\sqrt{2\sigma^2}n}}{\sqrt{2\sigma^2}x\,\arcsin(\frac{\sqrt{2\sigma^2}nx}{\kappa})e^{-x^2}\,dx} \rightarrow \E[\omega_i]{\arcsin(\frac{n(\omega_i - \omega_{G_r})}{\kappa})(\omega_i-\omega_{G_r})}.$$
    Second, remark that $x\,\arcsin(x)e^{-x^2}$ is even, and $\arcsin(x) \leq \frac{\pi}{2}x$ for $x > 0$.
    \begin{equation*}
    \begin{split}
        \frac{1}{\sqrt{\pi}}\int_{-\frac{\kappa}{\sqrt{2\sigma^2}n}}^{\frac{\kappa}{\sqrt{2\sigma^2}n}}{\sqrt{2\sigma^2}x\,\arcsin(\frac{\sqrt{2\sigma^2}nx}{\kappa})e^{-x^2}\,dx} & = \frac{2}{\sqrt{\pi}}\int_{0}^{\frac{\kappa}{\sqrt{2\sigma^2}n}}{\sqrt{2\sigma^2}x\,\arcsin(\frac{\sqrt{2\sigma^2}nx}{\kappa})e^{-x^2}\,dx},\\
        & \leq \frac{2}{\sqrt{\pi}}\int_{0}^{\frac{\kappa}{\sqrt{2\sigma^2}n}}{\frac{\sigma^2\pi nx^2}{\kappa}e^{-x^2}\,dx},\\
        & \leq \frac{2\sigma^2\sqrt{\pi} n}{\kappa}\int_{0}^{\infty}{x^2e^{-x^2}\,dx}.
    \end{split}
    \end{equation*}
    Finally, using integration by parts,
    $$\frac{1}{\sqrt{\pi}}\int_{-\frac{\kappa}{\sqrt{2\sigma^2}n}}^{\frac{\kappa}{\sqrt{2\sigma^2}n}}{\sqrt{2\sigma^2}x\,\arcsin(\frac{\sqrt{2\sigma^2}nx}{\kappa})e^{-x^2}\,dx} \leq \frac{\sigma^2\pi n}{2\kappa}.$$
    The result follows from taking the limit $\frac{\kappa}{\sqrt{2\sigma^2}n} \rightarrow \infty$.
\end{proof}

\begin{proof}
\label{proof:thm:dominated-gaussian-assortative-KSBM}
    \textbf{Theorem \ref{thm:dominated-gaussian-assortative-KSBM}}\\
    The first statement follows from our derivation in Theorem \ref{thm:dominated-identical-gaussican-assortative-KSBM} under Hypothesis \ref{hypothesis:gaussian-phase} and \ref{hypothesis:intra-community-dominated} prior to applying Hypothesis \ref{hypothesis:identical-frequency}. Observe that for $i \in G_r$,
    $$\E[i\in G_r]{(\theta^{A}_i(t)-\theta^{A}_{G_r}(t))(\omega_i-\omega_{G_r})} = \E[\omega_i]{\E[\theta_i(t)|w_i]{\theta^{A}_i(t)-\theta^{A}_{G_r}(t)|\omega_i}(\omega_i-\omega_{G_r})}.$$ If we model $\theta_i(t) - \theta^{A}_{G_r}(t)$ conditional on $w_i$ as being given by our Gaussian distribution with zero mean and variance $V^{A}_{G_r}(t)$ (modelling the transition from the uniform to clustered state) added to a deterministic steady state offset from the average oscillator $\theta^A_{G_r}$ caused by $\omega_i$, denoted $\Delta(\theta^{A}_i|\omega_i)$, we can rework the following expectation as
    $$\E[\theta_i(t)|w_i]{\theta^{A}_i(t)-\theta^{A}_{G_r}(t)|\omega_i} = \E[\theta_i(t)|w_i]{\gauss{0}{V^{A}_{G_r}(t)} + \Delta(\theta^{A}_i|\omega_i)|\omega_i} = \Delta(\theta^{A}_i|\omega_i).$$
    From Corollary \ref{cor:deviation-synchronized-steady-state-assortative-KSBM}, we know that at steady state $\Delta(\theta^{A}_i|\omega_i)(t) = \arcsin(\frac{n(\omega_i - \omega_{G_r})}{\kappa})$. It then only remains to evaluate $\E[\omega_i]{\arcsin(\frac{n(\omega_i - \omega_{G_r})}{\kappa})(\omega_i-\omega_{G_r})}$. Using Lemma \ref{lemma:variation-driving-bound}, we can approximate the value of this expectation under reasonable conditions by neglecting vanishing events. The derivative of the variance can therefore be rewritten as,
    \begin{equation*}
        \dot{V}^{A}_{G_r}(t) \sim \epsilon + \frac{2\kappa}{n}\E[i_\in G_r]{(\theta^{A}_i(t)-\theta^{A}_{G_r}(t))\E[j\in G_r]{\sin(\theta^{A}_j(t)-\theta_i^{A}(t))}},
    \end{equation*}
    where we let $\epsilon := 2\E[i\in G_r]{(\theta^{A}_i(t)-\theta^{A}_{G_r}(t))(\omega_i-\omega_{G_r})}$ which is bounded by $2\frac{\pi\sigma^2n}{2\kappa} = \frac{\pi\sigma^2n}{\kappa}$ by Lemma \ref{lemma:variation-driving-bound}. Since Hypothesis \ref{hypothesis:identical-frequency} only concerns the first term, we can directly use Theorem \ref{thm:dominated-identical-gaussican-assortative-KSBM} to obtain
    $$dV = \epsilon-\frac{2\kappa}{n}Ve^{-V}dt.$$
    For the second part, we look at the steady state condition,
    \begin{equation*}
        dV = 0\,dt \iff \frac{2\kappa}{n}Ve^{-V} = \epsilon \iff V = \frac{n}{2\kappa}\epsilon e^V.
    \end{equation*}
    This equation possess at most two fixed points which are given by $\frac{2\kappa}{n^2}f(V) = \epsilon$ where $f(V) = Ve^{-V}$. $f(V)$ possess a global maximum at $V = 1$ after an initial linear increase from $V=0$ and an asymptote toward zero at $V \rightarrow \infty$ after an exponential decay. In $dV$, this maximum has value $\frac{2\kappa}{n}f(1) = \frac{2\kappa}{n}e^{-1}$. Using this, we now have that $\frac{2\kappa}{n^2}f(V) = \epsilon$ has
    \begin{itemize}
        \item \textbf{two solutions} whenever $\epsilon < \frac{2\kappa}{n}e^{-1}$,
        \item \textbf{one solution} whenever $\epsilon = \frac{2\kappa}{n}e^{-1}$,
        \item \textbf{no solution} whenever $\epsilon > \frac{2\kappa}{n}e^{-1}$.
    \end{itemize}
    Since $dV = \epsilon - \frac{2\kappa}{n}f(V)dt$, it follows that when there are two solutions $V^*_1 < V^*_2$, then $V^*_1$, the smallest of the two, is a stable steady state, while $V^*_2$ is unstable. When there is a single solution $V^*_1 = V^*_2$, it is a saddle. We will only focus on the stable steady state $V^{SS} := V^*_1$ which should be the state one should observe when the system reaches steady state. Remark that since we know $\epsilon \leq \frac{\sigma^2\pi n}{\kappa}$, and we have two solutions when $\epsilon \leq \frac{2\kappa}{n}e^{-1}$, then we necessarily have two solutions when,
    $$\frac{\sigma^2\pi n}{\kappa} < \frac{2\kappa}{n}e^{-1} \iff e\sigma^2\pi n^2 < 2\kappa^2$$
    Since $V^{SS}$ is smaller than $V=1$ where $\frac{2\kappa}{n}f(V)$ is maximal, i.e. located during the initial (linear) increasing phase, we can conclude that, the larger $\epsilon$ the larger $V^{SS}$. If we denote this dependence as $V^{SS}(\epsilon)$, we can write for $\epsilon \leq \epsilon' < \frac{2\kappa}{n}e^{-1}$: $V^{SS}(\epsilon) \leq V^{SS}(\epsilon')$. Hence, if we assume $e\sigma^2\pi n^2 < 2\kappa^2$, and let $\epsilon' := \frac{\sigma^2\pi n}{\kappa}$, which we know is such that $\epsilon < \epsilon'$, then, $V^{SS}(\epsilon) \leq V^{SS}(\frac{\sigma^2\pi n}{\kappa}) =: v^*$ where $v^*$ is the stable fixed point of the following equation,
    $$V = \frac{n}{2\kappa}\epsilon'e^V = \frac{\pi}{2}(\frac{\sigma n}{\kappa})^2e^V.$$
\end{proof}

The following corollary shows that in practice, with sufficiently strong coupling $\kappa$, the intra-community phase variance vanishes at steady state.

\begin{corollary}
\label{cor:vanishing-variance-dominated-gaussian-assortative-KSBM}
    Under the same condition as Theorem \ref{thm:dominated-gaussian-assortative-KSBM}, with $n\sigma \ll \kappa$ in particular, $V^{SS} \leq \frac{\pi}{2}(\frac{\sigma n}{\kappa})^2$.
\end{corollary}
\begin{proof}
    Using $f(V) = Ve^{-V}$ from the last proof (\ref{proof:thm:dominated-gaussian-assortative-KSBM}), $V = \frac{\pi}{2}(\frac{\sigma n}{\kappa})^2e^V \iff \frac{\pi}{2}(\frac{\sigma n}{\kappa})^2 = f(V)$.  This equation has two fixed points when $\epsilon' := \frac{\pi}{2}(\frac{\sigma n}{\kappa})^2 < e^{-1}$, which holds since we assume $n\sigma \ll \kappa$, in particular it holds for $\epsilon' \ll 1$. The smallest of the two, $v^*$, occurs when $\epsilon'$ intersects the initial increasing linear phase of $f(V)$, i.e., $\epsilon' = f(v^*)$. Since $\epsilon' \ll 1$, this intersection occurs when $f(V) \sim V$, hence $v^* \sim \epsilon' = \frac{\pi}{2}(\frac{\sigma n}{\kappa})^2$. The conclusion follows from $V^{SS} \leq v^*$.
\end{proof}

\begin{proof}
\label{proof:lemma:dominated-gaussian-assortative-KSBM-is-clustered}
    \textbf{Lemma \ref{lemma:dominated-gaussian-assortative-KSBM-is-clustered}}\\
    From Corollary \ref{cor:vanishing-variance-dominated-gaussian-assortative-KSBM}, we know that $V^{SS} \leq \frac{\pi}{2}(\frac{\sigma n}{\kappa})^2$. When the dominated Gaussian assortative KSBM (see Thm. \ref{thm:dominated-gaussian-assortative-KSBM}) enters steady state after some time $T>0$, it is therefore $\frac{\pi}{2}(\frac{\sigma n}{\kappa})^2$-variance-clustered for each community. Since $n\sigma \ll \kappa$, it follows that $\frac{\pi}{2}(\frac{\sigma n}{\kappa})^2 \ll 1$, and we have by Lemma \ref{lemma:variance-clustered-implies-clustered-whp} that our system is $\sqrt{\frac{\pi}{2}}\frac{\sigma n}{\kappa}$-clustered with high probability by the same time.
\end{proof}

To conclude our derivations, we remove Hypothesis \ref{hypothesis:intra-community-dominated} and consider the general Gaussian assortative KSBM, which we will use to bridge the clusterization regime to the mean-field regime.

\begin{theorem}
\label{thm:gaussian-assortative-KSBM}
    (Gaussian Assortative KSBM)\\
    Under Hypothesis \ref{hypothesis:gaussian-phase} with $m \gg 1$ and $n\sigma \ll \kappa$, the assortative KSBM with community mean $\theta^{A}_{G_r}(t)$ and variance $V^{A}_{G_r}(t)$ for $r\in [n]$ follows,
    $$d\theta^{A}_{G_r} = \omega_{G_r} + \frac{2\kappa}{N(n-1)}e^{-\frac{V^{A}_{G_r}}{2}}\sum_{s=1}^n{\sin(\theta^{A}_{G_s} - \theta^{A}_{G_r})e^{-\frac{V^{A}_{G_s}}{2}}} \, dt$$
    $$dV^{A}_{G_r} = \epsilon - \frac{2\kappa}{n}V^{A}_{G_r}e^{-V^{A}_{G_r}} - \frac{4\kappa}{N(n-1)}V^{A}_{G_r}e^{-\frac{V^{A}_{G_r}}{2}}\sum_{s=1, s\neq r}^n{\cos(\theta^{A}_{G_s} - \theta^{A}_{G_r})e^{-\frac{V^{A}_{G_s}}{2}}}$$
    where $\epsilon \leq \frac{\sigma^2\pi n}{\kappa}$
\end{theorem}
\begin{proof}
    Let us first derive the differential equation for $\theta^{A}_{G_r}$. Recall that in an assortative KSBM for $\theta^{A}_{G_r} = \avg{\theta^{A}_i}_{i\in G_r}$, using Lemma \ref{lemma:expectation-sin-theta-is-0} and Law of Large Number (LLN),
    $$\dot{\theta}_{G_r}(t) = \omega_{G_r} + m\sum_{s\neq r}^n{\frac{\kappa}{N}\E[i\in G_r, j\in G_s]{\sin(\theta^{A}_j(t) - \theta^{A}_i(t))\I\{j \sim i\}}}.$$
    Using conditional expectation on $j \sim i$ and $\E{\I\{j\sim i\}} = P_{rs} = \frac{2}{m(n-1)}$,
    $$\E[i\in G_r, j\in G_s]{\sin(\theta^{A}_j(t) - \theta^{A}_i(t))\I\{j \sim i\}} = \frac{2}{m(n-1)}\E[i\in G_r, j\in G_s]{\sin(\theta^{A}_j(t) - \theta^{A}_i(t))|j\sim i}.$$
    Since $\theta^{A}_j(t) - \theta^{A}_i(t) \sim \gauss{\theta^{A}_{G_s}(t) - \theta^{A}_{G_r}(t)}{V^{A}_{G_r}(t) + V^{A}_{G_s}(t)}$ as a sum of independent Gaussian by assumption, it follows,
    $$\E[i\in G_r, j\in G_s]{\sin(\theta^{A}_j(t) - \theta^{A}_i(t))\I\{j \sim i\}} = \frac{2}{m(n-1)}\sin(\theta^{A}_{G_s}(t) - \theta^{A}_{G_r}(t))e^{-\frac{V^{A}_{G_r}(t) + V^{A}_{G_s}(t)}{2}}.$$
    Hence,
    $$\dot{\theta}^{A}_{G_r}(t) = \omega_{G_r} + \frac{2\kappa}{N(n-1)}\sum_{s\neq r}^n{\sin(\theta^{A}_{G_s}(t) - \theta^{A}_{G_r}(t))e^{-\frac{V^{A}_{G_r}(t) + V^{A}_{G_s}(t)}{2}}}.$$
    For $V^{A}_{G_r}$ if we look back at $\dot{V}^{A}_{G_r}(t) \sim \frac{2n}{N}\sum_{i\in G_r}{(\theta^{A}_i(t)-\theta^{A}_{G_r}(t))(\dot{\theta}^{A}_i(t)-\dot{\theta}^{A}_{G_r}(t))}$ and now use the non-truncated equation for $\dot{\theta}^{A}_i(t)$,
    $$\dot{\theta}^{A}_i(t) = \omega_i + \frac{\kappa}{N}\sum_{j \in G_r}{\sin(\theta^{A}_j(t) - \theta^{A}_i(t))} + \frac{\kappa}{N}\sum_{s \neq r}^n{\sum_{j \in G_s}{\sin(\theta^{A}_j(t) - \theta^{A}_i(t))\I\{j \sim i\}}}.$$
    Letting $p = P_{rs} = \frac{2}{m(n-1)}$, we get,
    \begin{equation*}
    \begin{split}
        \dot{V}^{A}_{G_r}(t) \sim \frac{2n}{N}\sum_{i\in G_r}(\theta^{A}_i(t)-\theta^{A}_{G_r}(t))\bigg(\omega_i & - \omega_{G-r} + \frac{\kappa}{N}\sum_{j \in G_r}{\sin(\theta^{A}_j(t) - \theta^{A}_i(t))}\\
        & + \frac{\kappa}{N}\sum_{s \neq r}^n{\sum_{j \in G_s}{\sin(\theta^{A}_j(t) - \theta^{A}_i(t))\I\{j \sim i\}}}\\
        & - \frac{\kappa p}{n}\sum_{s\neq r}^n{\sin(\theta^{A}_{G_s}(t) - \theta^{A}_{G_r}(t))e^{-\frac{V^{A}_{G_r}(t) + V^{A}_{G_s}(t)}{2}}}\bigg).
    \end{split}
    \end{equation*}
    Which we can express as expectations using the LLN and conditional expectations,
    \begin{equation*}
    \begin{split}
        \dot{V}^{A}_{G_r}(t) \sim & \ 2\E[i\in G_r]{(\theta^{A}_i(t)-\theta^{A}_{G_r}(t))(\omega_i - \omega_{G-r} + \frac{\kappa}{n}\E[j \in G_r]{\sin(\theta_j(t) - \theta_i(t))})}\\
        & + \frac{2\kappa}{n}\E[i\in G_r]{(\theta^{A}_i(t)-\theta^{A}_{G_r}(t))\sum_{s \neq r}^n{\E[j \in G_s]{\sin(\theta_j(t) - \theta_i(t))|j\sim i}P_{rs}}}\\
        & - \frac{2\kappa p}{n}\E[i\in G_r]{(\theta^{A}_i(t)-\theta^{A}_{G_r}(t))\sum_{s\neq r}^n{\sin(\theta^{A}_{G_s}(t) - \theta^{A}_{G_r}(t))e^{-\frac{V^{A}_{G_r}(t) + V^{A}_{G_s}(t)}{2}}}}.
    \end{split}
    \end{equation*}
    The first term we have already computed in Theorem \ref{thm:dominated-gaussian-assortative-KSBM}, directly gives us,
    $\epsilon - \frac{2\kappa}{n}V^{A}_{G_r}(t)e^{-V^{A}_{G_r}(t)}$
    with $\epsilon \leq \frac{\sigma^2\pi n}{\kappa}$. The second term can be computed in a similar fashion using Gaussian expectations,
    \begin{equation*}
    \begin{split}
        & \frac{2\kappa}{n}\E[i\in G_r]{(\theta^{A}_i(t)-\theta^{A}_{G_r}(t))\sum_{s \neq r}^n{\E[j \in G_s]{\sin(\theta^{A}_j(t) - \theta^{A}_i(t))|j\sim i}P_{rs}}}\\
        & \quad = -\frac{2\kappa p}{n}\sum_{s \neq r}^n{e^{-\frac{V^{A}_{G_s}(t)}{2}}\E[i\in G_r]{(\theta^{A}_i(t)-\theta^{A}_{G_r}(t))\sin(\theta^{A}_i(t) - \theta^{A}_{G_s}(t))}},\\
        & \quad = -\frac{2\kappa p}{n}V^{A}_{G_r}(t)e^{-\frac{V^{A}_{G_r}(t)}{2}}\sum_{s \neq r}^n{\cos(\theta^{A}_{G_r}(t) - \theta^{A}_{G_s}(t))e^{-\frac{V^{A}_{G_s}(t)}{2}}}.
    \end{split} 
    \end{equation*}
    Finally, remark that the last term vanishes since $\E[i\in G_r]{\theta^{A}_i(t)-\theta^{A}_{G_r}(t)} = \theta^{A}_{G_r}(t)-\theta^{A}_{G_r}(t)=0$,
    \begin{equation*}
    \begin{split}
        & - \frac{2\kappa p}{n}\E[i\in G_r]{(\theta^{A}_i(t)-\theta^{A}_{G_r}(t))\sum_{s\neq r}^n{\sin(\theta^{A}_{G_s}(t) - \theta^{A}_{G_r}(t))e^{-\frac{V^{A}_{G_r}(t) + V^{A}_{G_s}(t)}{2}}})}\\
        & = - \frac{2\kappa p}{n}\sum_{s\neq r}^n{\sin(\theta^{A}_{G_s}(t) - \theta^{A}_{G_r}(t))e^{-\frac{V^{A}_{G_r}(t) + V^{A}_{G_s}(t)}{2}}\E[i\in G_r]{\theta^{A}_i(t)-\theta^{A}_{G_r}(t)}} = 0.
    \end{split}
    \end{equation*}
    Putting everything together gets us the desired claim.
\end{proof}

The fixed points of this model are much more difficult to analyze than for the dominated model (Theorem \ref{thm:dominated-gaussian-assortative-KSBM}). One can see that we have added another term to the clustering dynamic,
$$- \frac{4\kappa}{N(n-1)}V^{A}_{G_r}e^{-\frac{V^{A}_{G_r}}{2}}\sum_{s=1, s\neq r}^n{\cos(\theta^{A}_{G_s} - \theta^{A}_{G_r})e^{-\frac{V^{A}_{G_s}}{2}}}$$
where $b_{rs} := \cos(\theta^{A}_{G_s} - \theta^{A}_{G_r})$ controls whether $G_s$ aids ($b_{rs} > 0$) or opposes ($b_{rs} < 0$) clustering of $G_r$.

One can see this phenomenon as arising from the $\sin(\theta^{A}_j - \theta^{A}_i)$ coupling, which is approximately linear for difference less than $\frac{\pi}{2}$. Indeed, if we have two communities $G_r, G_s$ close to each other (i.e. $\abs{\theta^{A}_{G_s} - \theta^{A}_{G_r}} < \frac{\pi}{2}$), $\theta^{A}_{G_s}$ acts linearly on the oscillators in $G_r$ in the differential equation, resulting in faster synchronization. Furthermore this contribution of other communities is sensibly stronger the more clustered they are, due to the factor $e^{-\frac{V^{A}_{G_s}}{2}}$. It is also important to note that their contribution scale with $p = \frac{2}{m(n-1)}$ and thus in the large $m$ limit, we get back to our dominated Gaussian assortative KSBM since intra-community coupling scales as $O(m)$ and inter-community coupling as $O(1)$ and are thus negligible during clusterization.

\subsection{Transition Time}

\begin{proof}
\label{proof:lemma:gaussian-transition-time-bound}
    \textbf{Lemma \ref{lemma:gaussian-transition-time-bound}}\\
    Let $r\in [n]$. When the communities are not synchronized $C_{inter}(r) = O(\frac{\kappa}{N}m(n-1)\frac{2}{m(n-1)}) = O(\frac{\kappa}{N})$ by bounding all sine terms by $1$, i.e. we are counting the expected number of non-zero terms.\\
    If $C_{intra}(r)$ matches $C_{inter}(r)$, then $C_{intra}(r) = O(\frac{\kappa}{N})$. That is $\sum_{j \in G_r}\sin(\abs{\theta_j(t) - \theta_i(t)}) = O(1)$. Under a Gaussian assortative KSBM, $\theta_i(t) -\theta_j(t) \sim \gauss{0}{2V(t)}$. Hence without loss of generality, we need $\sin(\abs{\theta_j(t) - \theta_i(t)}) = O(\frac{1}{m})$. Since for large $m \gg 1$ the above is small, then $\theta_j(t) - \theta_i(t) \ll 1$ also, and the condition becomes $\abs{\theta_j(t) - \theta_i(t)} = O(\frac{n}{N})$. Using Chebyshev's inequality,
    $$\p{\abs{\theta_j(t) - \theta_i(t)} \geq \frac{1}{m}} \leq 2V(t)m^2.$$
    To conclude, it suffices to remark that letting $V(t) = (\frac{1}{m})^{2 + \nu}$ bounds the probability by zero when $m \rightarrow \infty$.
\end{proof}

\begin{corollary}
\label{cor:transition-time-bound}
    Under the same condition as Lemma \ref{lemma:gaussian-transition-time-bound}, the same result holds for an assortative KSBM with intra-community variance $V(t)$.
\end{corollary}
\begin{proof}
    It suffices to remark that $\Var(\theta_i(t) -\theta_j(t)) \leq 2V(t)$. We can then bound by Chebyshev's inequality,
    $$\p{\abs{\theta_j(t) - \theta_i(t)} \geq \frac{1}{m}} \leq \Var(\theta_j(t) - \theta_i(t))m^2 \leq 2V(t)m^2.$$
\end{proof}

\begin{proof}
\label{proof:lemma:consistency-gaussian-assortative-KSBM}
    \textbf{Lemma \ref{lemma:consistency-gaussian-assortative-KSBM}}\\
    From Lemma \ref{lemma:gaussian-transition-time-bound}, we know that $t \gg t^*$, where $t^*$ is such that the intra-community variance $V(t^*) \leq 2(\frac{1}{m})^{2+\nu}$ for some $\nu > 0$. Since $m \rightarrow \infty$, then $N \rightarrow \infty$ and thus $V(t^*)$ is arbitrarily small.\\
    We need to show that $\dot{\Theta}^{MF}(t) = \dot{\Theta}^{A}(t)$. Notice that since $V^{A}_{G_r}(t^*) \rightarrow 0$ for any $r\in [n]$, then in Theorem \ref{thm:gaussian-assortative-KSBM}, $\dot{\theta}^{A}_{G_r} = \omega_{G_r} + \frac{2\kappa}{N(n-1)}\sum_{s=1}^n{\sin(\theta^{A}_{G_s} - \theta^{A}_{G_r}(t))}$ and $\dot{V}^{A}_{G_r} = \epsilon$ which is exactly $\dot{\Theta}^{MF}(t)$ up to increasing variance $V^{A}_{G_r}$, but since $\sigma^2 \ll \kappa$, then $\epsilon \rightarrow 0$ which allows us to conclude the first claim.\\
    \\
    For the second claim, since $t \ll t_{\trans}$, we know that the clusterization has not taken place, and thus the intra-community variance $V(t)$ is large. Furthermore taking into account $N \rightarrow \infty$, Theorem \ref{thm:gaussian-assortative-KSBM} becomes, $d\theta^{A}_{G_r} = \omega_{G_r}$ and $dV^{A}_{G_r} = \epsilon - \frac{2\kappa}{n}V^{A}_{G_r}e^{-V^{A}_{G_r}}$ which is exactly the formulation of Theorem \ref{thm:dominated-gaussian-assortative-KSBM}.
\end{proof}

\begin{lemma}
\label{lemma:transition-time-assortative-KSBM}
    (Transition Time in Assortative KSBM)\\
    Suppose that $m \rightarrow \infty$ and $n\sigma \ll \kappa$ for an assortative KSBM $\Theta^{A}$ and $t \gg t_{\trans}$, then, $\dot{\Theta}^{A}(t) = \dot{\Theta}^{MF}(t)$ where $\Theta^{MF}$ is the Mean-Field KSBM of Theorem \ref{thm:MF-KSBM}.
\end{lemma}
\begin{proof}
    We want to show that when $t \gg t_{\trans}$ then $\Theta^{A}(t)$ is $\epsilon'$-clustered with $\epsilon' \ll 1$ and $\epsilon' \ll \frac{N(n-1)}{\kappa}$. From Lemma \ref{cor:transition-time-bound}, we know that $t \gg t^*$, where $t^*$ is such that $V(t^*) \leq 2(\frac{1}{m})^{2+\nu}$ for some $\nu > 0$. Since $m \rightarrow \infty$, then $N \rightarrow \infty$ and thus $V(t^*)$ is arbitrarily small. This tell us by Lemma \ref{lemma:variance-clustered-implies-clustered-whp} that $\Theta^{MF}(t^*)$ is $\epsilon' := \frac{\sqrt{2}n}{N}$-clustered with probability tending to $1$ as $N \rightarrow \infty$. It immediately follows that $\epsilon' \ll 1$. From the definition $C_{rs} = \kappa/N$, it is straightforward to check that $\bar{C} = O(1/N^2)$. Thus the conditions for Theorem \ref{thm:MF-KSBM} to hold are verified and therefore $\dot{\Theta}^{A}(t) = \dot{\Theta}^{MF}(t)$.
\end{proof}

\subsection{Regime-Split Lead Matrices}

\begin{proof}
\label{proof:lemma:convergence-S-assortative-KSBM}
    \textbf{Lemma \ref{lemma:convergence-S-assortative-KSBM}}\\
    Using Lemma \ref{lemma:transition-time-assortative-KSBM}, we know that $\Theta$ is $\epsilon' := \frac{\sqrt{2}n}{N}$-clustered by time $t$ almost surely asymptotically, where $\epsilon' \rightarrow 0$. Thus $\theta_i \longrightarrow \theta_{G_r}$ for $i \in G_r, r\in [n]$ when $m \rightarrow \infty$. Since the clustering holds by time $t$, it also holds that $\dot{\theta}_i \longrightarrow \dot{\theta}_{G_r}$.\\
    By continuity of $\dot{f}$, for any $i \in G_r, r\in [n]$ and $t' \in [t, T]$,
    $$\dot{\gamma}_i(t') = \dot{f}(\theta_i(t'))\dot{\theta}_i(t') \longrightarrow \dot{f}(\theta_{G_r}(t'))\dot{\theta}_{G_r}(t') = \dot{\bar{\gamma}}_{G_r}(t').$$
    Hence,
    $$S_{i_1...i_m}(\gamma)(T) = \int_{\Delta^m(\overrightarrow{t})_0^T}{\prod_{j\in[m]}\dot{\gamma}_{i_j}(t_j)\,dt_j} \longrightarrow \int_{\Delta^m(\overrightarrow{t})_0^T}{\prod_{j\in[m]}\dot{\bar{\gamma}}_{G_{r_j}}(t_j)\,dt_j} = S_{G_{r_1}...G_{r_m}}(\bar{\gamma})(T).$$
\end{proof}

\begin{lemma}
\label{lemma:expectation-variance-lead-matrix}
    (Expectation and Variance of Lead Matrix for KSBM)\footnote{The expectation and variance are understood here in the sense of sample-average and sample-variance.}\\
    Suppose $\Theta$ a realisation of a KSBM $\cK$. If $\Theta$ is $\epsilon$-clustered and in synchronized steady state by time $t$, such that $\epsilon \ll 1$, $\avg{C_{rs}P_{rs}}_{r\neq s\in [n]}=O(\frac{1}{N})$ and $m \gg 1$, then for any $r, s \in [n]$ we have,
    \begin{equation*}
    \begin{split}
        \E[i\in G_r, j\in G_s]{L^{SS}_{ij}(\sin(\Theta)} & = L^{SS}_{G_rG_s}(\sin(\bar{\Theta})) + O(\epsilon\omega T),\\
        \Var_{i\in G_r, j\in G_s}(L^{SS}_{ij}(\sin(\Theta))) & = O((\epsilon\omega T)^2).
    \end{split}
    \end{equation*}
\end{lemma}
\begin{proof}
    By Lemma \ref{lemma:S-sin-theta}, we know for $i \in G_r, j\in G_s$ and $r,s \in [n]$,
    $$L^{SS}_{ij}(\sin(\Theta))(T) = \frac{\sin(\Delta\theta_{ij})}{2}(\omega T + \sin(\omega T)).$$
    Using, Theorem \ref{thm:MF-KSBM}, it directly follows that
    \begin{equation*}
    \begin{split}
        L^{SS}_{ij}(\sin(\Theta))(T) & = \frac{(\sin(\Delta\theta_{G_rG_s})+O(2\epsilon))}{2}(\omega T + \sin(\omega T)),\\
        & = \frac{\sin(\Delta\theta_{G_rG_s})}{2}(\omega T + \sin(\omega T)) + O(\epsilon\omega T),\\
        & = L^{SS}_{G_rG_s}(\sin(\bar{\Theta}))(T) + O(\epsilon\omega T).
    \end{split}
    \end{equation*}
    We now consider the expectation,
    $$\E[i\in G_r, j\in G_s]{L^{SS}_{ij}(\sin(\Theta))(T)} = L^{SS}_{G_rG_s}(\sin(\bar{\Theta}))(T) + O(\epsilon\omega T),$$
    and variance,
    $$\Var_{i\in G_r, j\in G_s}(L^{SS}_{ij}(\sin(\Theta))(T)) = \E[i\in G_r, j\in G_s]{(\sin(\Theta))(T) - L^{SS}_{G_rG_s}(\sin(\bar{\Theta}))(T))^2} = O((\epsilon\omega T)^2).$$
\end{proof}

\begin{proof}
\label{proof:lemma:L-sin-theta}
    \textbf{Lemma \ref{lemma:L-sin-theta}}\\
    See Lemma \ref{lemma:S-sin-theta}.
\end{proof}

\begin{lemma}
\label{lemma:frequency-noise-lead-matrix}
    (Frequency {Heterogeneity} in Lead Matrix)\\
    Let $\Theta$ a realisation of a KSBM $\cK$ and $\bar{\Theta} := (\theta_{G_1},...,\theta_{G_n})$, suppose $\Theta$ is in synchronized steady state by time $t$, then if intrinsic frequency {heterogeneity} $\sigma$ is such that for all $r\in [n]$,
    $$\frac{\sigma}{mC_{rr}P_{rr}} \ll 1 \text{ and } C_{rr}P_{rr} \gg \sum_{s \neq r}{C_{rs}P_{rs}}.$$
    Then for any $i \in G_r, j\in G_s$ and $r,s\in[n]$ and $T$ large enough such that $\omega T \gg 1$,
    $$\abs{L^{SS}_{ij}(\sin(\Theta))(T)-L^{SS}_{G_rG_s}(\sin(\bar{\Theta}))(T)} = O(\frac{\sigma}{m}(\frac{1}{C_{rr}P_{rr}} + \frac{1}{C_{ss}P_{ss}})\omega T).$$
\end{lemma}
\begin{proof}
    It suffices to remark that by Lemma \ref{lemma:deviation-synchronized-steady-state}, we have for $i \in G_r, r\in [n]$,
    $\abs{\theta_i - \theta_{G_r}} \leq \arcsin(\frac{O(\sigma)}{m\sum_sC_{rs}P_{rs}}) =: \epsilon'$. Hence $\Theta$ is $\epsilon'$-clustered, and given $C_{rr}P_{rr} \gg \sum_{s\neq r}{C_{rs}P_{rs}}$ and $\sigma \ll mC_{rr}P_{rr}$, we have,
    $$\epsilon' = O(\frac{\sigma}{mC_{rr}P_{rr}}),$$
    by approximate linearity of $arcsin$. We can then apply Lemma \ref{lemma:expectation-variance-lead-matrix} to conclude by remarking that we can then take $\epsilon = \frac{\sigma}{m}(\frac{1}{C_{rr}P_{rr}} + \frac{1}{C_{ss}P_{ss}})$ when dealing with communities $G_r$ and $G_s$.
\end{proof}

\subsection{Block-Clustering Metric}

We give some properties about community homogeneity and discriminativity as measures of values in a matrix being shared within a community and distinct across communities.

\begin{proposition}
\label{prop:homogeneousness}
     Let $B \in \R^{N\times N}$, $\lambda \in R$ and community assignment $(G_r)_{r\in[n]}$, then,
     \begin{itemize}
         \item $h(B|G) \geq 0$ and $h(B|G) = 0 \iff \exists C \in \R^{n\times n} \text{ s.t. } B = Tile(C)$,
         \item $h(\lambda B|G) = \lambda^2h(B|G)$,
     \end{itemize}
     where $Tile:\R^{n\times n} \rightarrow \R^{N\times N}$ transforms a $n\times n$ matrix to size $N\times N$ by repeating each entry in submatrices of size $m\times m$.
\end{proposition}
\begin{proof}
    The positivity follows from the fact that the variance is positive, hence the sum is also positive. Since a sum of positive summands is zero if and only if every summand is zero, we know that the variance inside each submatrix $B_{G_rG_s}$ is zero. Hence each submatrix is constant, and thus $B$ can be obtained by taking a matrix in $\R^{n\times n}$ and repeating each entry across $(m)^2$-sized submatrices. The last property holds for variance in particular, and by linearity of the average, the claim follows.
\end{proof}

\begin{proposition}
\label{prop:discriminativeness}
     Let $B \in \R^{N\times N}$, $\lambda \in R$ and community assignment $(G_r)_{r\in[n]}$, then,
     \begin{itemize}
         \item $d(B|G) \geq 0$ and $d(B|G) = 0 \iff \exists c \in \R \text{ s.t. } \bar{A} = Tile(c)$,
         \item $d(\lambda B|G) = \lambda^2d(B|G)$,
     \end{itemize}
     where $Tile: \R \rightarrow \R^{n\times n}$ is such that for any $i,j \in [N]$, $Tile(c)_{ij} = c$.
\end{proposition}
\begin{proof}
    As a sum of square, it is positive. Furthermore, it is zero if and only if each square distance is zero, i.e., all average entries of $B$ are the same, that is $\bar{B}$ is obtained by tiling a scalar. The last property holds for the square distance and is preserved by linearity of the average. 
\end{proof}

\begin{proposition}
\label{prop:clustering}
    Let $B \in \R^{N\times N}$ and $\lambda \in \R^*$ with community assignment $(G_r)_{r\in[n]}$, then,
    \begin{itemize}
        \item $g(B|G) \geq 0$,
        \item $g(\lambda B|G) = g(B|G)$.
    \end{itemize}
\end{proposition}
\begin{proof}
    As a ratio of two positive values, it follows $g(B|G)$ also is positive. For the second property, since $h(\lambda B|G) = \lambda^2h(B|G)$ and $d(\lambda B|G) = \lambda^2d(B|G)$, then $\lambda^2$ cancels out in the ratio. 
\end{proof}

\section{Hierarchical KSBM}
\label{appendix:sec:hierarchical-ksbm}

{In order to compare our experimental results with Arenas \emph{et al.} ~\cite{arenas2006-sync-top-comp-net}, we formulated a two-level hierarchical variant of the KSBM\footnote{Similarly to the assortative KSBM which is strictly speaking not a KSBM model, this hierarchical KSBM is also not a KSBM model. Nonetheless, in the large $m$ limit, the mean-field theorem tells us that they are dynamically equivalent. For all intents and purpose, the hierarchical KSBM defined below can therefore be thought of as a submodel of the KSBM.}. The \emph{hierarchical KSBM} $\cHK = \cHK(n_1, n_2, r,m,\kappa, \mu, \sigma, \theta^0)$ is a random Kuramoto model defined by Equation~\ref{eq:generalized-kuramoto} on $n_1$ fine-grained (first level) and $n_2$ coarse-grained communities (second level) (each consisting of $n_1/n_2$ sub-communities) such that:}
\begin{itemize}
    \item fine-grained communities are fully connected, $P_{rr} = 1$ for any $r \in [n_1]$;
    \item for each node $i \in G_r$ in a coarse-grained community $\hat{G}_q := \bigcup_{k=1}^{n_2}G_{r_{q+k}}$ with $q\in [n_2]$, we add $rm(n_2-1)$ edges between node $i$ and other nodes which do not belong to the same fined-grained community $G_r$ but are in the same coarse-grained community $\hat{G}_q$,
    \item for each node $i \in \hat{G}_q$, we add one edge between node $i$ and a node in another coarse-grained community $\hat{G}_p$ where $p\neq q$, and 
    \item the coupling strengths are uniform, $\oC_{ij} = \frac{\kappa}{N}$ for all $i,j\in [n]$ for some $\kappa \in \R$;
    \item intrinsic frequencies $\omega_i$ and initial conditions $\theta_i(0)$ are set according to Definition~\ref{def:ksbm}.
\end{itemize}

{By construction, in the final graph each node may be connected to more than $rm(n_2-1)$ nodes outside its fine-grained community but within its coarse-grained community. Outside of its coarse-grained community it may be connected to more than one node.}

{We introduced the terminology for fine- and coarse-grained communities for clarity, they correspond to the first and second-level communities discussed in Arenas \emph{et al.}~\cite{arenas2006-sync-top-comp-net}.}

{In our numerical experiment, we used $n_1=9$, $n_2=3$, $m=33$ and $\mu_h$ evenly spaced points in the range $[\frac{2}{9}, 2 \ rad/s]$ with $\sigma=0.1$. We chose the ratio $r = \frac{0.1}{(n_1/n_2) -1}$ in order to have the number of edges within coarse-grained community be comparable to that of Arenas \emph{et al.} \cite{arenas2006-sync-top-comp-net}. We chose $\kappa = 300$ to ensure all coarse-grained communities are synchronized in the end, technically resulting in a final third hierarchical level.}

{The resulting dynamics (Figure~\ref{fig:KSBM-dynamics-hierarchical}) show our transition time which marks the end of the clusterization regime to the formation of the first level communities, and the later transition to the steady state where all second level communities synchronize. In the transient regime, we can see that there is a second transition where the second level communities form (around $t=0.5s$) which could warrant a further regime split. Correspondingly, the lead and covariance matrices show in the transient regime a block structure which separates some of the first-level communities, whereas the steady state block structure is defined around the second-level structure (as the intrinsic frequencies between first level communities are not separated enough with respect to heterogeneity).}

\begin{figure}[H]
\centering
\includegraphics[width=0.92\linewidth]{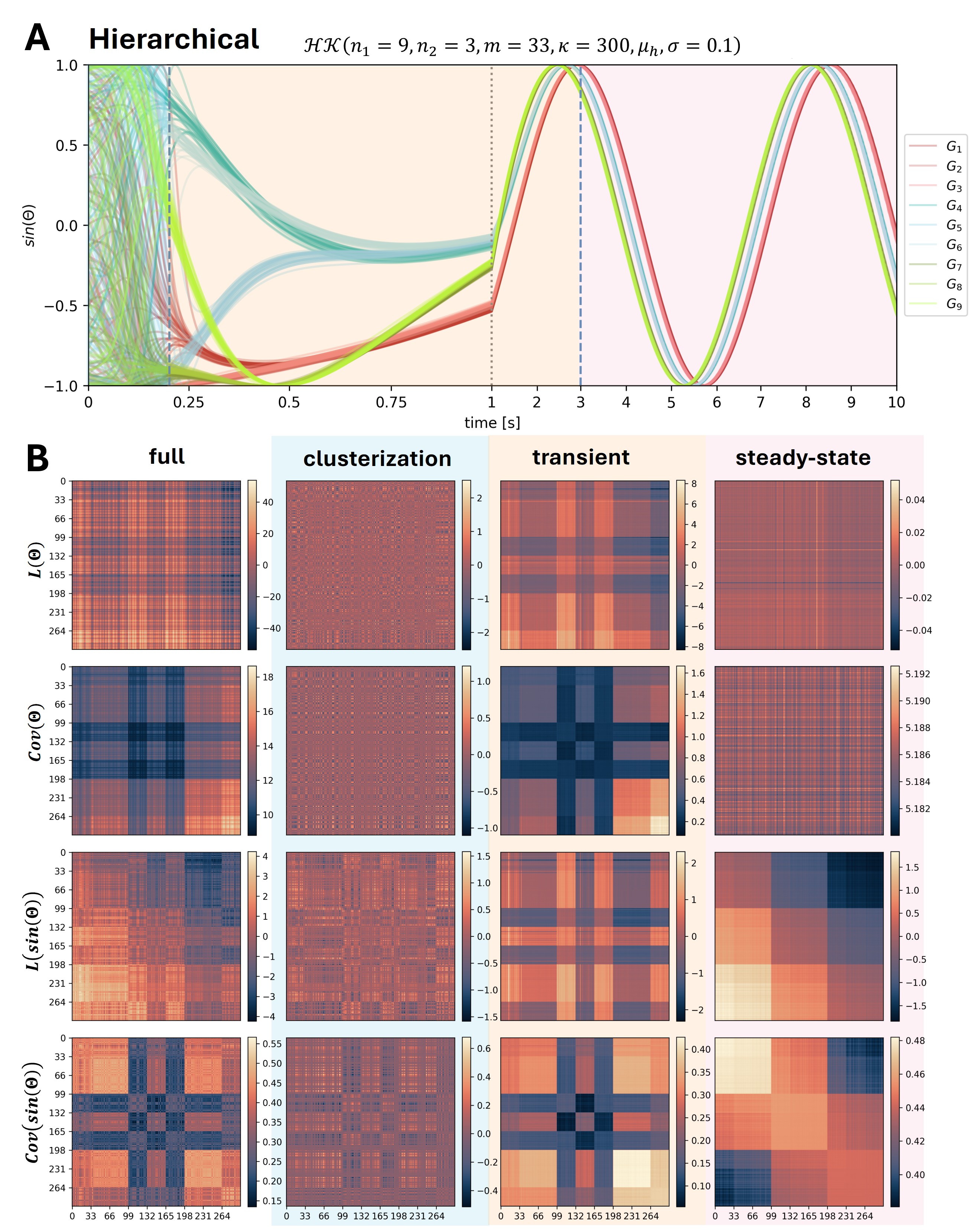}
\caption{\textbf{Hierarchical KSBM time series and lead matrices}\\
(A) Time series of the hierarchical KSBM $\cHK(n_1=9,n_2=3,r=0.05,m=33,\kappa=300,\mu_h, \sigma=0.1)$ with first level communities $G_1, ..., G_9$ and each second level community composed of three first level communities in sequential order. (B) Lead and covariance matrices computed from the full time series and each temporal regime for $\Theta$ and $\sin(\Theta)$.}
\label{fig:KSBM-dynamics-hierarchical}
\end{figure}

{We evaluated the performance of our SCE algorithm on the hierarchical KSBM, with the goal to either retrieve the first or the second level structure (Figure~\ref{fig:hierarchical-agreement}). Our algorithm is unable to perfectly extract the first level, but can still recover a significant part in the transient regime for $Cov(\sin(\Theta))$. The second regime can be reliably recovered in the steady state for $\sin(\Theta)$, and perfectly if we prune the community assignment to $n_2$ communities (see next section). This means that the SCE community estimations lie in between the first and second level, distinguishing part of the first level communities, but not all of them.}

\begin{figure}[ht!]
  \centering
  \includegraphics[width=1\linewidth]{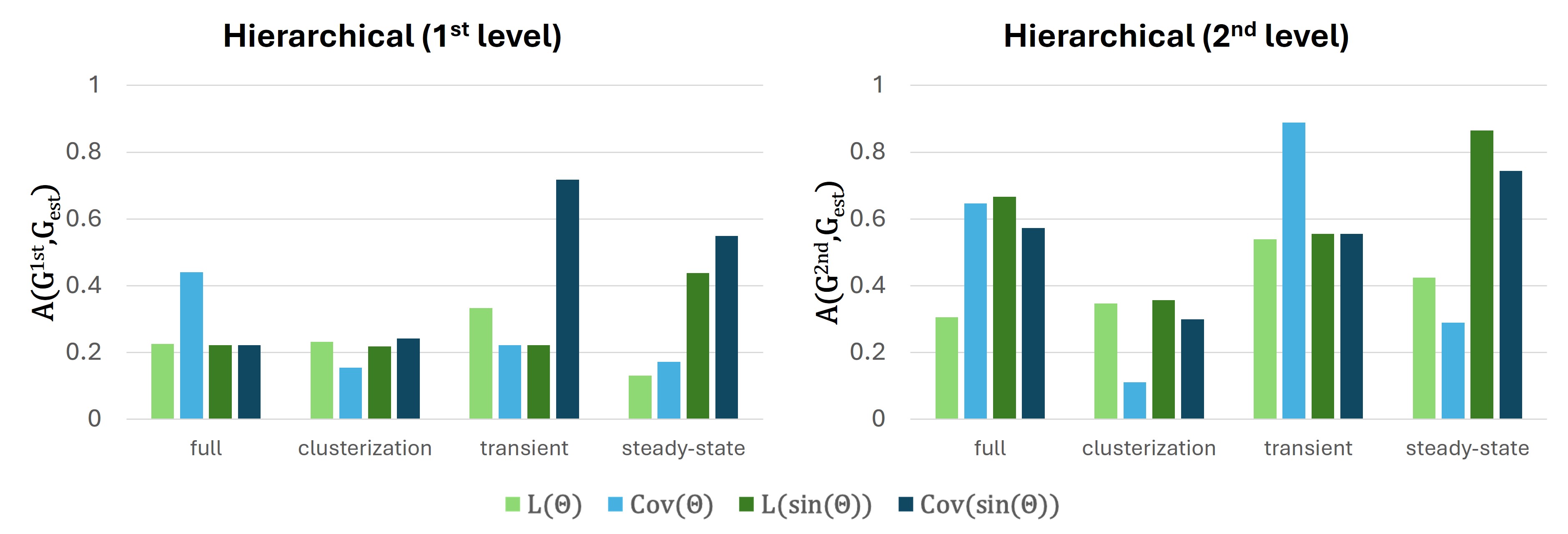}
  \caption{\textbf{SCE performance in hierarchical KSBM}\\
  Agreement $A(G,G_{est})$ and $A(\hat{G},G_{est})$ between the true first and second level community assignment $G$ and $\hat{G}$ and SCE algorithm $G_{est}$ over the full dynamic and each regime-split for the hierarchical KSBM. Exact recovery corresponds to an agreement of $1$ while for random assignment of the true number of community ($n_1$ or $n_2$) the agreement is $1/n_1$ (respectively $1/n_2$), but can be lower if the number of communities estimated is larger than $n_1$ or $n_2$. Significant recovery of the first level is possible in the transient regime for $Cov(\sin(\Theta))$, and recovery of the second level is possible with pruning in steady state for $L(\sin(\Theta))$ and $Cov(\sin(\Theta))$.}
  \label{fig:hierarchical-agreement}
\end{figure}

\section{Stochastic KSBM}
\label{appendix:sec:stochastic-ksbm}

{We also investigated the effect of noise added to the phase of each oscillator during simulation. More formally, we adapted the differential equations of the generalized Kuramoto model (Equation~\ref{eq:generalized-kuramoto}) to a system of stochastic differential equations.}

\begin{equation}
\label{eq:stochastic-generalized-kuramoto}
    d\theta_i(t) = \dot{\theta}_i(t)dt + b\,dW(t).
\end{equation}

{Where $W(t)$ is a Brownian motion scaled with Brownian noise parameter $b$. To solve the dynamics we used a forward Euler-Maruyama scheme.}

{We compared the performance in term of agreeement of the SCE and traditional algorithms such as $K$-means and hierarchical clustering (for various linkage) on the standard KSBM configuration using the stochastic differential equations above instead of the usual Kuramoto model (Figure~\ref{fig:stochastic-ksbm-agreement}). We refer to this model as the \emph{Stochastic KSBM}.}

{In order to compare our SCE algorithm with no a priori knowledge on the number of communities to traditional algorithms, we compared the methods for both the correct number $K=3$ and a wrong assumption $K=6$. We also considered a pruned version of the SCE, where we build a hierarchical clustering on the community assignment of the SCE which merges the communities based on the smallest average distance $\avg{D_{ij}}_{i\in G_r, j\in G_s}$ between two communities where $D$ is the distance matrix defined in Alg.~\ref{alg:structural-community-estimation}.}

\begin{figure}[ht!]
  \centering
  \includegraphics[width=1\linewidth]{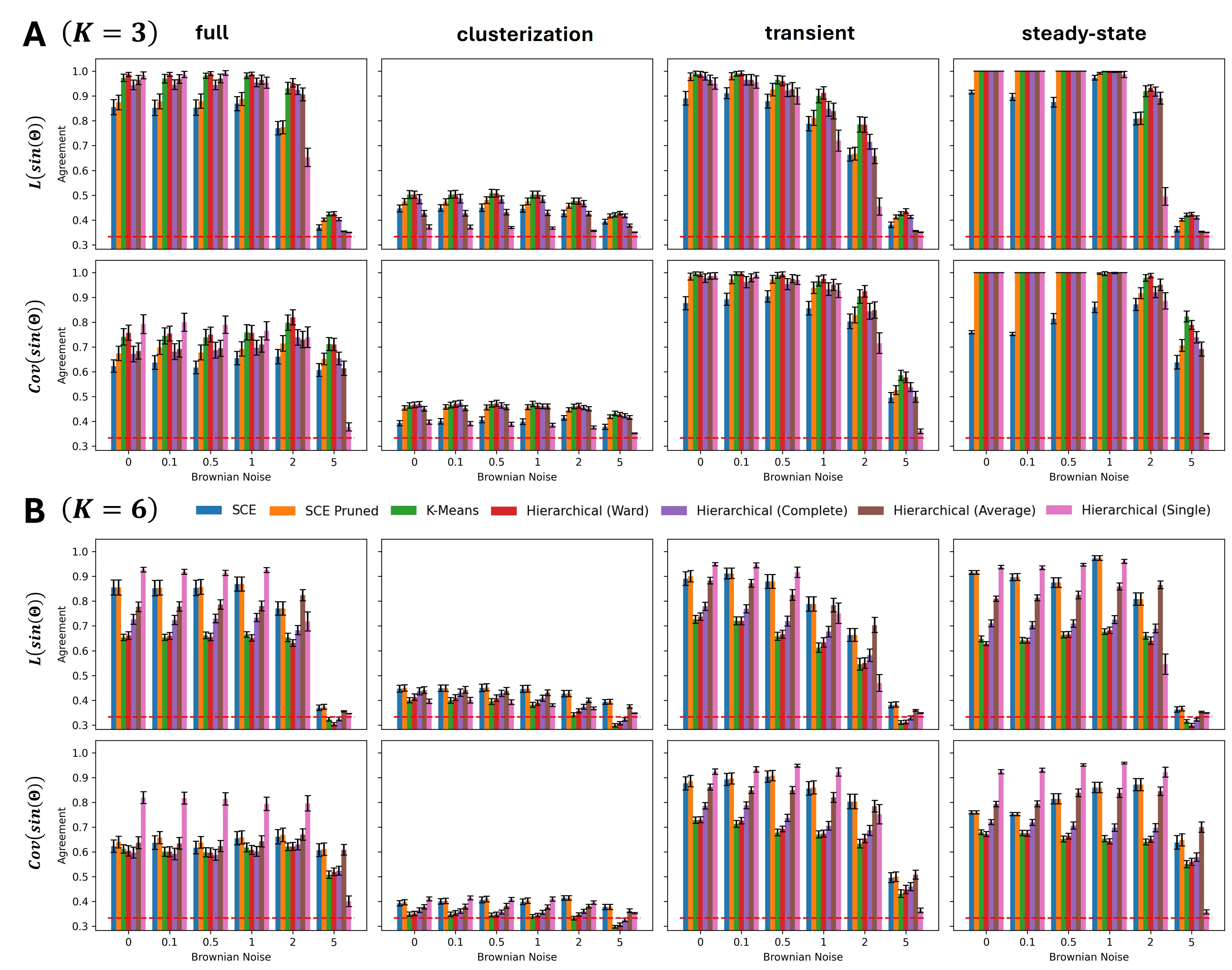}
  \caption{\textbf{Community Recovery in Stochastic KSBM}\\
  Agreement $A(G,G_{est})$ between the true community assignment $G$ and algorithms $G_{est}$ (with assumed number of communities (A) $K=3$ or (B) $K=6$) over the full dynamic and each regime-split for the Standard Stochastic KSBM at various Brownian noise level $b$ across $100$ simulations per noise level (the underlying graph is kept fixed). Exact recovery corresponds to an agreement of $1$ while for random assignment of the true number of community ($n$) the agreement is $1/n$ (red dashed line), but can be lower if the number of communities estimated is larger than $n$. Error bars correspond to the $.95$ confidence interval.}
  \label{fig:stochastic-ksbm-agreement}
\end{figure}

{We observe that our SCE underperforms at exact recovery when the number of communities is known for the other methods ($K=3$), but still captures large part of the community structure (agreement being around $0.85$). Notably the performance of all algorithms are resilient to the noise increases as long as it is low, with decreased performance in the transient regime for medium noise $b=2$, and finally major loss of performance at high noise $b=5$ when the oscillators dynamics resembles that of the Noisy KSBM. In particular, the pruned SCE has performance similar to most of the traditional method in the transient and steady state regimes.}

{If, on the contrary, the assumed number of communities is erroneous, then the performance of all traditional algorithm significantly decreases to around $0.7$, whereas SCE keeps its performance at around $0.85$. Pruning does not reduce the performance generally since it only reduces the number of clusters and does not create new clusters, hence if the SCE has found three clusters, it will not try to split them further. The exception being the single linkage hierarchical clustering which works better in almost all regimes and noise level since it preserves the large $K=3$ clusters during further partition, only creating small sub-clusters.}

\section{Generalization to Path Signatures}\label{appendix:sec:generalization-S}

\subsection{Steady State Path Signatures}

The following results cover the analytical formulas for steady state path signatures and lead matrices for $\gamma = \Theta, \ exp^{\imath\Theta}$ and $\sin(\Theta)$.\\

\begin{lemma}
\label{lemma:S-theta}
    (Steady state path signatures and lead matrix of the phase)\\Let $\Theta$ a realisation a KSBM $\cK$, then for any $I \in [N]^M, M \geq 1$ and  $i,j \in [N]$,
    \begin{equation*}
        S^{SS}_I(\Theta)(T) = \omega^{M}\frac{T^{M}}{M!} \text{ and } L^{SS}_{ij}(\Theta)(T) = 0.
    \end{equation*}
\end{lemma}
\begin{proof}
    Let $I \in [N]^M, M \geq 1$, we know from Lemma \ref{lemma:synchronized-steady state} that $\dot{\theta}_i(t_i) = \omega t + \Delta\theta_i$. Using this in the path signatures gets us,
    \begin{equation*}
        S^{SS}_I(\Theta)(T) = \int_{\Delta^M(0,T)}{\prod_{i\in I}{\dot{\theta}_i(t_i)dt_i}}
        = \int_0^T\int_0^{t_M}...\int_0^{t_2}{\omega^M\,dt_1...dt_M}
        = \omega^M\frac{T^M}{M!},
    \end{equation*}
    which is independent of what indices are in $I$. It then directly follows that for any $i,j \in [N]$,
    $$L^{SS}_{ij}(\Theta)(T) = \frac{1}{2}(S_{ij}^{SS}(\Theta)(T)-S_{ji}^{SS}(\Theta)(T))= 0.$$
\end{proof}

\begin{lemma}
\label{lemma:S-exp-i-theta}
    (Steady state path signatures and lead matrix of the complex phase)\\
    Let $\Theta$ a realisation a KSBM $\cK$, then for any $I \in [N]^M, M \geq 1$ and  $i,j \in [N]$
    \begin{equation*}
        S^{SS}_I(e^{\imath\Theta})(T) = \frac{\lambda_{I}}{M!}(e^{\imath\omega T}-1)^M \text{ and } L^{SS}_{ij}(e^{\imath\Theta})(T) = 0,
    \end{equation*}
    where $\lambda_{I} = e^{\imath\sum_{i\in I}{\Delta\theta_i}}$.
\end{lemma}
\begin{proof}
    At steady state, using Lemma \ref{lemma:synchronized-steady state}, $\gamma_i(t) = e^{\imath\Delta\theta_i}e^{\imath\omega t}$ and $\dot{\gamma}_i(t) = \imath\omega e^{\imath\Delta\theta_i}e^{\imath\omega t}$. We then prove the statement by induction on the length $M$ of the multi-index $I = (i_1, ..., i_M) \in [N]^M$. For the base case,
    \begin{equation*}
        S^{SS}_{i_1}(e^{\imath\Theta})(T) = \int_0^T{\dot{\gamma}_{i_1}(t)\, dt} = \gamma_{i_1}(T) - \gamma_{i_1}(0) = e^{\imath\Delta\theta_{i_1}}(e^{\imath\omega T}-1) = \frac{\lambda_{i_1}}{1!}(e^{\imath\omega T}-1)^1.
    \end{equation*}
    If we now assume
    $$S^{SS}_{i_1...,i_{M-1}}(e^{\imath\Theta})(T) = \frac{\lambda_{i_1...i_{M-1}}}{(M-1)!}(e^{\imath\omega T}-1)^{(M-1)},$$
    then
    \begin{equation*}
    \begin{split}
        S^{SS}_I(e^{\imath\Theta})(T) & = \int_0^T{S^{SS}_{i_1...,i_{M-1}}(e^{\imath\Theta})(t)\dot{\gamma}_{i_M}(t)\, dt} = \int_0^T{\frac{\lambda_{i_1...i_{M-1}}}{(M-1)!}(e^{\imath\omega t}-1)^{(M-1)}\imath\omega e^{\imath\Delta\theta_{i_M}}e^{\imath\omega t}\,dt},\\
        & = \lambda_I\int_0^T{\frac{\partial}{\partial t}(\frac{1}{M!}(e^{\imath\omega t}-1)^M)\,dt} = \frac{\lambda_I}{M!}(e^{\imath\omega T}-1)^M,
    \end{split}
    \end{equation*}
    as desired.\\
    \\
    If we consider a permutation $\sigma \in S_m$ on $[m]$ and define $\sigma(i_1...i_M) := (i_{\sigma(1)}...i_{\sigma(M)})$. Because the sum  $\sum_{i\in I}\Delta\theta_i$ is invariant under permutations of $I$, we have
    \begin{equation*}
        \lambda_{\sigma(I)} = \lambda_I \text{ and } S^{SS}_{\sigma(I)}(e^{\imath\Theta})(T) = S^{SS}_I(e^{\imath\Theta})(T).
    \end{equation*}
    It then immediately follows for $M=2$ that,
    $$L^{SS}_{ij}(e^{\imath\Theta})(T) = \frac{1}{2}(S_{ij}(e^{\imath\Theta})(T) - S_{ji}(e^{\imath\Theta})(T)) = 0$$
    since $(j,i)$ is a permutation of $(i,j)$.
\end{proof}

\begin{lemma}
\label{lemma:S-sin-theta}
     (Steady state path signatures and lead matrix of the sinusoid)\\
    Let $\Theta$ be a realisation a KSBM $\cK$, then for any $i,j \in [N]$
    \begin{equation*}
    \begin{split}
        S^{SS}_{ij}(\sin(\Theta))(T) = & \frac{1}{2}sin^2(\omega T)\cos(\Delta\theta_i + \Delta\theta_j) + \frac{1}{4}\sin(2\omega T)\sin(\Delta\theta_i + \Delta\theta_j),\\
        & + \frac{\omega T}{2}\sin(\Delta\theta_{ij}) + [\sin(\theta_j(t))]_0^T\sin(\Delta\theta_i),
    \end{split}
    \end{equation*}
    $$L^{SS}_{ij}(\sin(\Theta))(T) = \frac{\sin(\Delta\theta_{ij})}{2}(\omega T + \sin(\omega T)).$$
    For $I \in [N]^M, M\geq 1$, then
    \begin{equation*}
    \begin{split}
        S^{SS}_I(e^{\imath\Theta}-e^{-\imath\Theta})(T) & = (\imath\omega)^M\sum_{U\subseteq I}{\lambda_{U\subseteq I}B_{U\subseteq I}},
    \end{split}
    \end{equation*}
    where $\lambda_{U\subseteq I} = e^{\imath(\sum_{u\in U}{\Delta\theta_u}-\sum_{u\in I\setminus U}{\Delta\theta_u})}$ and,
    \begin{equation*}
        B_{U\subseteq I} = \int_{\Delta^M(0,T)}{(\prod_{u\in U}e^{\imath\omega t_u})(\prod_{u\in I\setminus U}-e^{-\imath\omega t_u})\,dt_{i_1}...dt_{i_M}}.
    \end{equation*}
\end{lemma}
\begin{proof}
    $S^{SS}_{ij}(\sin(\Theta))(T)$ can be computed by expanding the sinusoid and simple integration. We focus on computing $L^{SS}_{ij}(\sin(\Theta))(T)$ as a difference of the signatures. First remark that all the terms in the signatures that are in function of $\Delta\theta_i + \Delta\theta_j$ will cancel out due to permutation invariance. Second, let us expand,
    \begin{equation*}
    \begin{split}
        [\sin(\theta_j(t))]_0^T\sin(\Delta\theta_i) & = \sin(\omega T + \Delta\theta_j)\sin(\Delta\theta_i) - \sin(\Delta\theta_j)\sin(\Delta\theta_i),\\
        & = \sin(\omega T)\cos(\Delta\theta_j)\sin(\Delta\theta_i) + (\cos(\omega T)-1)\sin(\Delta\theta_j)\sin(\Delta\theta_i),
    \end{split}
    \end{equation*}
    where the second term is invariant under permutation of $(i,j)$ and thus also vanishes. Hence,
    \begin{equation*}
    \begin{split}
        2L^{SS}_{ij}(\sin(\Theta))(T) & = \frac{\omega T}{2}\sin(\Delta\theta_{ij}) - \frac{\omega T}{2}\sin(\Delta\theta_{ji}) + \sin(\omega T)(\sin(\Delta\theta_i)\cos(\Delta\theta_j) - \cos(\Delta\theta_i)\sin(\Delta\theta_j)),\\
        & = \omega T \sin(\Delta\theta_{ij}) + \sin(\omega T)\sin(\Delta\theta_{ij}).
    \end{split}
    \end{equation*}
    From which the claim readily follows.\\
    \\
    For the general Path Signatures $S^{SS}_I(e^{\imath\Theta}-e^{-\imath\Theta})(T)$, let $I\in [N]^M, M\geq 1$, we know
    $$\dot{\gamma}_i(t) = \imath \omega (e^{\imath (\omega t + \Delta\theta_i)} - e^{-\imath (\omega t + \Delta\theta_i)})$$
    Therefore, the product of the derivatives decomposes into a sum of products of $m$ factors where one selects either $e^{\imath (\omega t + \Delta\theta_i)}$ or $-e^{-\imath (\omega t + \Delta\theta_i)}$ for each
    \begin{equation*}
    \begin{split}
        \prod_{i\in I}\dot{\gamma}_i(t_i) & = (\imath \omega)^M\sum_{U\subseteq I}{(\prod_{u\in U}e^{\imath (\omega t_u + \Delta\theta_u)})(\prod_{u\in I\setminus U}-e^{-\imath (\omega t_u + \Delta\theta_u)})},\\
        & = (\imath \omega)^M\sum_{U\subseteq I}{\lambda_{U\subseteq I}(\prod_{u\in U}e^{\imath (\omega t_u)})(\prod_{u\in I\setminus U}-e^{-\imath (\omega t_u)})}.
    \end{split}
    \end{equation*}
    Using the linearity of the integral, we obtain,
    \begin{equation*}
        S^{SS}_I(e^{\imath\Theta}-e^{-\imath\Theta})(T) = \int_{\Delta^M(0,T)}{\prod_{i\in I}\dot{\gamma}_i(t_i)\,dt_i} = (\imath\omega)^M\sum_{U\subseteq I}{\lambda_{U\subseteq I}B_{U\subseteq I}}.
    \end{equation*}
\end{proof}

\subsection{Block-Clustering Metric}

Community homogeneity and discriminativity generalizes to any tensor $B \in \R^{(N^M)}$ with properties from Prop. \ref{prop:homogeneousness},\ref{prop:discriminativeness} and Prop.\ref{prop:clustering} being preserved.\\

\begin{definition}
    (Community Homogeneity)\\
    Let $B \in \R^{(N^M)}$ be a $M$-dimensional tensor, and consider a community assignment in the form of a partition $\coprod_{r\in [n]}G_r = [N]$. We define homogeneousness of $B$ as,
    $$h(B|G) = \frac{1}{n^M}\sum_{r_1,...,r_M \in [n]}{\Var_{I \in \prod_{k\in[M]}G_{r_k}}(B_I)}.$$
\end{definition}

\begin{definition}
    (Community Discriminativity)\\
    Let $B \in \R^{(N^M)}$ a $M$-dimensional tensor and community assignment $\coprod_{r\in [n]}G_r = [N]$. We define the discriminativity of $B$ as,
    $$d(B|G) = \frac{1}{n^M}\sum_{r_1,...,r_M \in [n]}{(B_{G_{r_1}...G_{r_M}}-B_{G_{r_1}...G_{r_1}})^2 + ... +(B_{G_{r_1}...G_{r_M}}-B_{G_{r_M}...G_{r_M}})^2},$$
    where $B_{G_{r_1}...G_{r_M}} = \E[I \in \prod_{k\in[M]}G_{r_k}]{B_I}.$
\end{definition}

\subsection{Structural Community Estimation Algorithm}

We adjust $D$ as the $l_2$ distance matrix between oscillators representative vector in tensor $B \in \R^{(N^M)}$, $v_i = (B_I)_{I \in [n]^M \text{ s.t. } i \in I}$ for $i \in [N]$ ordered on the index position of $i$ in $I$.

\begin{algorithm}
\caption{Structural community estimation algorithm}\label{alg:structural-community-estimation}
\begin{algorithmic}[1]
\State Given tensor $B \in \R^{(N^M)}$ with number of oscillators $N$
\State Compute $l_2$ distance matrix $D \in \R^{N\times N}_{\geq 0}$ based on $B$
\State Initialize community $G_1 \gets [N]$ and $g \gets 0$
\For{$n = 1:N$}
    \State Find $(i,j) = \argmax_{i,j \in G_r, r\in[n]}D_{ij}$
    \State Let $\mu_r \gets i$ where $r$ is such that $i\in G_r$ and $\mu_{n+1} \gets j$
    \State Initialize new communities $G'_\cdot \gets \{\}$
    \For{$k \in [N]$}
        \State $G'_{r^*} \gets G'_{r^*} \cup \{k\}$ where $r^* = \argmin_{r\in [n+1]}D_{k\mu_r}$ 
    \EndFor
    \State Compute $g' = g(B|G)/n$
    \If{$g'> g$} 
        \State $g \gets g'$ and $G \gets G'$
    \Else
        \State \textbf{break}
    \EndIf 
\EndFor
\State return $G$
\end{algorithmic}
\end{algorithm}

\end{document}